\documentclass[11pt,english,reqno]{amsart}
 \usepackage[foot]{amsaddr}

\usepackage[linktocpage,colorlinks,linkcolor=blue,anchorcolor=blue,citecolor=blue,urlcolor=blue,pagebackref]{hyperref}

\usepackage[letterpaper]{geometry}
\usepackage{fancyhdr,listings,amsfonts,amsmath,color}
\usepackage{latexsym,amssymb,amsmath,amsfonts,amsthm,xcolor,graphicx,longtable,wrapfig}
\usepackage[utf8]{inputenc}
\usepackage{geometry}
\usepackage{textgreek}
\usepackage{cleveref}
\usepackage{stmaryrd}

\allowdisplaybreaks
\renewcommand*{\backref}[1]{\ifx#1\relax \else Page #1 \fi}
\renewcommand*{\backrefalt}[4]{%
  \ifcase #1 \footnotesize{(Not cited.)}%
  \or        \footnotesize{(Cited on page~#2.)}%
  \else      \footnotesize{(Cited on pages~#2.)}%
  \fi
}

 \newcommand{\kb}[1]{{\color{magenta}\bf[KB: #1]}}
 \newcommand{\ye}[1]{{\color{blue}\bf[YE: #1]}}
\usepackage{amsmath}
\usepackage{graphicx}
\usepackage{float}
\usepackage{color}
\usepackage{pifont}
\usepackage{latexsym,amssymb,amsmath,amsfonts,amsthm,xcolor,graphicx}
\usepackage[utf8]{inputenc}
\usepackage{amsmath}
\usepackage{graphicx,enumitem}
\usepackage{float}
\usepackage{color}
\usepackage{pifont}
\usepackage{amssymb}
\usepackage{multirow}
\usepackage{diagbox}
\usepackage{cancel}
\usepackage{bbm}
\usepackage{xcolor}
\usepackage[ruled,vlined]{algorithm2e}
\usepackage{algpseudocode}

\usepackage{mathtools}

\newcommand{\lv}{\left\lVert}
\newcommand{\rv}{\right\rVert}

\newcommand{\mb}{\mathbb}
\newcommand{\mc}{\mathcal}

\newcommand{\trace}{\text{trace}}
\newcommand{\leftlangle}{\left\langle}
\newcommand{\rightrangle}{\right\rangle}
\newcommand{\KL}{\textsc{KL}}

\newtheorem{theorem}{Theorem}
\newtheorem{assumption}{Assumption}
\newtheorem{lemma}{Lemma}
\newtheorem{corollary}{Corollary}
\newtheorem{definition}{Definition}
\newtheorem{proposition}{Proposition}
\newtheorem{conjecture}{Conjecture}
\theoremstyle{definition}
\newtheorem{assump}{Assumption}
\newenvironment{myassump}[2][]
  {\renewcommand\theassump{#2}\begin{assump}[#1]}
  {\end{assump}}
\newtheorem{remark}{Remark}
\usepackage{mathtools}

\title{Regularized Stein Variational Gradient Flow} %\textbf{Corresponding author}

\author{Ye He $^1$}
\address{$^1$ School of Mathematics, Georgia Institute of Technology, 686 Cherry Street, Atlanta, GA, 30332. Email: yhe367@gatech.edu}
\author{Krishnakumar Balasubramanian $^2$}
\address{$^2$ Department of Statistics, University of California, Davis, 399 Crocker Lane, One Shields Avenue
Davis, CA 95616. Email: kbala@ucdavis.edu}
\author{Bharath K. Sriperumbudur$^3$}
\address{$^3$Department of Statistics, Pennsylvania State University, 314,  Thomas Building, University Park, PA 16802. Email: bks18@psu.edu}

\author{Jianfeng Lu$^4$}
\address{$^4$Mathematics Department, Duke University, Box 90320, 120 Science Dr., Durham, NC 27708-0320. Email: jianfeng@math.duke.edu}

\date{}

\begin{document}
\maketitle
\begin{abstract}
The Stein Variational Gradient Descent (SVGD) algorithm is \textcolor{black}{a} deterministic particle method for sampling. However, a mean-field analysis reveals that the gradient flow corresponding to the SVGD algorithm (i.e., the Stein Variational Gradient Flow) only provides a constant-order approximation to the Wasserstein Gradient Flow corresponding to the KL-divergence minimization. In this work, we propose the Regularized Stein Variational Gradient Flow, which interpolates between the Stein Variational Gradient Flow and the Wasserstein Gradient Flow. We establish various theoretical properties of the  Regularized Stein Variational Gradient Flow (and its time-discretization) including convergence to equilibrium, existence and uniqueness of weak solutions, and stability of the solutions. We provide preliminary numerical evidence of the improved performance offered by the regularization. 
\end{abstract}
%\tableofcontents
%\newpage

\vspace{0.07in}
\noindent \textbf{MSC Subject Classification.}  35Q62, 35Q68, 65C60, 82C22.
\vspace{0.07in}

\noindent \textbf{Keywords.} Wasserstein gradient flow, Stein variational gradient descent, particle-based sampling, convergence to equilibrium, mean-field analysis, reproducing kernel Hilbert space, regularization.
\vspace{0.07in}

%\noindent \textbf{Communicated by Joan Bruna}

\section{Introduction}\label{sec:introduction}

Given a potential function $V:\mathbb{R}^d \to \mathbb{R}$, the sampling problem involves generating samples from the density
\begin{align}\label{eq:target}
\pi(x) \coloneqq Z^{-1}e^{-V(x)},~\qquad\text{with}\qquad Z\coloneqq \int e^{-V(x)}dx
\end{align}
being the normalization constant, which is typically assumed to be unknown or hard to compute. The task of sampling arises in several fields of applied mathematics, including Bayesian statistics and machine learning in the context of numerical integration. There are two widely-used approaches for sampling: (i) diffusion-based \emph{randomized} algorithms, which are based on discretizations of certain diffusion processes, and (ii) particle-based \emph{deterministic} algorithms, which are discretizations of certain \emph{approximate} gradient flows. A central idea connecting the two approaches is the seminal work by~\cite{jordan1998variational} which provided a variational interpretation of the Langevin diffusion as the Wasserstein Gradient Flow (WGF),
\begin{align}\label{eq:W2GFofLD}
\partial_t \mu_t = \nabla \cdot \left( \mu_t ~\nabla_{W_2}F(\mu_t) \right) = \nabla \cdot \left( \mu_t ~\nabla \log \frac{\mu_t}{\pi} \right),
\end{align}
where the term $\nabla_{W_2}F(\mu_t)= \nabla \log \frac{\mu_t}{\pi} $ could be interpreted as the Wasserstein gradient\footnote{See, for example,~\cite{ambrosio2005gradient, santambrogio2017euclidean} for the exact definition.} of the relative entropy functional (also called as the Kullback–Leibler divergence), defined by
\begin{align}
F(\mu_t) = \text{KL}(\mu_t|\pi)\coloneqq\int_{\mb{R}^d} \log \frac{\mu_t(x)}{\pi(x)} \mu_t(x) dx,\nonumber
%\label{eq:relentropy}
\end{align}
evaluated at $\mu_t$. This leads to the idea that sampling could be viewed as \emph{optimization on the space of densities/measures}, a viewpoint that has provided a deeper understanding of the sampling problem~\cite{wibisono2018sampling, trillos2020bayesian}. 

There are several merits and disadvantages to both the randomized and deterministic discretizations of the (approximate) WGF. First, note that obtaining exact space-time discretization of the WGF in~\eqref{eq:W2GFofLD} is not possible. Indeed, due to the presence of the diffusion term, when initialized with a $N$-particle based empirical measure, the particles do not remain as particles for any time $t >0$. Hence, on the one hand, randomized discretizations like the Langevin Monte Carlo algorithm, are used as implementable space-time discretizations of the WGF. On the other hand, motivated by applications where the randomness in the discretization is undesirable, in the applied mathematics literature, other discretizations of \emph{approximate} WGF were developed. Such methods are predominantly based on using mollifiers and we refer the reader to~\cite{raviart1985analysis, russo1990deterministic,degond1990deterministic, craig2016blob, carrillo2019blob} for a partial list and to~\cite{chertock2017practical}, for a comprehensive overview.  

Recently, in the machine learning community, the Stein Variational Gradient Descent~\cite{liu2016stein,liu2017stein} was proposed as another deterministic discretization of approximate WGF, and has gathered significant attention due to applications in reinforcement learning~\cite{liu2017steinb}, graphical modeling~\cite{wang2018stein}, measure quantization~\cite{xu2022accurate}, and other fields of machine learning and applied mathematics \cite{wang2019stein, chewi2020svgd, chewi2020svgd, korba2020non}.  Due to the use of the reproducing kernels, the Stein Variational Gradient Descent (SVGD) algorithm provides a space-time discretization of the following \emph{approximate}  Wasserstein Gradient Flow (which we refer to as the Stein Variational Gradient Flow (SVGF) for simplicity)
\begin{align}\label{eq:SVGDGFofLD}
\partial_t \mu_t = \nabla \cdot \left( \mu_t ~ \mathcal{T}_{k,\mu_t}\nabla \log \frac{\mu_t}{\pi} \right),
\end{align}
where $\mathcal{T}_{k,\mu}: L^d_2(\mu) \to  L^d_2(\mu)$\footnote{\textcolor{black}{Defined in Section \ref{sec:notation}.}} is the integral operator defined as $\mathcal{T}_{k,\mu}f(x) = \int k(x,y) f(y) \mu(y) dy $ for \textcolor{black}{any} function $f\in L^d_2(\mu)$ , and for a kernel $k:\mathbb{R}^d\times\mathbb{R}^d \to \mathbb{R}$; see, for example~\cite{lu2019scaling}. Hence, SVGD (which is based on the SVGF), in this context, while being deterministic only provides a discretization of a \emph{constant-order approximation} to the Wasserstein Gradient Flow due to the presence of the kernel integral operator. Indeed, if $\text{supp}(\mu_t)=\mathbb{R}^d$ and $k$ is \textcolor{black}{a} bounded continuous translation invariant \emph{characteristic} kernel~\cite{sriperumbudur2010hilbert} on $\mathbb{R}^d$ (e.g., Gaussian, Laplacian kernels), then 
\begin{align*}
\Vert \mathcal{T}_{k,\mu_t}-I\Vert_{\text{op}}&=\sup\{\Vert \mathcal{T}_{k,\mu_t}f-f\Vert_{L^d_2(\mu_t)}:\Vert f\Vert_{L^d_2(\mu_t)}=1\}\ge \Vert \mathcal{T}_{k,\mu_t}\mathbf{1}-\mathbf{1}\Vert_{L^d_2(\mu_t)} \\ 
&\ge \Vert 1-\smallint k(\cdot,x)\mu_t(x)\,dx\Vert_{L_2(\mu_t)}>0,
\end{align*}
where $\mathbf{1}=(1,{\ldots},1)^\intercal \in \mb{R}^d$. This shows that the order of the error is crucially dependent on the choice of the kernel $k$.

To overcome the above issue with the SVGF, in this work, we propose the Regularized Stein Variational Gradient Flow (R-SVGF). To motivate the proposed flow, we first note that the Wasserstein gradient $\nabla \log (\mu_t/\pi)$ lives in $L^d_2(\mu_t)$, while the kernelized Wasserstein gradient $\mathcal{T}_{k,\mu_t}\nabla \log(\mu_t/\pi)$ \emph{morally} lives in the reproducing kernel Hilbert space\footnote{\textcolor{black}{Defined in Section \ref{sec:RKHS}.}} $\mathcal{H}^d_k \subset  L^d_2(\mu_t)$. If $\nabla \log (\mu_t/\pi)\in \overline{\text{Ran}(\mathcal{T}_{k,\mu_t})}$, then it is easy to verify that 
$$
\Vert ((1-\nu)\mathcal{T}_{k,\mu_t}+\nu I)^{-1}\mathcal{T}_{k,\mu_t}\nabla \log(\mu_t/\pi)-\nabla \log(\mu_t/\pi)\Vert_{L_2^d(\mu_t)}\rightarrow 0,\quad\text{as}\quad \nu\rightarrow 0. 
$$
Additionally, if 
$\nabla \log (\mu_t/\pi) $ is sufficiently smooth, i.e., there exists $\gamma\in \left(0,\frac{1}{2}\right]$ such that $\nabla \log (\mu_t/\pi) = \mathcal{T}^\gamma_{k,\mu_t} h$, for some $h\in L_2^d(\mu_t)$ (see, for example,~\cite{cucker2007learning}), then 
$$
\Vert ((1-\nu)\mathcal{T}_{k,\mu_t}+\nu I)^{-1}\mathcal{T}_{k,\mu_t}\nabla \log(\mu_t/\pi)-\nabla \log(\mu_t/\pi)\Vert_{L_2^d(\mu_t)}=O(\nu^{2\gamma}), \quad\text{as}\quad \nu\rightarrow 0. 
$$ 
In other words, $ ((1-\nu)\mathcal{T}_{k,\mu_t}+\nu I)^{-1}\mathcal{T}_{k,\mu_t}\nabla \log(\mu_t/\pi)$ is a good approximation to $\nabla \log(\mu_t/\pi)$ for small $\nu$. With this motivation, we propose the following R-SVGF given by
\begin{align}\label{eq:RSVGDGFofLD}
\partial_t \mu_t  = \nabla \cdot \left( \mu_t ~ \left( (1-\nu)\mathcal{T}_{k,\mu_t}+\nu I \right)^{-1} \mathcal{T}_{k,\mu_t}\left(\nabla \log \frac{\mu_t}{\pi} \right) \right),
%\nabla \cdot \left( \mu_t ~ (\mathcal{T}_{k,\mu_t}+ \nu I)^{-1} \mathcal{T}_{k,\mu_t}\nabla \log \frac{\mu_t}{\pi} \right),
\end{align}
for some regularization parameter $\nu \in (0,1]$, where 
% The advantage of the proposed R-SVGF is that independent of the chosen kernel and the target density, the \textcolor{red}{operator norm} between the difference of the  $((1-\nu)\mathcal{T}_{k,\mu_t}+ \nu I)^{-1} \mathcal{T}_{k,\mu_t}$ and the identity operator is of the order $\nu^{2\gamma}$. 
R-SVGF arbitrarily approximates the WGF as $\nu\rightarrow 0$. It is important to note that in the case of $\gamma=1/2$, we have $\nabla \log (\mu_t/\pi) \in \mathcal{H}_k^d$, yet,~\eqref{eq:SVGDGFofLD} suffers from the drawback of providing only a constant-order approximation to \eqref{eq:W2GFofLD}. 

\subsection{Summary of Contributions} Our contributions in this work are as follows:
\begin{enumerate}
    \item We propose the Regularized SVGF (R-SVGF) that interpolates between the Wasserstein Gradient Flow and the SVGF. The advantage of the proposed flow is that one could obtain an implementable space-time discretization as long as the regularization parameter is bounded away from zero. The main intuition behind the proposed flow is to pick an appropriately small regularization parameter so that we could arbitrarily approximate the WGF (Theorems~\ref{thm:relation to Langevin FI} and~\ref{thm:relation to Langevin}).
    \item For the R-SVGF, we  provide rates of convergence to the equilibrium density in two cases: (i) in the Fisher Information metric under no further assumptions on the target $\pi \in \mathcal{P}(\mathbb{R}^d)$  (Theorem~\ref{thm:Fisher convergence continuous}) and (ii) in the KL-divergence metric under an LSI assumption on the target (Theorem~\ref{thm:decay of KL under LSI}). We also establish similar results for the time-discretized R-SVGF (Theorems~\ref{thm:decay of Fisher along pop limit} and~\ref{thm:decay of KL along pop limit log-Sobolev version}). 
    \item We characterize the existence and uniqueness (Theorem~\ref{thm:regularized mf PDE unique and existence}), and stability (Theorem~\ref{thm:stability Wp}) of solutions to the R-SVGF. 
    \item We provide preliminary numerical experiments demonstrating the advantage of the space-time discretization of the R-SVGF, which we call the Regularized Stein Variational Gradient Descent (R-SVGD) algorithm, over the standard SVGD algorithm. 
\end{enumerate}

\subsection{Organization} The rest of the paper is organized as follows. In Section~\ref{sec:notation}, we introduce the notations used in the rest of the paper. In Section~\ref{sec:RKHS}, we provide the preliminaries on reproducing kernel Hilbert spaces required for our work. In Section~\ref{sec:regularized SVGD}, we introduce the R-SVGF, along with the notion of regularized Stein-Fisher information, required for our analysis.  Due to the technical nature of the proofs, we postpone the results on existence and uniqueness of the R-SVGF, and related stability results respectively to Sections~\ref{sec:exist} and~\ref{sec:stable}. In Section~\ref{sec:Convergence results}, we provide convergence results on the R-SVGF and its time-discretized version. We conclude in Section~\ref{sec:practical} with a space-time discretization which provides a practically implementable algorithm, and provide preliminary empirical results.

\subsection{Notations}\label{sec:notation} 
We use the following notations throughout this work:  
\begin{itemize}
\item For a matrix, $\|\cdot \|_2$ denotes the matrix 2-norm (spectral norm) and $\| \cdot \|_{HS}$ denotes the Hilbert-Schmidt norm which is defined as $\|A \|_{HS}^2=\sum_{i,j=1}^d |a_{ij}|^2$ for any matrix $A=(a_{ij})_{i,j\in [d]}$.
\item The term $id$ denotes the $d\times d$ identity matrix. $I_d$ corresponds to the identity operator in the RKHS. $I$ corresponds to the identity operator in $L_2(\mu)$.
%\item \textcolor{red}{$\mc{P}(\mb{R}^d)$ denotes the space of all probability measures on $\mb{R}^d$} 
\item $\left(L_\infty(\mb{R}^d),\lv \cdot\rv_{L_\infty(\mb{R}^d)}\right)$ denotes the space of all essentially bounded measurable functions on $\mb{R}^d$ with $\lv f \rv_{L_\infty(\mb{R}^d)}:= \inf \{ C: |f(x)|\le C \  \text{for almost every }x\in \mb{R}^d \} $ for any $f\in L_\infty(\mb{R}^d)$.
\item $\mc{C}(\mb{R}^d;\mb{R}^d)$ is the space of all $\mb{R}^d$-valued continuous functions on $\mb{R}^d$.
\item $\mc{C}_0^\infty ([0,\infty)\times \mb{R}^d)$ is the space of all $\mb{R}$-valued measurable functions on $[0,\infty)\times \mb{R}^d$ that vanish at infinity, i.e., for any $f\in \mc{C}_0^\infty ([0,\infty)\times \mb{R}^d)$, $f(t,x)\to 0$ as $t\to \infty$ and $f(t,x)\to 0$ as $x\to\infty$.
\item For any function space $\mc{H}$ on $\mb{R}^d$, $\mc{C}([0,T];\mc{H})$ is the space of functions $f$ such that for any fixed $t\in [0,T]$, $f(t,\cdot)\in \mc{H}$ and for any fixed $x\in \mb{R}^d$, $f(\cdot,x)$ is a continuous function on $[0,T]$. $\mc{C}^1([0,T];\mc{H})$ is the space of functions $f$ such that for any fixed $t\in [0,T]$, $f(t,\cdot)\in \mc{H}$ and for any fixed $x\in \mb{R}^d$, $f(\cdot,x)$ is a continuous function with a continuous first order derivative on $[0,T]$.
\item \textcolor{black}{Let $\mc{P}(\mb{R}^d)$ denote the space of all probability densities (with respect to the Lebesgue measure) on $\mb{R}^d$ that are twice continuously differentiable with full supports on $\mb{R}^d$ and finite second moments.}
\item \textcolor{black}{For any $\mu_1,\mu_2\in\mc{P}(\mb{R}^d)$ and any $p\ge 1$, $\mc{W}_p(\mu_1,\mu_2)$ denotes the Wasserstein-$p$ distance between $\mu_1$ and $\mu_2$ defined as $\mc{W}_p(\mu_1,\mu_2)\coloneqq \inf_{X\sim \mu_1,Y\sim \mu_2} \mb{E}[|X-Y|^p]^{\frac{1}{p}}$.}
\item For any $\mu\in \mc{P}(\mb{R}^d)$, $\left(L_2(\mu), \lv \cdot \rv_{L_2(\mu)} \right)$ is the space of all $\mu$-square integrable measurable function on $\mb{R}^d$ with $\lv f\rv_{L_2(\mu)}^2:=\int_{\mb{R}^d} |f(x)|^2 \mu(x) dx$. %\textcolor{black}{$\left(L_2^d(\mu), \lv \cdot \rv_{L_2^d(\mu)} \right)$ is the space of all $d$-dimensional vector valued functions whose entries are $\mu$-square integrable measurable functions with $\lv f\rv_{L_2^d(\mu)}^2:=\sum_{i=1}^d \lv f_i \rv_{L_2(\mu)}^2$ for any $f=[f_1,\cdots, f_d]\in L_2^d(\mu)$. }
\item \textcolor{black}{For any function space $\mathcal{H}$, we say $f\in \mc{H}^d$ if $f=[f_1,\cdots,f_d]^\intercal$ such that $f_i\in \mc{H}$ for all $i\in [d]$. If the function space $(\mc{H},\lv \cdot \rv_{\mc{H}})$ is further a Hilbert space, $(\mc{H}^d, \lv\cdot\rv_{\mc{H}^d})$ denotes the Hilbert space of all vector valued functions whose entries are in $(\mc{H},\lv\cdot\rv_{\mc{H}})$ with $\lv f \rv_{\mc{H}^d}^2\coloneqq \sum_{i=1}^d \lv f_i \rv_{\mc{H}}^2$ for any $f=[f_1,\cdots,f_d]^\intercal \in \mc{H}^d$.}
\item Let $\left(\mathcal{H},\lv \cdot \rv_{\mc{H}}\right)$ and $\left(\mathcal{G},\lv \cdot \rv_{\mc{G}}\right)$ be two function spaces. For an operator $A:\mathcal{H} \to \mathcal{G} $, we denote the adjoint operator of $A$ by $A^*$. We denote the operator norm by $\| A\|_{\mathcal{H} \to \mathcal{G}}$, which is defined as $\lv A \rv_{\mc{H}\to \mc{G}}:=\sup_{\lv u \rv_{\mc{H}}\le 1} \lv A u \rv_{\mc{G}}$. When we don't emphasize the spaces, we denote the operator norm of $A$ by $\lv A \rv_{\text{op}}$ for simplicity.
\item Let $\left( \mathcal{H},\|\cdot \|_{\mathcal{H}} \right)$ and $\left( \mathcal{G},\|\cdot \|_{\mathcal{G}} \right)$ be two Hilbert spaces. For an operator $A:\mathcal{H}\to \mathcal{G}$, we denote the Hilbert-Schmidt norm by $\| A \|_{HS}$ which is defined as $\| A \|_{HS}^2:=\sum_{i\in I} \| A e_i\|_{{\mathcal{G}}}^2 $ where $\{e_i\}_{i\in I}$ is an orthonormal basis of $\mathcal{H}$. We denote the nuclear norm by $\|A\|_{nuc}$ which is defined as $\lv A \rv_{nuc}:=\sum_{i\in I} \langle (A ^* A)^{\frac{1}{2}}e_i, e_i \rangle _\mathcal{H} $ where $\{e_i\}_{i\in I}$ is an orthonormal basis of $\mathcal{H}$.
%\item Define Hilbert Schmidt norm and nuclear norm (point v in prop 1)
\item For a smooth function $f:\mb{R}^d\times\mb{R}^d\to \mb{R}$, $\nabla_1 f $ denotes the gradient of $f$ in the first variable and $\nabla_2 f $ denotes the gradient of $f$ in the second variable.
\item For a map $\phi: \mathbb{R}^d \to \mathbb{R}^d$, $\phi_i$ denotes the $i$-th component of the function value and $\nabla \phi$ denotes the Jacobian, i.e., $(\nabla \phi)_{ij} = \partial_j\phi_i $ for all $i,j\in [d]$.
\item $T_{\#}\rho$ represents the push-forward of the density $\rho$ under a map $T$.
\item $\langle \cdot, \cdot \rangle_{\mc{H}}$ denotes inner-product in the Hilbert space $\mc{H}$. $\langle \cdot, \cdot \rangle$ denotes inner-product in the Euclidean space $\mb{R}^d$.
\end{itemize}

\iffalse
\section{Notations}
\begin{table}[h]
\centering
\begin{tabular}{ c c c } 
 \hline
 Parameters & Notations & Equations\\ \hline
 number of particles & $N$ & \\ 
 dimension & $d$ & \\
 time step size & $h$ &  \\
 regularized parameter & $\nu$ & \\
 position of $i$-th particle at step $n$ & $X_n^i$ & \\
 constants related to assumptions on $V$ & $C_{1,0}$ & \cref{ass:existence and uniqueness on V} \\
 constants related to assumptions on $V$ & $C_V$ & \cref{ass:existence and uniqueness on V V2} \\
 $W_p$ stability constant & $C$ & \eqref{eq:stability constant}
 \end{tabular}
 \caption{A list of all parameters.}
 \label{tab:table1}
 \end{table}
\fi

\section{Preliminaries on Reproducing Kernel Hilbert Space}\label{sec:RKHS} In this section, we introduce some properties of RKHS which would be used later in the formulation and analysis of R-SVGF. We refer the reader to~\cite{steinwart2008support, berlinet2011reproducing, paulsen2016introduction} for the basics of RKHS. We let $\mc{H}_k$ be a separable RKHS over $\mb{R}^d$ with the reproducing kernel $k:\mb{R}^d\times \mb{R}^d\to \mb{R}_{>0}$ and with $\| \cdot\|_{\mc{H}_k}$ denoting the associated RKHS norm. 
We make the following assumption on the kernel function $k$ throughout the paper. 

\begin{myassump}{A1}\label{ass:general assumption on k} The kernel function $k:\mb{R}^d \times \mb{R}^d \to \mb{R}$ is \textcolor{black}{symmetric}, strictly positive definite and continuous. 
\end{myassump}
Additional assumptions on the kernel will be introduced when they are required. The following results are essentially based on~\cite[Lemma 4.23, and Theorems 4.26 and 4.27]{steinwart2008support}. \begin{proposition}[\cite{steinwart2008support}]
\label{prop:RKHS property} The following holds.
\begin{itemize}
    \item [(i)] The kernel function $k$ is bounded if and only if every $f\in \mc{H}_k$ is bounded. Moreover, the inclusion $\iota_d:\mc{H}_k\to L_\infty(\mb{R}^d)$ is continuous and $\lv \iota_d \rv_{\mc{H}_k\to L_\infty(\mb{R}^d)}=\lv k \rv_\infty$, where $\lv k \rv_\infty:=\sup_{x\in \mb{R}^d} \sqrt{k(x,x)} $.
    \item [(ii)] Let $\mu$ be a $\sigma$-finite measure on $\mb{R}^d$ and $k$ be a measurable kernel. Assume that 
    \begin{align*}%\label{eq:L2 kernel}
        \lv k \rv_{L_2(\mu)}:=\left( \int_{\mb{R}^d} {k(x,x)} d\mu(x) \right)^{\frac{1}{2}}<\infty.
    \end{align*}
    Then $\mc{H}_k$ consists of $2$-integrable functions and the inclusion $\iota_{k,\mu}:  \mc{H}_k\to L_2(\mu)$ is continuous with $\lv \iota_{k,\mu}\rv_{\mc{H}_k\to L_2(\mu) }\le \lv k \rv_{L_2(\mu)} $. Moreover, the adjoint of this inclusion is the operator $\iota_{k,\mu}^* :L_2(\mu)\to \mc{H}_k $ defined by
    \begin{align*}%\label{eq:RKHS integral operator}
        \iota_{k,\mu}^* g(x):=\int_{\mb{R}^d} k(x,y) g(y) d\mu(y) ,\qquad g\in L_2(\mu),\ x\in \mb{R}^d. 
    \end{align*}
    \item [(iii)] Under the assumptions in $(ii)$, $\mc{H}_k$ is dense in $L_2(\mu)$ if and only if $\iota_{k,\mu}^*:L_2(\mu)\to \mc{H}_k$ is injective. Alternatively, $\iota_{k,\mu}^*:L_2(\mu)\to \mc{H}_k$ has a dense image if and only if  $\iota_{k,\mu}:\mc{H}_k\to L_2(\mu)$ is injective. 
    \item [(iv)]  Under the assumptions in $(ii)$, $\iota_{k,\mu}:\mc{H}_k\to L_2(\mu)$ is a Hilbert-Schmidt operator with $\lv \iota_{k,\mu} \rv_{HS}=\lv k \rv_{L_2(\mu)}$. Moreover, the integral operator $\mc{T}_{k,\mu}=\iota_{k,\mu}\iota_{k,\mu}^*:L_2(\mu)\to L_2(\mu)$ is compact, positive, self-adjoint, and nuclear with $\lv \mc{T}_{k,\mu} \rv_{nuc}=\lv \iota_{k,\mu} \rv_{HS}=\lv k \rv_{L_2(\mu)}$.
\end{itemize}
\label{prop2}
\end{proposition}

%In the following sections, we will use these RKHS properties frequently when we introduce the formulation and analysis to the regularized SVGD and its mean-field limit. 
 
%\begin{remark}\label{rem:corollary on A1} 
% Assumption~\ref{ass:general assumption on k} ensures all the statements in Proposition~\ref{prop:RKHS property} to be true.\\

% \eqref{eq:L2 kernel} holds for all probability measure $\rho$ \textcolor{red}{(You mean $\rho$ or $\mu$?)}. All the statements in proposition \ref{prop:RKHS property} are true under \cref{ass:general assumption on k}.
% %\end{remark} Similarly, we say $f\in L_2^d(\mu)$, if $f=[f_1,\cdots,f_d]$ such that $f_i\in L_2(\mu)$ for all $i\in [d]$.

  %The $L^d_2(\mu)$ norm of $f \in L^d_2(\mu)$ is given by $\lv f \rv_{L_2^d(\mu)}^2:=\sum_{i=1}^d \lv f_i \rv_{L_2(\mu)}^2$.
When $f\in \mc{H}_k^d$ with $f=[f_1,\cdots, f_d]^\intercal$ and $g\in \mc{H}_k$, we define $\langle f,g \rangle_{\mc{H}_k}$ as a vector in $\mb{R}^d$ and $\left( \langle f,g \rangle_{\mc{H}_k} \right)_i=\langle  f_i,g \rangle_{\mc{H}_k}$ for all $i\in [d]$. When $f\in L_2^d(\mu)$ with $f=[f_1,\cdots, f_d]^\intercal$ and $g\in L_2(\mu)$, we define $\langle f,g \rangle_{L_2(\mu)}$ as a vector in $\mb{R}^d$ and $\left( \langle f,g \rangle_{L_2(\mu)} \right)_i=\langle  f_i,g \rangle_{L_2(\mu)}$ for all $i\in [d]$. \textcolor{black}{Note also that $\mc{T}_{k,\mu}=\iota_{k,\mu}\iota^*_{k,\mu}$ and} $\text{Ran}(\textcolor{black}{\mc{T}_{k,\mu}^{1/2}})=\mathcal{H}_k\subset L_2(\mu)$. We refer the interested reader to~\cite{cucker2007learning} for more details.

Finally, we remark that by letting $(\lambda_i,e_i)_{i=1}^\infty$ 
 be the set of eigenvalues and eigenfunctions of the operator $\textcolor{black}{\mc{T}_{k,\mu}}$ where $\lambda_1\ge \lambda_2\ge \cdots >0$ and $(e_i)_{i=1}^\infty$ form an orthonormal system in $\text{Ran}(\mc{T}_{k,\mu})$, we have the following spectral representation that, for all $f\in L_2(\mu)$,
\begin{equation}\label{eq:SD of Tk}
    \begin{aligned}
    \mc{T}_{k,\mu}f=\sum_{i=1}^\infty \lambda_i \leftlangle f, e_i \rightrangle_{L_2(\mu)} e_i.
    \end{aligned}
\end{equation}
Computing the spectral representation, in general for any given $\mu$ and kernel $k$ is a non-trivial task. Results are only known on a case-by-case basis; see, for example,~\cite{minh2006mercer,azevedo2014sharp,chen2020deep,scetbon2021spectral}. However, we use the decomposition only in our analysis. For the purely practical algorithm that we describe eventually in Section~\ref{sec:practical}, we do not need to know the decomposition explicitly. 

\begin{remark}\label{rem:trivial}
\textcolor{black}{Strictly speaking, the above notation implicitly assumes that the operator $\mathcal{T}_{k,\mu}$ has a trivial null space, in which case $\overline{\text{Ran}(\mathcal{T}_{k,\mu})} \equiv L_2(\mu)$ and hence the eigenfunctions $(e_i)_{i=1}^\infty$ form an orthonormal basis to $L_2(\mu)$. However, our analysis does not require this condition on $\mathcal{T}_{k,\mu}$. In particular, if $\mathcal{T}_{k,\mu}$ has a non-trivial null-space, then $\overline{\text{Ran}(\mathcal{T}_{k,\mu})} \subset L_2(\mu)$. In this case, our analysis still holds true. For example, with a slight abuse of notation, if we let $e_i$, for certain values of $i$, also denote the basis of the null-space of $\mathcal{T}_{k,\mu}$, conclusions similar to our results hold.} 
\end{remark}

\section{Regularized SVGF}\label{sec:regularized SVGD}

We now introduce the formulation of the Regularized-SVGF and discuss its connection with the SVGF and the WGF. Recall that in the mean-field limit, the SVGF in~\eqref{eq:SVGDGFofLD} only provides a constant order approximation to the WGF in~\eqref{eq:W2GFofLD}, due to the presence of the operator $\mathcal{T}_{k,\mu}$. As the operator $\mathcal{T}_{k,\mu}$ is not invertible, we seek to obtain a regularized inverse so that we end up with  the following Regularized-SVGF, as in~\eqref{eq:RSVGDGFofLD}, for some regularization parameter $\nu \in (0,1]$. Note in particular that, as $\nu \to 0$, the Regularized-SVGF gets arbitrarily close to the WGF. Our goal in this section is to derive the above-mentioned R-SVGF from first principles.

The central operator required in our formulation is the following \emph{Stein operator}, which is defined for all \textcolor{black}{$p\in\mc{P}(\mb{R}^d)$}, and for all smooth maps $\phi:\mb{R}^d\to \mb{R}^d$, as
\begin{align*}
    \mc{A}_p\phi(x)= \phi(x) \otimes \nabla \log p(x)+\nabla \phi(x), 
\end{align*}
where $\otimes$ denotes the outer-product. %\textcolor{black}{It's worth noting that the operator $\nabla$ in the definition of the Stein operator can be defined as the weak derivative for any weakly smooth $p$. Here for simplicity, we assume that $p\in \mc{P}^2(\mb{R}^d)$.} 
\textcolor{black}{Generally speaking, the Stein operator can also be defined for some densities that are outside of $\mc{P}(\mb{R}^d)$. We restrict to $p\in \mc{P}(\mb{R}^d)$ in the paper for the purpose of our analysis.} Now, the Wasserstein Gradient Flow in~\eqref{eq:W2GFofLD} could be thought of as follows. Consider moving a particle $x \sim \rho$ (for some $\rho \in \mathcal{P}(\mathbb{R}^d)$) based on the mapping $x \mapsto T(x)\coloneqq x+h\phi(x)$, where $h>0$ is a step-size parameter, and $\phi$ is a vector-field chosen so that the KL-divergence between the pushforward of $\rho$ according to $T$, denoted as $T_{\#}\rho$, and the target density $\pi$ in minimal.  Liu and Wang \cite[Theorem 3.1]{liu2016stein}, showed that 
$$\nabla_h \KL(T_{\#}\rho|\pi)|_{h=0}=-\mb{E}_{x\sim \rho}[\text{trace}(\mc{A}_\pi \phi(x))].
$$
We also refer to~\cite{jordan1998variational} for an earlier version of the same result. Based on this observation, if we try to find the vector-field $\phi$ in the unit-ball of $L_2^d(\rho)$ that maximizes the quantity $[\mb{E}_{x\sim \rho}[\text{trace}(\mc{A}_\pi \phi(x) )]]^2$, a straight-forward calculation based on integration-by-parts, results in the optimal $\phi$ being the Wasserstein gradient $\nabla \log \frac{\rho}{\pi}$. To have a practical implementation,~\cite{liu2016stein} considered maximizing  $[\mb{E}_{x\sim \rho}[\text{trace}(\mc{A}_\pi \phi(x) )]]^2$ over the unit-ball in the RKHS $\mc{H}_k^d$, which results in the optimal vector-field being equal to $\mathcal{T}_{k,\rho}\nabla \log \frac{\rho}{\pi}$, and correspondingly results in the SVGF in~\eqref{eq:SVGDGFofLD}. 
 
%Recall that Wasserstein Gradient Flows corresponds to minimization of the KL-divergence. Let $T:\mb{R}^d\to \mb{R}^d$ be a map defined as $T(x)\coloneqq x+h\phi(x)$ for all $x\in \mb{R}^d$.  

In this work, we propose to find the vector field $\phi$ that maximizes $\left[\mb{E}_{x\sim \rho} \left[\trace(\mc{A}_\pi\phi(x))\right] \right]^2$ over the unit-ball with respect to an interpolated norm between $L_2^d(\rho)$ and $\mc{H}_k^d$. Specifically, the interpolation norm that we consider is of the form $\nu \lv \cdot \rv_{\mc{H}_k^d}^2+(1-\nu)\lv \cdot \rv_{L_2^d(\rho)}^2$, for some regularization parameter $\nu\in (0,1]$, which trades-off between $\Vert\cdot\Vert^2_{\mathcal{H}^d_k}$ and $\Vert\cdot\Vert^2_{L^d_2(\rho)}$. We also remark here that a similar idea has been leveraged in the context of RKHS-based statistical hypothesis testing~\cite{balasubramanian2021optimality}. Formally, for  $\rho, \pi \in \mc{P}(\mb{R}^d)$, we consider the following optimization problem. 
\begin{equation}%\label{eq:regularized KSD}
    \begin{aligned}
    S(\rho, \pi):=\max_{\phi\in \mc{H}_k^d} \left\{ \left[\mb{E}_{x\sim \textcolor{black}{\rho}} \left[\trace(\mc{A}_\pi\phi(x))\right] \right]^2 \qquad \text{such that} \quad \nu \lv \phi \rv_{\mc{H}_k^d}^2+(1-\nu)\lv \phi \rv_{L_2^d(\rho)}^2\le 1 \right\}.\nonumber
    \end{aligned}
\end{equation}
 For any $\rho\in \mc{P}(\mb{R}^d)$, the optimal vector field, $\phi$ that minimizes $\KL(T_{\#}\rho|\pi)$ can be described via the following result. \textcolor{black}{To make the statements more consistent with the analysis (and proofs), we will denote $\mc{T}_{k,\cdot}$ explicitly as $\iota_{k,\cdot}\iota^*_{k,\cdot}$ in the rest of the paper. } Before we proceed, we also recall that the Fisher information between two densities $\rho, \pi \in \mathcal{P}(\mathbb{R}^d)$, is defined as 
\begin{align}\label{eq:fisher}
    I(\rho|\pi)\coloneqq\lv \nabla \log \frac{\rho}{\pi} \rv_{L_2(\rho)}^2 =\sum_{i=1}^\infty \left|\left\langle \nabla \log \frac{\rho}{\pi}, e_i \right\rangle_{L_2(\rho)}\right|^2,
\end{align}
with $(e_i)_{i=1}^\infty$ being an orthonormal basis to $L_2(\rho)$.

\begin{proposition}\label{lem:regularized SVGD optimal vf} Let $T(x)=x+h\phi(x)$ and $T_{\#}\rho(z)$ be the density of $z=T(x)$ for any $\rho \in \mathcal{P}(\mathbb{R}^d)$. For $\nu \in(0,1]$, define %Consider all the perturbation directions in the set 
$$\mc{B}:=\{ \phi\in \mc{H}_k^d: \nu\lv \phi \rv_{\mc{H}_k^d}^2+(1-\nu)\lv \phi \rv_{L_2^d(\rho)}^2 \le 1 \}.$$ 
\textcolor{black}{Let $\pi\propto e^{-V}$ and $\pi\in \mc{P}(\mb{R}^d)$.} \textcolor{black}{If $I(\rho|\pi)<\infty$ and $k$ satisfies Assumption \ref{ass:general assumption on k} with $\int k(x,x)\rho(x)dx<\infty$,} then % of vector-valued RKHS $\mc{H}_k^d$, 
the direction of steepest descent in $\mc{B}$ that maximizes $$-\nabla_h \emph{KL}(T_{\#}\rho|\pi)|_{h=0} $$ is %$\phi_{\rho,\pi}^*$ 
given by 
\begin{equation}%\label{eq:op vf formula}
    \begin{aligned}
    \phi_{\rho,\pi}^*(\cdot)\propto \left((1-\nu)\iota_{k,\rho}^*\iota_{k,\rho}+\nu I_d \right)^{-1}\mb{E}_{x\sim \rho} [-\nabla V(x)k(x,\cdot)+\nabla k(x,\cdot)],\nonumber
    \end{aligned}
\end{equation}
where $\iota_{k,\rho}:\mc{H}_k^d\to L_2^d(\rho)$ is the inclusion operator and $\iota_{k,\rho}^*$ is its adjoint, as in Proposition~\ref{prop2}. Furthermore, under the optimal vector field $\phi_{\rho,\pi}^*$, we have $-\nabla_h \emph{KL}(T_{\#}\rho|\pi)|_{h=0}=S(\rho,\pi)$.
\end{proposition}
\begin{proof}
%[Proof of lemma \ref{lem:regularized SVGD optimal vf}]
\label{proof:optimal vf}
First note that according to \cite[Theorem 3.1]{liu2016stein}, we have $$\nabla_h \KL(T_{\#}\rho|\pi)|_{h=0}=-\mb{E}_{x\sim \textcolor{black}{\rho}}[\text{trace}(\mc{A}_\pi \phi(x) )].$$ 
Therefore, we have
\begin{align*}
    \phi_{\rho,\pi}^*=\underset{\phi\in \mc{H}_k^d}{\arg\max} \left\{ [\mb{E}_{x\sim \rho} [\trace(\mc{A}_\pi\phi(x))] ]^2 \qquad \text{such that} \quad \nu \lv \phi \rv_{\mc{H}_k^d}^2+(1-\nu)\lv \phi \rv_{L_2^d(\rho)}^2\le 1 \right\}.
\end{align*}
Next, observe that we have
\begin{align*}
    \mb{E}_{x\sim \rho} [\trace(\mc{A}_{\pi}\phi(x))]&=\sum_{i=1}^d \mb{E}_{x\sim\rho} [-\partial_i V(x)\phi_i(x)+\partial_i \phi_i(x)] \\
    &=\sum_{i=1}^d \mb{E}_{x\sim\rho} [-\partial_i V(x)\langle \phi_i, k(x,\cdot)\rangle_{\mc{H}_k} + \langle \phi_i, \partial_i k(x,\cdot)\rangle_{\mc{H}_k} ] \\
    &=\langle \phi, \mb{E}_{x\sim \rho} [-\nabla V(x)k(x,\cdot)+\nabla k(x,\cdot)] \rangle_{\mc{H}_k^d}.
\end{align*}
\textcolor{black}{The term $\mb{E}_{x\sim \rho} [-\nabla V(x)k(x,\cdot)+\nabla k(x,\cdot)]$ is finite for all $x\in \mb{R}^d$ due to the following inequality:
\begin{align*}
    |\mb{E}_{x\sim \rho} [-\nabla V(x)k(x,\cdot)+\nabla k(x,\cdot)]|&=|\iota_{k,\rho}^*\nabla \log \frac{\rho}{\pi} (x)|= \Big| \int k(x,y) \nabla \log \frac{\rho(y)}{\pi(y)} \rho(y)dx \Big| \\
    &\le I(\rho|\pi)^{\frac{1}{2}} \sqrt{k(x,x)}\Big(\int k(y,y)\rho(y)dy\Big)^{\frac{1}{2}}<\infty,
\end{align*}
where the first inequality follows from the Cauchy-Schwartz inequality and the reproducing property of the RKHS. The  second inequality follows from conditions on $k$ and $\rho$.
}
Meanwhile, the constraint can be written as
\begin{align*}
    \nu \lv \phi \rv_{\mc{H}_k^d}^2+(1-\nu)\lv \phi \rv_{L_2^d(\rho)}^2 &=\nu\langle \phi,\phi \rangle_{\mc{H}_k^d}+(1-\nu)\langle \iota_{k,\rho}\phi, \iota_{k,\rho}\phi \rangle_{L_2^d(\rho)} \\
    &=\langle \left(\nu I_d+(1-\nu)\iota_{k,\rho}^*\iota_{k,\rho}\right)\phi,\phi \rangle_{\mc{H}_k^d}\\
    &=\lv \left((1-\nu)\iota_{k,\rho}^*\iota_{k,\rho}+\nu I_d \right)^{\frac{1}{2}} \phi \rv_{\mc{H}_k^d}^2,
\end{align*}
where $I_d:\mc{H}_k\to \mc{H}_k$ is the identity operator. Now, note that  $\left((1-\nu)\iota_{k,\rho}^*\iota_{k,\rho}+\nu I_d \right)^{\frac{1}{2}}$ is well-defined since $\iota_{k,\rho}^*\iota_{k,\rho}:\mc{H}_k^d\to \mc{H}_k^d$ is positive, compact and self-adjoint. Therefore, based on the above display, the constraint $\{ \phi\in \mc{H}_k^d:  \nu \lv \phi \rv_{\mc{H}_k^d}^2+(1-\nu)\lv \phi \rv_{L_2^d(\rho)}^2\le 1 \}$ is equivalent to 
\begin{align*}
    \{ \phi\in \mc{H}_{k}^d: \psi=\left((1-\nu)\iota_{k,\rho}^*\iota_{k,\rho}+\nu I_d\right)^{\frac{1}{2}}\phi \quad and\quad \lv \psi \rv_{\mc{H}_k^d}\le 1 \}.
\end{align*}
Since the spectrum of $\iota_{k,\rho}^*\iota_{k,\rho}$ is positive and $\nu\in (0,1]$, $(1-\nu)\iota_{k,\rho}^*\iota_{k,\rho}+\nu I_d$ is invertible. For all $\phi\in \mc{H}_k^d$, there exists a unique $\psi\in \mc{H}_k^d$ such that $\left((1-\nu)\iota_{k,\rho}^*\iota_{k,\rho}+\nu I_d \right)^{-\frac{1}{2}}\psi=\phi$. Applying this fact along with the equivalent form of the constraint, we have
\begin{align*}
    \mb{E}_{x\sim \rho} (\trace(\mc{A}_{\pi}\phi(x)))&=\left\langle \left((1-\nu)\iota_{k,\rho}^*\iota_{k,\rho}+\nu I_d\right)^{-\frac{1}{2}}\psi, \mb{E}_{x\sim \rho} [-\nabla V(x)k(x,\cdot)+\nabla k(x,\cdot)]\right \rangle_{\mc{H}_k^d} \\
    &=\left\langle \psi, \left((1-\nu)\iota_{k,\rho}^*\iota_{k,\rho}+\nu I_d \right)^{-\frac{1}{2}} \mb{E}_{x\sim \rho} [-\nabla V(x)k(x,\cdot)+\nabla k(x,\cdot)] \right\rangle_{\mc{H}_k^d}\\
    &\le \lv \left((1-\nu)\iota_{k,\rho}^*\iota_{k,\rho}+\nu I_d \right)^{-\frac{1}{2}} \mb{E}_{x\sim \rho} [-\nabla V(x)k(x,\cdot)+\nabla k(x,\cdot)] \rv_{\mc{H}_k^d}
\end{align*}
where the second identity follows from the fact that $\left((1-\nu)\iota_{k,\rho}^*\iota_{k,\rho}+\nu I_d \right)^{-\frac{1}{2}}$ is self-adjoint and the upper bound in the last inequality is achieved when 
$$
\psi^*\propto \left((1-\nu)\iota_{k,\rho}^*\iota_{k,\rho}+\nu I_d \right)^{-\frac{1}{2}} \mb{E}_{x\sim \rho} [-\nabla V(x)k(x,\cdot)+\nabla k(x,\cdot)],
$$
and the result hence follows. 
% Therefore the optimal vector field $\phi_{\rho,\pi}^*$ is given by
% \begin{equation}\label{eq:optimal vf}
%     \begin{aligned}
%     \phi_{\rho,\pi}^*&= (\iota_{k,\rho}^*\iota_{k,\rho}+\nu I_d)^{-\frac{1}{2}}\psi^*\\
%     &\propto (\iota_{k,\rho}^*\iota_{k,\rho}+\nu I_d)^{-1} \mb{E}_{x\sim \rho} [-\nabla V(x)k(x,\cdot)+\nabla k(x,\cdot)]
%     \end{aligned}
% \end{equation}
\end{proof}

%In the formulation of SVGD, the vector field is chosen to be the optimal vector field within the unit ball in the RKHS that minimizes the KL divergence of the push forward density to the target density. The RKHS-unit ball constraint helps to have an explicit presentation for the optimal vector field. However, the RKHS structure would cause difficulty in showing fast convergence for the Fokker-Planck equation. Inspired by this observation, we construct a regularized SVGD by considering a different constraint set, which preserves the advantage of RKHS and the feature of Langevin dynamics at the same time. %In the following analysis, we use the same notations as they are used in \cite{liu2016stein}.  As an analogue to the kernelized Stein discrepancy in \cite{liu2016stein}, 

%In the framework of mean field limit, as the number of particles tends to infinity, the macroscopic behavior of the particle system \eqref{eq:regularized SVGD ODE system} is described by the following

With the optimal-vector field as derived above, we consider the following mean-field
partial differential equation (PDE) as the R-SVGF:
\begin{equation}\label{eq:regularized SVGD mf PDE}
    \begin{aligned}
    & \partial_t \rho_t=\nabla \cdot \left( \rho_t\  \iota_{k,\rho_t}\left((1-\nu)\iota_{k,\rho_t}^*\iota_{k,\rho_t}+\nu I_d \right)^{-1}\iota_{k,\rho_t}^*\left(\nabla \log \frac{\rho_t}{\pi}\right)  \right).
%    & \rho(0,\cdot)=\rho_0(\cdot).
    \end{aligned}
\end{equation}
It is important to notice that the R-SVGF interpolates between the SVGF and the WGF. However, the regime of interest for us is when $\nu \to 0$, as we get arbitrarily close to the WGF. We quantify this statement precisely in the later sections. On the other hand, when $\nu\to 1$, the R-SVGF becomes the SVGF. 

\begin{remark}\label{remark:alternative op vf} 
We now make the following remarks about the above result. 
\begin{itemize}
\item[(i)] We can alternatively write $\phi_{\rho,\pi}^*$ from Proposition~\ref{lem:regularized SVGD optimal vf} as
\begin{align*}
    \phi_{\rho,\pi}^* \propto -\left((1-\nu)\iota_{k,\rho}^*\iota_{k,\rho}+\nu I_d \right)^{-1} \iota_{k,\rho}^* \left(\nabla \log \frac{\rho}{\pi} \right),
\end{align*}
since we have
\begin{align*}
    \mb{E}_{x\sim \rho} [-\nabla V(x)k(x,\cdot)+\nabla k(x,\cdot)]
    =&\int_{\mb{R}^d} k(\cdot, x)\left(-\nabla V(x)-\frac{\nabla \rho(x)}{\rho(x)}\right)\rho(x)dx\\
    =&-\iota_{k,\rho}^*\left(\nabla \log \frac{\rho}{\pi}\right).
\end{align*}

\item[(ii)] The operator in \eqref{eq:regularized SVGD mf PDE} has an equivalent expression, as we discuss below. First, we claim that $$\iota_{k,\rho}\left((1-\nu)\iota_{k,\rho}^*\iota_{k,\rho}+\nu I_d\right)^{-1}\iota_{k,\rho}^*=\left((1-\nu)\iota_{k,\rho}\iota_{k,\rho}^*+\nu I \right)^{-1}\iota_{k,\rho}\iota_{k,\rho}^*.$$ 
To see that, we start with the trivial identity in the first line below and proceed as  
\begin{align*}
       & \left((1-\nu)\iota_{k,\rho}\iota_{k,\rho}^*+\nu I \right)\iota_{k,\rho}=\iota_{k,\rho}\left((1-\nu)\iota_{k,\rho}^*\iota_{k,\rho}+\nu I_d \right),\\
\implies & \iota_{k,\rho}=\left((1-\nu)\iota_{k,\rho}\iota_{k,\rho}^*+\nu I \right)^{-1}\iota_{k,\rho}\left((1-\nu)\iota_{k,\rho}^*\iota_{k,\rho}+\nu I_d \right) \\
   \implies & \iota_{k,\rho}\left((1-\nu)\iota_{k,\rho}^*\iota_{k,\rho}+\nu I_d \right)^{-1}=\left((1-\nu)\iota_{k,\rho}\iota_{k,\rho}^*+\nu I \right)^{-1}\iota_{k,\rho} \\
 \implies     & \iota_{k,\rho}\left((1-\nu)\iota_{k,\rho}^*\iota_{k,\rho}+\nu I_d \right)^{-1}\iota_{k,\rho}^*=\left((1-\nu)\iota_{k,\rho}\iota_{k,\rho}^*+\nu I \right)^{-1}\iota_{k,\rho}\iota_{k,\rho}^*. 
\end{align*}
According to this observation, \eqref{eq:regularized SVGD mf PDE} can also be written in the following form
\begin{align*}%\label{eq:regularized SVGD mf PDE version2}
%   \left\{
   \begin{aligned}
     \partial_t \rho_t&=\nabla \cdot \left( \rho_t\   \left((1-\nu)\iota_{k,\rho_t}\iota_{k,\rho_t}^*+\nu I \right)^{-1}\iota_{k,\rho_t}\iota_{k,\rho_t}^*\left(\nabla \log \frac{\rho_t}{\pi}\right)  \right)\\
     & = \nabla \cdot \left( \rho_t ~ \left( (1-\nu)\mathcal{T}_{k,\rho_t}+\nu I \right)^{-1} \mathcal{T}_{k,\rho_t}\left(\nabla \log \frac{\rho_t}{\pi} \right) \right),\\
    % &=\nabla \cdot \left( \rho_t ~ (1-\nu)^{-1}( \frac{\nu}{1-\nu} I+\mathcal{T}_{k,\rho_t})^{-1} \mathcal{T}_{k,\rho_t}\left(\nabla \log \frac{\rho_t}{\pi} \right) \right).
%    & \rho(0,\cdot)=\rho_0(\cdot).
   \end{aligned}
\end{align*}
thereby providing the R-SVGF introduced in ~\eqref{eq:RSVGDGFofLD} in Section~\ref{sec:introduction}. \\

\item [(iii)]\textbf{Particle-based spatial discretization.} %\textcolor{red}{[JL: perhaps ``particle based spatial discretization''?]} 
We now describe the spatial discretization of the R-SVGF. Based on the results in Proposition~\ref{lem:regularized SVGD optimal vf} and (ii) in Remark \ref{remark:alternative op vf}, we obtain the following ODE system:
\begin{align*}%\label{eq:regularized SVGD ODE system}
\left\{
    \begin{aligned}
    %&\frac{d x_i(t)}{dt}=-\left((1-\nu)\iota_{k,\rho_t^N}^*\iota_{k,\rho_t^N}+\nu I_d \right)^{-1}\iota_{k,\rho_t^N}^*\left(\nabla \log \frac{\rho_t^N}{\pi}\right)\left(x_i(t)\right)\\
    \frac{d x_i(t)}{dt}&=-\left((1-\nu)\iota_{k,\rho_t^N}^*\iota_{k,\rho_t^N}+\nu I_d \right)^{-1}\left(\frac{1}{N}\sum_{j=1}^N -\nabla_2 k\left(x_i(t),x_j(t)\right)+k\left(x_i(t),x_j(t)\right)\nabla V(x_j(t)) \right)\\
     x_i(0)&=x_i^0\in \mb{R}^d,\quad i=1,2,\ldots,N
    \end{aligned}
    \right.,
\end{align*}
where $\{x_i(t)\}_{i=1}^N$ is the set of $N$ particles. $\rho_t^N=\frac{1}{N}\sum_{j=1}^N \delta
_{x_j(t)}$ is the empirical distribution at time $t$, provides a $N$-particle spatial discretization of the R-SVGF.\\

\item[(iv)] \textbf{Time discretization.} We also have the following time-discretization of the R-SVGF. Let $\{h_{n}\}_{n=1}^\infty$ be the sequence of time step-size. We denote the density at the $n$-th iterate by $\rho^n$ for all integers $n \geq 1$. Then the time discretization of the R-SVGF can be written as
\begin{equation}\label{eq:pop limit regularized SVGD}
    \begin{aligned}
    \rho^{n+1}=\left(id-h_{n+1} \mc{D}_{\nu_{n+1},\rho^n} \nabla \log \frac{\rho^n}{\pi}\right)_{\#\rho^n},
    \end{aligned}
\end{equation}
where $\mc{D}_{\nu_n,\rho^n}=\left((1-\nu_n)\iota_{k,\rho^n}\iota^*_{k,\rho^n}+\nu_n I_d \right)^{-1}\iota_{k,\rho^n}\iota_{k,\rho^n}^*$.\\
\item[(v)] The parameter $\nu$ can also be made to be dependent on $t$ or $n$; in fact, in our analysis, we pick a time-varying regularization parameter. 
\end{itemize}
\end{remark}

\section{Convergence Results in Continuous and Discrete Time}\label{sec:Convergence results}
Our goal in this section is to derive convergence guarantees for the R-SVGF. Before we proceed, we introduce the notion of Regularized Stein-Fisher information (or Regularized Kernel Stein Discrepancy).

\subsection{Regularized Stein-Fisher Information and its Properties}
Related to the Fisher information defined in~\eqref{eq:fisher}, several works, for example~\cite{korba2020non,duncan2019geometry, salim2022convergence}, used the notion of Stein-Fisher Information to understand the convergence properties of the SVGD algorithm. The Stein-Fisher information was introduced in~\cite{chwialkowski2016kernel, liu2016kernelized,gorham2017measuring}  under the name Kernel Stein Discrepancy. It is defined as \begin{align*}
    I_{Stein}(\rho|\pi):=\lv \iota^*_{k,\rho}\nabla \log \frac{\rho}{\pi} \rv_{\mathcal{H}^d_k}^2=\left\langle \nabla \log \frac{\rho}{\pi}, \iota_{k,\rho}\iota_{k,\rho}^* \nabla \log \frac{\rho}{\pi} \right\rangle_{L^d_2(\rho)}=\sum_{i=1}^\infty \lambda_i \left|\left\langle \nabla \log \frac{\rho}{\pi}, e_i \right\rangle_{L_2(\rho)}\right|^2, 
\end{align*}
where $(\lambda_i,e_i)_{i=1}^\infty$ are the set of eigenvalues and eigenvectors of the operator $\iota_{k,\rho}\iota_{k,\rho}^*$, with $\lambda_1\ge \lambda_2\ge \cdots >0$. 
A drawback of the Stein-Fisher information is that it is a weaker metric, for example in comparison to the Fisher information metric in~\eqref{eq:fisher}; see~\cite{gorham2017measuring,gorham2019measuring,simon2020metrizing}. Indeed, comparing~\eqref{eq:fisher} and the above definition for the Stein-Fisher information, it is immediately clear that the Stein-Fisher information is severely restrictive, in particular when the eigenvalues of the chosen RKHS decay fast.  To counter this effect, we introduce the following regularized Stein-Fisher information and show that when the regularization parameter is chosen appropriately, the regularized Stein-Fisher information upper and lower bounds \textcolor{black}{the} Fisher information.

%Below, we introduce a regularized version of the Stein-Fisher information and show that as the regularization parameter tends to zero, it converges to the standard Fisher information.

\begin{definition}[Regularized Stein-Fisher Information]
\label{def:regularized Fisher information} 
For any $\textcolor{black}{\rho \in \mc{P}(\mb{R}^d)}$, the regularized Stein Fisher information from $\rho$ to $\pi$, denoted as $I_{\nu,Stein}(\rho|\pi)$, is defined as
\begin{align}\label{eq:regularized Fisher information} 
    I_{\nu,Stein}(\rho|\pi):=\left\langle \iota_{k,\rho}^* \nabla \log \frac{\rho}{\pi}, \left((1-\nu)\iota_{k,\rho}^*\iota_{k,\rho}+\nu I_d \right)^{-1} \iota_{k,\rho}^* \nabla \log \frac{\rho}{\pi}\right\rangle_{\mc{H}_k^d}.
\end{align}
\end{definition}
The regularized Stein Fisher information in \eqref{eq:regularized Fisher information} is well-defined because the operator 
$$
(1-\nu)\iota_{k,\rho}^*\iota_{k,\rho}+\nu I_d : \mc{H}_k^d\to \mc{H}_k^d
$$ 
is positive and for any $f\in \mc{H}_k^d$, $\left((1-\nu)\iota_{k,\rho}^*\iota_{k,\rho}+\nu I_d \right) f =0$ if and only if $f=0$.
\begin{remark}\label{prop:equivalent def}
The regularized Stein Fisher information has the following alternative representation:
    \begin{align}\label{eq:regularized Fisher information spectral form}
        I_{\nu,Stein}(\rho|\pi)=\sum_{i=1}^\infty \frac{\lambda_i}{(1-\nu)\lambda_i+\nu} \left|\left\langle \nabla \log \frac{\rho}{\pi}, e_i  \right\rangle_{L_2(\rho)}\right|^2.
    \end{align}
\end{remark}
For $\nu>0$, with the fact that $\lambda_i$ decreases to zero as $i\to\infty$, the regularized Stein Fisher information and the Stein Fisher information both encode the spectral decay information of $\iota_{k,\rho}\iota_{k,\rho}^*$. However, note that the regularized Stein Fisher information tends to the Fisher information as $\nu \to 0$. Hypothetically speaking, if $\nu$ is set to zero, then the regularized Stein Fisher information actually becomes the Fisher information. In our analysis, we will take advantage of the relation between the regularized Stein Fisher information and the Fisher information, while studying  the convergence properties of R-SVGF under Log-Sobolev inequality assumptions on the target $\pi$. A precise relation between the regularized Stein Fisher information and the Fisher information is stated in the following result. Before stating the result, we introduce the following notation for convenience. For $\gamma\in (0,\frac{1}{2}]$, we denote the pre-image of $\nabla \log \frac{\rho}{\pi}\in L^d_2(\rho)$ under $(\iota_{k,\rho}\iota_{k,\rho}^*)^{\gamma}$ as 
    $$
    \mathfrak{I}(\rho, \gamma):=(\iota_{k,\rho}\iota_{k,\rho}^*)^{-\gamma}\nabla \log \frac{\rho}{\pi}.
    $$
 Note that $\Vert \mathfrak{I}(\rho, \gamma)\Vert_{L^d_2(\rho)}$ is finite if and only if $\nabla \log \frac{\rho}{\pi}\in \text{Ran}((\iota_{k,\rho}\iota_{k,\rho}^*)^{\gamma})$.

\begin{proposition}[Equivalence relation between $I(\rho|\pi)$ and $I_{\nu,Stein}(\rho|\pi)$]\label{lem:relation between I and regularized Istein} Let \textcolor{black}{$\rho,\pi\in\mc{P}(\mb{R}^d)$ }% such that $I(\rho|\pi)$ and 
%$I_{\nu,Stein}(\rho|\pi)$ is \textcolor{black}{finite}. 
 and suppose there exists $\gamma\in (0,\frac{1}{2}]$ such that $\lv \mathfrak{I}(\rho, \gamma) \rv_{L_2^d(\rho)}<\infty.$ If the regularization parameter is chosen to satisfy the following condition,
\begin{align}\label{eq:cond}
\frac{\nu}{1-\nu}\le \left( \frac{I(\rho|\pi)}{2\lv \mathfrak{I}(\rho, \gamma) \rv_{L_2^d(\rho)}^2} \right)^{\frac{1}{2\gamma}},
\end{align}
then we have that 
$$
\frac{1}{2}(1-\nu)^{-1} I(\rho|\pi)\le I_{\nu,Stein}(\rho|\pi)\le (1-\nu)^{-1}I(\rho|\pi).
$$
\end{proposition}

\begin{proof}[Proof of Proposition~\ref{lem:relation between I and regularized Istein}]
%[Proof of lemma \ref{lem:relation between I and regularized Istein}]
\label{proof:relation between I and regularized Istein} 
According to~\eqref{eq:regularized Fisher information spectral form}, we have
\begin{align*}
    I_{\nu,Stein}(\rho|\pi)&=\sum_{i=1}^\infty \frac{\lambda_i}{(1-\nu)\lambda_i+\nu} \left|\left\langle \nabla \log \frac{\rho}{\pi}, e_i \right\rangle_{L_2(\rho)}\right|^2 \\
    &\le (1-\nu)^{-1}\sum_{i=1}^\infty  \left|\left\langle \nabla \log \frac{\rho}{\pi}, e_i \right\rangle_{L_2(\rho)}\right|^2\le (1-\nu)^{-1}I(\mu|\pi).
\end{align*}

On the other hand, since $\lv \mathfrak{I}(\rho, \gamma) \rv_{L_2^d(\rho)}<\infty$ for some $\gamma\in (0,\frac{1}{2}]$, there exists $h=\mathfrak{I}(\rho, \gamma)\in L_2^d(\rho)$ such that $$\nabla \log \frac{\rho}{\pi}=(\iota_{k,\rho}\iota_{k,\rho}^*)^{\gamma}h.$$ 
Therefore
\begin{align}
    (1-\nu)^{-1}I(\rho|\pi)-I_{\nu,Stein}(\rho|\pi)&=\sum_{i=1}^\infty \frac{(1-\nu)^{-1}\nu}{(1-\nu)\lambda_i+\nu} \left|\langle (\iota_{k,\rho}\iota_{k,\rho}^*)^{\gamma}h,e_i \rangle_{L_2(\rho)}\right|^2\label{Eq:equal1}\\
    &=\sum_{i=1}^\infty \frac{(1-\nu)^{-1}\nu \lambda_i^{2\gamma}}{(1-\nu)\lambda_i+\nu} |\langle h,e_i \rangle_{L_2(\rho)}|^2\nonumber\\
    &\le (1-\nu)^{-1-2\gamma}\nu^{2\gamma} \lv \mathfrak{I}(\rho, \gamma) \rv_{L_2^d(\rho)}^2\nonumber\\
    &\le \frac{1}{2}(1-\nu)^{-1}I(\rho|\pi),\nonumber
\end{align}
where the second to last inequality follows from the fact that
\begin{align*}
\sup_i \left(\frac{(1-\nu)^{-1}\nu \lambda_i^{2\gamma}}{(1-\nu)\lambda_i+\nu}\right)&=(1-\nu)^{-1-2\gamma}\nu^{2\gamma}\sup_i\left(\frac{(1-\nu)\lambda_i}{(1-\nu)\lambda_i+\nu}\right)^{2\gamma} \left(\frac{\nu}{(1-\nu)\lambda_i+\nu}\right)^{1-2\gamma}\\
&\le (1-\nu)^{-1\textcolor{black}{-2\gamma}}\nu^{2\gamma},
\end{align*}
and the last inequality follows from the condition in \eqref{eq:cond}.
\end{proof}

\subsection{Convergence results for R-SVGF}

\subsubsection{Relationship between R-SVGF and WGF}

\textcolor{black}{Assuming the existence and uniqueness of the R-SVGF (see Section~\ref{sec:exist}) and WGF (see~\cite{jordan1998variational} for sufficient conditions)} we now provide the relationship between the R-SVGF and the WGF in various metrics. We first start with the relationship in the Fisher information metric, without any stringent assumptions on the target density (thereby allowing for multi-modal and complex densities that arise in practice). Note that the Fisher information metric corresponds to the first-order stationarity metric for the WGF, obtained by minimizing the KL divergence. This metric has been recently proposed as a meaningful metric to consider in the case of sampling from general non-log-concave densities in~\cite{balasubramanian2022towards}. Note in particular under mild conditions on $q$ (e.g., connected support) that having the Fisher information $I(p|q) = 0$ implies $p\equiv q$. However, even when $I(p|q) \leq  \epsilon$, for some $\epsilon >0$, we have that the modes of the two densities are well-aligned, as argued in~\cite{balasubramanian2022towards}. \textcolor{black}{In order to state our next result, we denote by  $(\lambda_{i,t},e_{i,t})_{i=1}^\infty$, the set of eigenvalues and eigenvectors of the operator $\iota_{k,\rho_t}\iota_{k,\rho_t}^*$ for any $t\ge 0$, with $\lambda_{1,t}\ge \lambda_{2,t}\ge \cdots >0$.}

\begin{theorem}[Relation to the WGF in Relative Fisher Information]\label{thm:relation to Langevin FI}  Let $(\rho_t)$ be the solution to \eqref{eq:regularized SVGD mf PDE} and $(\mu_t)$ be the solution to the WGF, i.e.,
\begin{equation}\label{eq:Langevin FPE}
\left\{
\begin{aligned}
  &\partial_t \mu= \nabla \cdot \left( \mu \nabla \log \frac{\mu}{\pi} \right) ,\\
  &\mu(0,\cdot)=\mu_0(\cdot).
\end{aligned}
    \right.
\end{equation}
For any $t>0$, suppose there exists $\gamma_t\in(0,\frac{1}{2}]$ such that $\lv \mathfrak{I}(\rho_t, \gamma_t) \rv_{L_2^d(\rho_t)}<\infty.$ Then, for any initial density $\mu_0 \in \mathcal{P}(\mathbb{R}^d)$, and for any $T\in (0,\infty)$, we have
\begin{align}\label{eq:relation to Langevin FI}
    \int_0^T I(\rho_t|\mu_t) dt \le \frac{4}{3} \KL(\rho_0|\mu_0) +\frac{4}{3}%\lv k \rv_\infty^2 
     \int_0^T \textcolor{black}{(\lambda_{1,t}\vee 1)^2}\nu^{2\gamma_t}(1-\nu)^{-2\gamma_t} \lv \mathfrak{I}(\rho_t, \gamma_t) \rv_{L_2^d(\rho_t)}^2 dt,
\end{align}
where $\lambda_{1,t}$ is the largest eigenvalue of $\iota_{k,\rho_t}\iota^*_{k,\rho_t}$ for all $t\ge 0$.
\end{theorem}
\begin{proof}[Proof of Theorem~\ref{thm:relation to Langevin FI}]
First note that we have the following upper bound on $ \frac{d}{dt}\KL(\rho_t|\mu_t)$: 
\begin{align*}
    &~\quad\frac{d}{dt}\KL(\rho_t|\mu_t)\\
    &=\frac{d}{dt} \int_{\mb{R}^d} \log \frac{\rho_t(x)}{\mu_t(x)} \rho_t(x)dx \\
    &=\int_{\mb{R}^d} \partial_t \rho_t(x) \log \frac{\rho_t(x)}{\mu_t(x)} dx+\int_{\mb{R}^d} \left( \frac{\partial_t \rho_t(x)}{\mu_t(x)}+\rho_t(x)\partial_t\left(\frac{1}{\mu_t(x)}\right) \right) \frac{\mu_t(x)}{\rho_t(x)}\rho_t(x)dx \\
    &=\int_{\mb{R}^d} \partial_t \rho_t(x) \log \frac{\rho_t(x)}{\mu_t(x)} dx+\int_{\mb{R}^d} \partial_t \rho_t(x) dx -\int_{\mb{R}^d} \partial_t \mu_t(x) \frac{\rho_t(x)}{\mu_t(x)} dx \\
    &=-\int_{\mb{R}^d} \leftlangle \iota_{k,\rho_t}\left( (1-\nu)\iota_{k,\rho_t}^*\iota_{k,\rho_t}+\nu I_d\right)^{-1}\iota_{k,\rho_t}^* \nabla \log \frac{\rho_t(x)}{\pi(x)},\nabla \log \frac{\rho_t(x)}{\mu_t(x)} \rightrangle \rho_t(x)dx \\
    &\quad + 0+\int_{\mb{R}^d} \leftlangle \mu_t(x)\nabla \log \frac{\mu_t(x)}{\pi(x)}, \nabla \left( \frac{\rho_t(x)}{\mu_t(x)} \right) \rightrangle dx \\
    &= -\int_{\mb{R}^d} \leftlangle \iota_{k,\rho_t}\left( (1-\nu)\iota_{k,\rho_t}^*\iota_{k,\rho_t}+\nu I_d\right)^{-1}\iota_{k,\rho_t}^* \nabla \log \frac{\rho_t(x)}{\pi(x)},\nabla \log \frac{\rho_t(x)}{\mu_t(x)} \rightrangle \rho_t(x)dx \\
    &\quad + \int_{\mb{R}^d} \leftlangle \nabla \log \frac{\rho_t(x)}{\mu_t(x)},\nabla \log \frac{\mu_t(x)}{\pi(x)} \rightrangle \rho_t(x)dx \\
    &=-\int_{\mb{R}^d} \leftlangle \nabla \log \frac{\rho_t(x)}{\mu_t(x)}, \nabla \log \frac{\rho_t(x)}{\pi(x)}-\nabla \log \frac{\mu_t(x)}{\pi(x)} \rightrangle \rho_t(x)dx\\
    &\quad- \int_{\mb{R}^d} \leftlangle \left(\iota_{k,\rho_t}\left( (1-\nu)\iota_{k,\rho_t}^*\iota_{k,\rho_t}+\nu I_d\right)^{-1}\iota_{k,\rho_t}^*-I\right) \nabla \log \frac{\rho_t(x)}{\pi(x)},\nabla \log \frac{\rho_t(x)}{\mu_t(x)} \rightrangle \rho_t(x)dx \\
    &\le -\int_{\mb{R}^d} \left| \nabla \log \frac{\rho_t(x)}{\mu_t(x)} \right|^2 \rho_t(x)dx\\
    &\quad +\frac{1}{4}\int_{\mb{R}^d} \left| \nabla \log \frac{\rho_t(x)}{\mu_t(x)} \right|^2 \rho_t(x)dx+\lv \left(\iota_{k,\rho_t}\left( (1-\nu)\iota_{k,\rho_t}^*\iota_{k,\rho_t}+\nu I_d\right)^{-1}\iota_{k,\rho_t}^*-I\right)\nabla \log \frac{\rho_t}{\pi} \rv_{L_2^d(\rho_t)}^2 \\
    &=-\frac{3}{4} I(\rho_t|\mu_t)+\sum_{i=1}^\infty \frac{(1-\lambda_{i,t})^2\nu^2}{\left((1-\nu)\lambda_{i,t}+\nu\right)^2} \left| \leftlangle \nabla \log \frac{\rho_t}{\pi},e_{i,t} \rightrangle_{L_2(\rho_t)}\right|^2.
\end{align*}
In the above calculation, the fourth equality follows by integration-by-parts. The inequality follows by Young's inequality for the inner product (i.e.,  $\langle p,q \rangle\le \frac{1}{2}c|p|^2+\frac{1}{2c}|q|^2$ for any $p,q\in \mb{R}^d$) and the last equality follows from the proof of Proposition~\ref{lem:relation between I and regularized Istein}. Since  $\nabla \log \frac{\rho_t}{\pi}=(\iota_{k,\rho_t}\iota_{k,\rho_t}^*)^{\gamma_t} h_t$ for some $\gamma_t\in(0,1/2]$ with $h_t:=\mathfrak{I}(\rho_t, \gamma_t)\in L_2^d(\rho_t)$, we obtain
% Now, under the assumption in Theorem \ref{thm:relation to Langevin FI}, and according to the proof of Proposition~\ref{lem:relation between I and regularized Istein}, for any $t>0$, there exists $\gamma_t\in(0,1/2]$ and $h_t=\mathfrak{I}(\rho_t, \gamma_t)\in L_2^d(\rho_t)$ such that $\nabla \log \frac{\rho_t}{\pi}=(\iota_{k,\rho_t}\iota_{k,\rho_t}^*)^{\gamma_t} h_t$. Therefore
\begin{align*}
    \frac{d}{dt}\KL(\rho_t|\mu_t)&\le -\frac{3}{4}I(\rho_t|\mu_t)+  \left(\max_{i} (1-\lambda_{i,t})^2 \right)  \left(\max_i \frac{\lambda_{i,t} ^{\gamma_t} \nu}{\textcolor{black}{(1-\nu)}\lambda_{i,t}+\nu}\right)^2 \lv h_t \rv_{L_2^d(\rho_t)}^2 \\
    &\le -\frac{3}{4}I(\rho_t|\mu_t)+%\lv k \rv_\infty^2 
    \textcolor{black}{(\lambda_{1,t}\vee 1)^2} \nu^{2\gamma_t}\textcolor{black}{(1-\nu)^{-2\gamma_t}} \lv \mathfrak{I}(\rho_t, \gamma_t) \rv_{L_2^d(\rho_t)}^2 ,
\end{align*}
where the last inequality follows from the facts that 
$$
\max_i (1-\lambda_{i,t})^2\le (\max_i \lambda_{i,t}\vee 1)^2\quad\text{and}\quad \frac{\lambda_{i,t}^{\gamma_t}\nu}{(1-\nu)\lambda_{i,t}+\nu}\le \nu^{\gamma_t}(1-\nu)^{-\gamma_t}.
$$
% \bk{How do you know $\lambda_1 <1$? If we do not know, we need to assume this.}\ye{$\lambda_{1}$ is actually a quantity that depends on $t$. We denote it as $\lambda_{1,t}$ and include it in the integral.}
Integrating from $t=0$ to $t=T$, we get
\begin{align*}
    \KL(\rho_T|\mu_T)-\KL(\rho_0|\mu_0)\le -\frac{3}{4} \int_0^T I(\rho_t|\mu_t) dt +%\lv k \rv_\infty^2
    \int_0^T  \textcolor{black}{(\lambda_{1,t}\vee 1)^2}\nu^{2\gamma_t}\textcolor{black}{(1-\nu)^{-2\gamma_t}} \lv \mathfrak{I}(\rho_t, \gamma_t) \rv_{L_2^d(\rho_t)}^2 dt.
\end{align*}
Since KL-divergence is non-negative, \eqref{eq:relation to Langevin FI} is proved.
\end{proof}
\textcolor{black}{
\begin{remark} The sequence $\{\lambda_{1,t}\}_{t\geq 0}$ in Theorem \ref{thm:relation to Langevin FI} is the largest eigenvalue of $\iota_{k,\rho_t}\iota_{k,\rho_t}^*$ for all $t\ge 0$. Note that it depends on both the kernel $k$ and the solution $(\rho_t)$ to \eqref{eq:regularized SVGD mf PDE}. If the kernel function $k$ is assumed to be uniformly bounded, according to Proposition \ref{prop:RKHS property}, $\lambda_{1,t}^2$ is uniformly upper bounded by  $\sup_{x\in \mb{R}^d} k(x,x) $ for all $t\ge 0$. For specific choices of kernel function, initial and  target distributions, it is required to track $\lambda_{1,t}$ for all $t\ge 0$. We will illustrate this further with an example in Remark \ref{gaussianexample}.
\end{remark}
}
\begin{remark}
The above result shows that as long as $\rho_0 = \mu_0$, i.e., both the WGF and the R-SVGF are initialized with the same density, and $\nu$ is chosen such that 
$$
T^{-1}\int_0^T \textcolor{black}{(\lambda_{1,t}\vee 1)^2}\nu^{2\gamma_t}(1-\nu)^{-2\gamma_t} \lv \mathfrak{I}(\rho_t, \gamma_t) \rv_{L_2^d(\rho_t)}^2 dt \to 0,
$$ 
the \emph{averaged} Fisher information along the path tends to zero. This shows the benefit of regularizing the SVGF -- it enables one to closely approximate the WGF with appropriate choice of the regularization parameters. 
\end{remark}

\subsubsection{Convergence to Equilibrium along the Fisher Information}
We now provide results on the convergence to equilibrium along the Fisher information for the R-SVGF. We re-emphasize here that our result provided below holds as long as the target \textcolor{black}{$\pi\in \mc{P}(\mb{R}^d)$}, without additional structural assumptions (via, say, functional inequalities).  

\begin{theorem}[\textbf{Convergence of Fisher information}]\label{thm:Fisher convergence continuous} Let $(\rho_t)$ be the solution to \eqref{eq:regularized SVGD mf PDE}. For any $t>0$, suppose there exists $\gamma_t\in(0,\frac{1}{2}]$ such that $\lv \mathfrak{I}(\rho_t, \gamma_t) \rv_{L_2^d(\rho_t)}<\infty.$ Then 
\begin{equation*}%\label{eq:fisher convergence continuous}
    \int_0^\infty I(\rho_t|\pi) dt\le (1-\nu)\KL(\rho_0|\pi)+\int_0^\infty \nu^{2\gamma_t}(1-\nu)^{-2\gamma_t} \lv \mathfrak{I}({\rho}_t, \gamma_t) \rv_{L_2^d(\rho_t)}^2 dt, 
\end{equation*}
Furthermore, if $\int_0^\infty \nu^{2\gamma_t}(1-\nu)^{-2\gamma_t} \lv \mathfrak{I}(\rho_t, \gamma_t) \rv_{L_2^d(\rho_t)}^2 dt<\infty$, then we get $I(\textcolor{black}{\Bar{\rho}_t}|\pi)\to 0$ as $t\to \infty$, \textcolor{black}{where $\Bar{\rho}_t\coloneqq \frac{1}{t}\int_0^t \rho_s ds$ is the averaged density of $(\rho_s)_{0\le s\le t}$}.
\end{theorem}

Before proving the above theorem, we introduce a few intermediate results.

\begin{proposition}[Decay of the KL-divergence] \label{cor:KL derivative} For the solution $(\rho_t)_{t\ge 0}$ to the PDE \eqref{eq:regularized SVGD mf PDE}, it holds that
\textcolor{black}{
\begin{align}\label{eq:KL derivative}
    \frac{d}{dt}\KL(\rho_t|\pi)=-I_{\nu,Stein}(\rho_t|\pi),
\end{align}
}
and consequently
\textcolor{black}{
\begin{align}\label{eq:KL decay}
    \frac{d}{dt}\KL(\rho_t|\pi)\le 0.
\end{align}}
\end{proposition} 
\begin{proof}[Proof of Proposition~\ref{cor:KL derivative}]
Note that
  \begin{align*}
      \frac{d}{dt}\KL(\rho_t|\pi) &=\frac{d}{dt} \int_{\mb{R}^d} \rho_t \log \frac{\rho_t}{\pi} dx \\
      &=\int_{\mb{R}^d} \partial_t \rho_t \log \frac{\rho_t}{\pi} dx+\int_{\mb{R}^d} \partial_t \rho_t dx \\
      &=-\int_{\mb{R}^d} \leftlangle \nabla \log\frac{\rho_t}{\pi}(x) , \iota_{k,\rho_t}\left( (1-\nu)\iota_{k,\rho_t}^*\iota_{k,\rho_t}+\nu I_d\right)^{-1}\iota_{k,\rho_t}^*\left(\nabla \log \frac{\rho_t}{\pi}\right)(x) \rightrangle \rho_t(x) dx+0\\
      &=-\leftlangle \nabla \log \frac{\rho_t}{\pi}, \iota_{k,\rho_t}\left( (1-\nu)\iota_{k,\rho_t}^*\iota_{k,\rho_t}+\nu I_d\right)^{-1}\iota_{k,\rho_t}^*\left(\nabla \log \frac{\rho_t}{\pi}\right)   \rightrangle_{L_2^d(\rho_t)}\\
      &=-\leftlangle \iota_{k,\rho_t}^*\nabla \log \frac{\rho_t}{\pi} , \left( (1-\nu)\iota_{k,\rho_t}^*\iota_{k,\rho_t}+\nu I_d\right)^{-1}\iota_{k,\rho_t}^* \left(\nabla \log \frac{\rho_t}{\pi}\right)  \rightrangle_{\mc{H}_k^d}.
  \end{align*}
 It suffices to show that for all $\nu > 0$, $\left( (1-\nu)\iota_{k,\rho_t}^*\iota_{k,\rho_t}+\nu I_d\right)^{-1}$ is a positive operator from $\mc{H}_k^d$ to $\mc{H}_k^d$. By the definition of $\iota_{k,\rho_t}$, for any $f\in \mc{H}_k^d$ with $\lv f \rv_{\mc{H}_k^d}=1$,
 \begin{align*}
     \langle f, \left( (1-\nu)\iota_{k,\rho_t}^*\iota_{k,\rho_t}+\nu I_d\right)f \rangle_{\mc{H}_k^d}&=(1-\nu)\langle \iota_{k,\rho_t}f, \iota_{k,\rho_t}f \rangle_{L_2^d(\rho_t)}+\nu \lv f \rv_{\mc{H}_k^d}^2 \\
     &=(1-\nu)\lv \iota_{k,\rho_t} f \rv_{L_2^d(\rho_t)}^2+\nu > 0
 \end{align*}
 for all $\nu> 0$. Therefore, $ (1-\nu)\iota_{k,\rho_t}^*\iota_{k,\rho_t}+\nu I_d$ is a positive operator from $\mc{H}_k^d$ to $\mc{H}_k^d$. So is the operator $\left( (1-\nu)\iota_{k,\rho_t}^*\iota_{k,\rho_t}+\nu I_d\right)^{-1}$. Hence, we have~\eqref{eq:KL decay}. The claim in~\eqref{eq:KL derivative} follows directly from~\eqref{eq:KL decay}, \eqref{eq:SD of Tk} and Definition \ref{def:regularized Fisher information}.
 \end{proof}

%\begin{corollary}\label{cor:KL derivative} 

%\end{corollary}

\noindent  We now provide the proof of Theorem~\ref{thm:Fisher convergence continuous}. \textcolor{black}{For the proof, we recall that we use  $(\lambda_{i,t},e_{i,t})_{i=1}^\infty$ to denote the set of eigenvalues and eigenvectors of the operator $\iota_{k,\rho_t}\iota_{k,\rho_t}^*$ for any $t\ge 0$, with $\lambda_{1,t}\ge \lambda_{2,t}\ge \cdots >0$.}

\begin{proof}[Proof of Theorem~\ref{thm:Fisher convergence continuous}]
%[Proof of theorem \ref{thm:Fisher convergence continuous}] 
From Proposition \ref{cor:KL derivative} and  \eqref{Eq:equal1}, we know that 
\begin{align*}
    \frac{d}{dt}\KL(\rho_t|\pi)
    % &=-\sum_{i=1}^\infty \frac{\lambda_i}{(1-\nu)\lambda_i+\nu} \left| \leftlangle \nabla \log \frac{\rho_t}{\pi}, e_i \rightrangle_{L_2(\rho_t)} \right|^2 \\
    &=-(1-\nu)^{-1}I(\rho_t|\pi)+\sum_{i=1}^\infty\frac{(1-\nu)^{-1}\nu}{(1-\nu)\lambda_{i,t}+\nu} \left|\leftlangle \nabla \log \frac{\rho_t}{\pi}, e_{i,t} \rightrangle_{L_2(\rho_{t})}\right|^2,
\end{align*}
where $\nabla \log \frac{\rho_t}{\pi}=(\iota_{k,\rho_t}\iota_{k,\rho_t}^*)^{\gamma_t}h_t$ for some $\gamma_t\in (0,\frac{1}{2}]$ with $h_t:= \mathfrak{I}(\rho_t, \gamma_t)\in L_2^d(\rho_t)$.
% Under the assumption in Theorem \ref{thm:Fisher convergence continuous} and according to the proof in Proposition \ref{lem:relation between I and regularized Istein}, there exist $\gamma_t\in (0,\frac{1}{2}]$ and $h_t= \mathfrak{I}(\rho_t, \gamma_t)\in L_2^d(\rho_t) $ such that $\nabla \log \frac{\rho_t}{\pi}=(\iota_{k,\rho_t}\iota_{k,\rho_t}^*)^{\gamma_t}h_t$. 
Therefore, we have
% for any $t>0$, there exists $\gamma_t\in(0,1/2]$ and $h_t\in L_2^d(\rho_t)$ such that $\nabla \log \frac{\rho_t}{\pi}=(\iota_{k,\rho_t}\iota_{k,\rho_t}^*)^{\gamma_t} h_t$ and $\lv h_t\rv_{L_2^d(\rho_t)}=R(t)$. Therefore
\begin{align*}
   &\frac{d}{dt}\KL(\rho_t|\pi)\\
   =&-(1-\nu)^{-1}I(\rho_t|\pi)+\sum_{i=1}^\infty \frac{(1-\nu)^{-1}\nu}{(1-\nu)\lambda_{i,t}+\nu} \left| \leftlangle (\iota_{k,\rho_t}\iota_{k,\rho_t}^*)^{\gamma_t}h_t , e_{i,t} \rightrangle_{L_2(\rho_t)}\right|^2  \\
   =&-(1-\nu)^{-1}I(\rho_t|\pi)+\sum_{i=1}^\infty \frac{(1-\nu)^{-1}\nu\lambda_{i,t}^{2\gamma_t}}{(1-\nu)\lambda_{i,t}+\nu}  \left| \leftlangle h_t , e_{i,t} \rightrangle_{L_2(\rho_t)}\right|^2\\
   =&\sum_{i=1}^\infty (1-\nu)^{-1-2\gamma_t}\nu^{2\gamma_t}\left(\frac{(1-\nu)\lambda_{i,t}}{(1-\nu)\lambda_{i,t}+\nu}\right)^{2\gamma_t}\left(\frac{\nu}{(1-\nu)\lambda_{i,t}+\nu}\right)^{1-2\gamma_t} \left|\leftlangle h_t , e_{i,t} \rightrangle_{L_2(\rho_t)}\right|^2\\
   &\qquad -(1-\nu)^{-1}I(\rho_t|\pi)\\
   \le& -(1-\nu)^{-1}I(\rho_t|\pi)+(1-\nu)^{-1-2\gamma_t}\nu^{2\gamma_t}  \lv \mathfrak{I}(\rho_t, \gamma_t) \rv_{L_2^d(\rho_t)}^2.
\end{align*}
The result follows by integrating over $t$ and noting that the KL-divergence is non-negative. Now, with $\rho_t$ denoting the solution to (9), we have that $I(\rho_t|\pi)$ is non-negative and continuous in $t$. The claim of convergence \textcolor{black}{follows from the convexity of $\rho \mapsto I(\rho|\mu)$}.
\end{proof}

\subsubsection{Convergence in KL-divergence under LSI}
While the previous result was provided for any the target density $\pi \in \mathcal{P}(\mathbb{R}^d)$, in this section, we provide improved convergence results of the R-SVGF under the assumption that the $\pi$ further satisfies the Log-Sobolev Inequality. Recall that we say that $\pi$ satisfies the Log-Sobolev inequality with constant $\lambda>0$ if for all $\mu\in \mc{P}(\mb{R}^d)$: 
\begin{equation*}%\label{eq:log-Sobolev}
    \KL(\mu|\pi)\le \frac{1}{2\lambda} I(\mu|\pi).
\end{equation*}

Our first result below is a stronger version of the result in Theorem~\ref{thm:relation to Langevin FI}, under the assumption that the target $\pi$ satisfies LSI and Assumption~\ref{conjecture:LSI along langevin} on the initialization of the WGF. 
\begin{assumption}\label{conjecture:LSI along langevin}
The initial density $\mu_0$ is chosen so that the solution $(\mu_t)$ to \eqref{eq:Langevin FPE} also satisfies LSI with parameter $\lambda$, for all $t>0$.
\end{assumption}

Under the stronger assumption that the target density $\pi$ is strongly log-concave, following the arguments in~\cite[Theorem 8]{vempala2019rapid}, it is easy to show that Assumption~\ref{conjecture:LSI along langevin} is satisfied as long as $\mu_0$ is chosen such that it satisfies LSI. We conjecture that the same holds true even when the target density satisfies LSI and additional mild smoothness assumptions (i.e., LSI is preserved along the trajectory as long as the initial density $\mu_0$ satisfies LSI, presumably with additional milder assumptions). However, a proof of this conjecture has eluded us thus far.

\begin{theorem}[Relation to the WGF under LSI]\label{thm:relation to Langevin} Assume $\pi$ satisfies the log-Sobolev inequality with parameter $\lambda$. Let $(\rho_t)$ be the solution to \eqref{eq:regularized SVGD mf PDE}. Let $(\mu_t)$ be the solution to the WGF, defined in~\eqref{eq:Langevin FPE}, with $\mu_0$ satisfying Assumption~\ref{conjecture:LSI along langevin}. For any $t>0$, suppose there exists $\gamma_t\in(0,\frac{1}{2}]$ such that $ \lv \mathfrak{I}(\rho_t, \gamma_t) \rv_{L_2^d(\rho_t)}<\infty.$ Then, for any $T\in (0,\infty)$, we have
\begin{align}\label{eq:relation to Langevin}
    \KL(\rho_T|\mu_T)\le e^{-3\lambda T/2 }\KL(\rho_0|\mu_0)+%\lv k \rv_\infty^2
     \int_0^T \textcolor{black}{(\lambda_{1,t}\vee 1)^2}\nu^{2\gamma_t}\textcolor{black}{(1-\nu)^{-2\gamma_t}} e^{-3\lambda(T-t)/2}  \lv \mathfrak{I}(\rho_t, \gamma_t) \rv_{L_2^d(\rho_t)}^2 dt,
\end{align}
where $\lambda_{1,t}$ is the largest eigenvalue of $\iota_{k,\rho_t}\iota^*_{k,\rho_t}$ for all $t\ge 0$.
\end{theorem}

\begin{proof}[Proof of Theorem~\ref{thm:relation to Langevin}]
\label{pf:relation to Langevin} 
Following the same arguments as in the proof of Theorem \ref{thm:relation to Langevin FI}, we obtain that for any $t>0$, 
% there exists $\gamma_t\in(0,1/2]$ and $h_t=  \mathfrak{I}(\rho_t, \gamma_t) \in L_2^d(\rho_t)$ such that $\nabla \log \frac{\rho_t}{\pi}=(\iota_{k,\rho_t}\iota_{k,\rho_t}^*)^{\gamma_t} h_t$ and we get
\begin{align*}
    \frac{d}{dt}\KL(\rho_t|\mu_t)&\le -\frac{3}{4} I(\rho_t|\mu_t)+\sum_{i=1}^\infty \frac{(1-\lambda_i)^2\nu^2}{\left((1-\nu)\lambda_i+\nu\right)^2} \left| \leftlangle \nabla \log \frac{\rho_t}{\pi},e_i \rightrangle_{L_2(\rho_t)}\right|^2\\
    &\le -\frac{3}{4}I(\rho_t|\mu_t)+ %\lv k \rv_\infty^2 
    \textcolor{black}{(\lambda_{1,t}\vee 1)^2}\nu^{2\gamma_t}\textcolor{black}{(1-\nu)^{-2\gamma_t}} \nu^{2\gamma_t}  \lv \mathfrak{I}(\rho_t, \gamma_t) \rv_{L_2^d(\rho_t)}^2.
\end{align*}
Hence, under Assumption~\ref{conjecture:LSI along langevin} we obtain
\begin{align*}
    \frac{d}{dt}\KL(\rho_t|\mu_t)\le -\frac{3\lambda}{2}\KL(\rho_t|\mu_t)+%\lv k \rv_\infty^2
    \textcolor{black}{(\lambda_{1,t}\vee 1)^2}\nu^{2\gamma_t}\textcolor{black}{(1-\nu)^{-2\gamma_t}}  \lv \mathfrak{I}(\rho_t, \gamma_t) \rv_{L_2^d(\rho_t)}^2.
\end{align*}
Finally, \eqref{eq:relation to Langevin} follows from the Gronwall's inequality.
\end{proof}

Our second result is a stronger version of the result in Theorem~\ref{thm:Fisher convergence continuous}, under the assumption that the target distribution $\pi$ satisfies LSI. We remark that convergence to equilibrium of the related WGF under various functional inequalities is a well-studied topic. We refer the interested reader to~\cite{bakry2014analysis} for a detailed overview.

\begin{theorem}[Decay of KL-divergence under LSI]\label{thm:decay of KL under LSI} Assume that $\pi$ satisfies the log-Sobolev inequality with $\lambda>0$.  Let $(\rho_t)$ be the solution to \eqref{eq:regularized SVGD mf PDE}. For any $t>0$, suppose there exists $\gamma_t\in(0,\frac{1}{2}]$ such that $ \lv \mathfrak{I}(\rho_t, \gamma_t) \rv_{L_2^d(\rho_t)}<\infty.$ Then, %\textcolor{red}{[JL: ``Then''?]} 
for any $T\in (0,\infty)$, we have
\begin{align*}%\label{eq:decay of KL under LSI}
    \KL(\rho_T|\pi)\le e^{-2(1-\nu)^{-1}\lambda T}\KL(\rho_0|\pi)+ \int_0^T \nu^{2\gamma_t}(1-\nu)^{-2\gamma_t-1}  \lv \mathfrak{I}(\rho_t, \gamma_t) \rv_{L_2^d(\rho_t)}^2 e^{2(1-\nu)^{-1}\lambda(t-T)}dt.
\end{align*}
\end{theorem}
\begin{proof}[Proof of Theorem~\ref{thm:decay of KL under LSI}]
\label{pf:decay of KL under LSI}From the proof of Theorem \ref{thm:Fisher convergence continuous}, we have
\begin{align*}
    \frac{d}{dt}\KL(\rho_t|\pi)&\le -(1-\nu)^{-1}I(\rho_t|\pi)+(1-\nu)^{-1-2\gamma_t}\nu^{2\gamma_t}  \lv \mathfrak{I}(\rho_t, \gamma_t) \rv_{L_2^d(\rho_t)}^2\\
    &\le -2(1-\nu)^{-1}\lambda \KL(\rho_t|\pi)+(1-\nu)^{-1-2\gamma_t}\nu^{2\gamma_t}  \lv \mathfrak{I}(\rho_t, \gamma_t) \rv_{L_2^d(\rho_t)}^2,
\end{align*}
where the last inequality follows the log-Sobolev inequality. The final statement follows from Gronwall's inequality.
\end{proof}
\begin{remark}[Exponential Decay of KL-divergence]\label{rem:exponential decay of KL} Yet another way to state the above result is via the introducing the following  regularized Stein-LSI, similar to the introduction of Stein-LSI in~\cite{duncan2019geometry}. However, the introduction of Stein-LSI is quite restrictive in the sense that it couples assumptions on the target and the chosen RKHS. This makes verifying the conditions more delicate. To counter this effect, we now introduce the notion of Regularized Stein-LSI. We say that $\pi\in\mc{P}(\mb{R}^d)$ satisfies the regularized Stein log-Sobolev inequality with constant $\lambda>0$ if for all $\mu\in \mc{P}(\mb{R}^d)$: 
\begin{equation}\label{eq:regularized Stein log-Sobolev}
    \KL(\mu|\pi)\le \frac{1}{2\lambda} I_{\nu,Stein}(\mu|\pi).
\end{equation}
An advantage of the above condition is that, as $\nu \to 0$ the regularized Stein-LSI inequality becomes equivalent to the standard LSI inequality. Under the condition that the target density $\pi$ satisfies~\eqref{eq:regularized Stein log-Sobolev}, and letting $(\rho_t)$ be the solution to \eqref{eq:regularized SVGD mf PDE}, it holds that
\begin{align}\label{eq:KL exponential decay}
    \KL(\rho_t|\pi) \le e^{-2\lambda t} \KL(\rho_0|\pi).
\end{align}
The proof of \eqref{eq:KL exponential decay} follows immediately from Proposition \ref{cor:KL derivative} and \eqref{eq:regularized Stein log-Sobolev}. 
\end{remark}

%%%%%%%%%%%%%%%%%%%%%%%%%%%%%%%%%%%%%%%%%%%%%%%%%%%%%%%%%%%%%%%%%%%%%%%%%%%%%%%%%%%%%%%%%%
\subsection{Convergence results for Time-discretized R-SVGF}

In this section, we analyze the convergence properties of the time-discretized R-SVGF in \eqref{eq:pop limit regularized SVGD}. To do so, we require the following additional assumptions. 
\begin{myassump}{A2}\label{ass:decay of KL under pop limit}
The following conditions hold: 
\begin{itemize}
    \item [(1)] There exists a constant $B>0$ such that $\lv \nabla_1 k(x,\cdot) \rv_{\mc{H}_k^d}\le B$ for all $x\in \mb{R}^d$.
    \item [(2)] The potential function $V:\mb{R}^d\to \mb{R}$ is twice continuously differentiable and gradient Lipschitz with parameter $L$.
    \item [(3)] Along the time discretization \eqref{eq:pop limit regularized SVGD},  $I(\rho^n|\pi)<\infty$ for all fixed $n\ge 0$.
\end{itemize}
\end{myassump}
The smoothness assumptions in points (1) and (2) of Assumption~\ref{ass:decay of KL under pop limit} are commonly required in analyzing any discrete-time algorithms, albeit deterministic~\cite{korba2020non, salim2022convergence} or randomized~\cite{vempala2019rapid, chewi2021analysis, balasubramanian2022towards}. While it could be relaxed (see, for example,~\cite{sun2022note}), in general it is impossible to completely avoid them as in the case of analyzing the corresponding flows. Before stating our results, we also introduce some convenient notations. We let 
\begin{align}\label{eq:sandr}
\mathfrak{S}_n:=\left(\sup_{i} \frac{{\lambda_i^{(n)}}^{1+2\gamma_n}}{\left((1-\nu_{n+1}){\lambda_i^{(n)}}+\nu_{n+1}\right)^2}\right) \quad~\text{and}\quad R_n:=\lv  \mathfrak{I}(\rho^n, \gamma_n)  \rv_{L_2^d(\rho^n)},
\end{align}
where the sequence $\{\lambda_i^{(n)}\}_{i\ge 1}$ corresponds to the positive eigenvalues of the operator $\iota_{k,\rho^n}\iota_{k,\rho^n}^*$ in the order of decreasing values for all $n\ge 0$.

\begin{theorem}[Convergence in Fisher Divergence]\label{thm:decay of Fisher along pop limit} Suppose Assumption \ref{ass:decay of KL under pop limit} holds. Let $(\rho^n)$ be the time discretization of the R-SVGF described in \eqref{eq:pop limit regularized SVGD} with initial condition $\rho^0=\rho_0$ such that $\KL(\rho_0|\pi)\le R$. For each $n$, suppose that $\nu_{n+1}$ and the step-size $h_{n+1}$ are chosen such that, 
\begin{equation}\label{eq:time nu constraint decay of Fisher pop limit}
\begin{aligned}
    h_{n+1}&< \min\left\{
    \frac{1-\nu_{n+1}}{L}, \frac{\alpha-1}{  \alpha B R_n \sqrt{ \mathfrak{S}_n}}
    \right\},
\end{aligned}
\end{equation}
where $\alpha\in(1,2]$ is some constant, and suppose that there exists $\gamma_n\in(0,\frac{1}{2}]$, such that $\mathfrak{I}(\rho^n, \gamma_n) \in L_2^d(\rho^n)$. Then,
\begin{equation}\label{eq:decay of Fisher along pop limit thm}
    \sum_{n=0}^\infty \frac{h_{n+1}}{2(1-\nu_{n+1})}I(\rho^n|\pi)\le \sum_{n=0}^\infty \nu_{n+1}^{2\gamma_n}(1-\nu_{n+1})^{-2\gamma_n-1} h_{n+1} \left(1+\frac{1}{2}\nu_{n+1}^{-1}\alpha^2 B^2 h_{n+1} \right) R_n^2 +R.
    \end{equation}
\end{theorem}

Before proving Theorem~\ref{thm:decay of Fisher along pop limit}, we first prove the following intermediate result. \textcolor{black}{For the proofs, we let $(\lambda_i^{(n)},e_i^{(n)})_{i=1}^\infty$to denote the set of eigenvalues and eigenvectors of the operator $\iota_{k,\rho^n}\iota_{k,\rho^n}^*$, with $\lambda_1^{(n)}\ge \lambda_2^{(n)}\ge \cdots >0$.}

\begin{lemma}\label{lem:bound for second order term} For each $n\ge 1$, define  $g=\mc{D}_{\nu_{n+1},\rho^n}\nabla \log \frac{\rho^n}{\pi}$. %$\phi_t=I_d-t g$ for all $t\in [0,h_{n+1}]$ and $\Tilde{\rho}_t=(\phi_t)_{\#\rho^n}$.
Under the conditions in Theorem \ref{thm:decay of Fisher along pop limit}, we have that, for any $x\in\mb{R}^d$ and $t\in [0,h_{n+1}]$,
\begin{align*}%\label{eq:bound for HS and op norm}
    \lv \nabla g(x) \rv_{HS}^2 \le B^2 R_n^2  \mathfrak{S}_n ~~~\text{and}~~~\lv  (id-t \nabla g(x))^{-1} \rv_2 \le \alpha.
\end{align*}
\end{lemma}
\begin{proof}[Proof of Lemma~\ref{lem:bound for second order term}]
Since for each $n$, there exists $\gamma_n\in(0,1/2]$ and a function $ h=\mathfrak{I}(\rho^n, \gamma_n)  \in L_2^d(\rho^n)$ such that $(\iota_{k,\rho^n}\iota_{k,\rho^n}^*)^{2\gamma_n} h_j=\partial_j \log \frac{\rho^n}{\pi} $, where $h_j$ is the $j$-th component of the function value of $h$, we have
\begin{align*}
     \lv \nabla g(x) \rv_{HS}^2&=\sum_{j,l=1}^d \left|\frac{\partial g_j(x) }{\partial x_l}\right|^2\\
    &=\sum_{j,l=1}^d \left(\sum_{i=1}^\infty \frac{\lambda_i^{(n)}}{(1-\nu_{n+1})\lambda_i^{(n)}+\nu_{n+1}} \leftlangle \partial_j \log \frac{\rho^n}{\pi},e_i^{(n)} \rightrangle_{L_2(\rho^n)}\partial_l e_i^{(n)}(x) \right)^2 \\
    &=\sum_{j,l=1}^d \left(\sum_{i=1}^\infty \frac{{\lambda_i^{(n)}}^{1+\gamma_n}}{(1-\nu_{n+1})\lambda_i^{(n)}+\nu_{n+1}} \leftlangle h_j,e_i^{(n)} \rightrangle_{L_2(\rho^n)}\partial_l e_i^{(n)}(x) \right)^2\\
    &\le \sum_{j,l=1}^d \left( \sum_{i=1}^\infty \leftlangle h_j,e_i^{(n)} \rightrangle_{L_2(\rho^n)}^2 \right)\left( \sum_{i=1}^\infty \frac{{\lambda_i^{(n)}}^{2+2\gamma_n}}{\left((1-\nu_{n+1}){\lambda_i^{(n)}}+\nu_{n+1}\right)^2}\left| \partial_l e_i^{(n)}(x) \right|^2 \right)\\
    &=\left( \sum_{i=1}^\infty \left| \leftlangle h,e_i^{(n)} \rightrangle_{L_2(\rho^n)} \right|^2 \right) \left( \sum_{i=1}^\infty \frac{{\lambda_i^{(n)}}^{2+2\gamma_n}}{\left((1-\nu_{n+1}){\lambda_i^{(n)}}+\nu_{n+1}\right)^2}\left| \nabla e_i^{(n)}(x) \right|^2 \right)\\
    &\le \sup_{i} \left( \frac{{\lambda_i^{(n)}}^{1+2\gamma_n}}{\left((1-\nu_{n+1}){\lambda_i^{(n)}}+\nu_{n+1}\right)^2}  \right) \lv \nabla_1 k(x,\cdot) \rv_{\mc{H}_k^d}^2 R_n^2 \\
    &\le B^2 R_n^2 \sup_{i} \left( \frac{{\lambda_i^{(n)}}^{1+2\gamma_n}}{\left((1-\nu_{n+1}){\lambda_i^{(n)}}+\nu_{n+1}\right)^2}  \right).
   % &\le \frac{1}{2} B^2 I(\rho^n|\pi) \nu_{n+1}^{-2\gamma_n}(1-\nu_{n+1})^{2\gamma_n} \sup_{i} \left( \frac{{\lambda_i^{(n)}}^{1+2\gamma_n}}{\left((1-\nu_{n+1}){\lambda_i^{(n)}}+\nu_{n+1}\right)^2}  \right) 
\end{align*}
In the above, the first inequality follows from Cauchy-Schwartz inequality, the second inequality follows from the fact that 
$$
\sum_{i=1}^\infty \textcolor{black}{\lambda_i^{(n)}}\left|\nabla e_i^{(n)}(x)\right|^2=\sum_{i=1}^\infty \langle \nabla_1 k(x,\cdot), \textcolor{black}{\sqrt{\lambda_i^{(n)}}}e_i^{(n)}  \rangle_{\mc{H}_k^d}^2 =\lv \nabla_1 k(x,\cdot) \rv_{\mc{H}_k^d}^2,
$$
and the last inequality follows from Assumption \ref{ass:decay of KL under pop limit}. Meanwhile, since $\lv \nabla g(x) \rv_{2}\le \lv \nabla g(x) \rv_{HS}$ for all $x\in \mb{R}^d$, for every $t\in [0,h_{n+1}]$, we have
\begin{align*}
     \lv  (id-t \nabla g(x))^{-1} \rv_{2} &\le \sum_{m=0}^\infty \lv t \nabla g(x) \rv_{2}^m \le \sum_{m=0}^\infty \lv t \nabla g(x) \rv_{HS}^m \\
     &\le \sum_{m=0}^\infty \left( h_{n+1} B R_n \sup_{i}\left( \frac{{\lambda_i^{(n)}}^{1+2\gamma_n}}{\left((1-\nu_{n+1}){\lambda_i^{(n)}}+\nu_{n+1}\right)^2}  \right)^{\frac{1}{2}}  \right)^m\\
     %&\le \sum_{m=0}^\infty \left(\frac{\sqrt{2}}{2} t B I(\rho^n|\pi)^{\frac{1}{2}}\nu_{n+1}^{-\gamma_n}(1-\nu_{n+1})^{\gamma_n} \sup_{i}\left( \frac{{\lambda_i^{(n)}}^{1+2\gamma_n}}{\left((1-\nu_{n+1}){\lambda_i^{(n)}}+\nu_{n+1}\right)^2}  \right)^{\frac{1}{2}}\right)^m \\
     &\le \sum_{m=0}^\infty \left(\frac{\alpha-1}{\alpha}\right)^m
     =\alpha.
\end{align*}
where the last inequality follows from \eqref{eq:time nu constraint decay of Fisher pop limit}.
\end{proof}

\begin{proof}[Proof of Theorem~\ref{thm:decay of Fisher along pop limit}]

 We start from studying the single step along \eqref{eq:pop limit regularized SVGD}. In the following analysis, for each $n\ge 1$, we denote $g=\mc{D}_{\nu_{n+1},\rho^n}\nabla \log \frac{\rho^n}{\pi}$, $\phi_t(x)=x-t g(x)$ for all $x\in\mb{R}^d$, $t\in [0,h_{n+1}]$ and $\Tilde{\rho}_t=(\phi_t)_{\#}\rho^n$. Therefore, we have 
 \begin{align*}
    \rho^n=\Tilde{\rho}_0 \quad \text{and} \quad \rho^{n+1}=({\phi_{h_{n+1}}})_{\#}\rho^n=\Tilde{\rho}_{h_{n+1}}.
\end{align*}
The following analysis is motivated by \cite[Proposition 3.1]{salim2022convergence}.  According to \cite[Theorem 5.34]{villani2021topics}, the velocity field ruling the evolution of $\Tilde{\rho}_t$ is $\omega_t\in L_2^d(\Tilde{\rho}_t)$ and $\omega_t(x)=-g(\phi_t^{-1}(x))$. Define $\psi(t)=\KL(\Tilde{\rho}_t|\pi)$. According to the chain rule in \cite[Section 8.2]{villani2021topics},
\begin{align*}
    \psi'(t)&= \leftlangle \nabla_{W_2} \KL(\Tilde{\rho}_t|\pi), \omega_t \rightrangle_{L_2^d(\Tilde{\rho}_t)},\\
    \psi''(t)&=\leftlangle \omega_t, \text{Hess}_{\KL(\cdot|\pi)}(\Tilde{\rho}_t)\omega_t \rightrangle_{L_2^d(\Tilde{\rho}_t)}.
\end{align*}
where $\text{Hess}_{\KL(\cdot|\pi)}(\Tilde{\rho}_t)$ is the Wasserstein Hessian of $\KL(\cdot |\pi)$ at $\Tilde{\rho}_t$. For any $\mu\in \mc{P}(\mb{R}^d)$ and any $v$ in the Wasserstein tangent space at $\mu$, the Wasserstein Hessian is given by~\textcolor{black}{\cite{korba2020non},}
\begin{align*}
     \leftlangle v, \text{Hess}_{\KL(\cdot|\pi)}(\mu)v \rightrangle_{L_2^d(\mu)} &=\leftlangle v, \nabla^2 V v \rightrangle_{L_2^d(\mu)}+\mb{E}_{\mu}[\lv  \nabla v(X) \rv_{HS}^2].
\end{align*}Therefore we can expand the difference in KL-divergence between the two consecutive iterations as
\begin{align}\label{eq:deacay of KL pop limit difference}
    &~~~\quad\psi(h_{n+1})-\psi(0)\nonumber\\
    &=\psi'(0)h_{n+1}+\int_0^{h_{n+1}} (h_{n+1}-t)\psi''(t)dt \nonumber\\
    &=-h_{n+1}\leftlangle \nabla_{W_2} \KL(\rho^n |\pi), g \rightrangle_{L_2^d(\rho^n)}+\int_0^{h_{n+1}}(h_{n+1}-t)\langle \omega_t, \text{Hess}_{\KL(\cdot|\pi)}(\Tilde{\rho}_t)\omega_t \rangle_{L_2^d(\Tilde{\rho}_t)} dt.
\end{align}
The first term on the right-hand side of \eqref{eq:deacay of KL pop limit difference} can be studied via the spectrum of the operator $\iota_{k,\rho^n}\iota_{k,\rho^n}^*$.
\begin{align*}%\label{eq:KL decay pop limit bound 1}
    &\quad-h_{n+1}\leftlangle \nabla_{W_2} \KL(\rho^n |\pi), g \rightrangle_{L_2^d(\rho^n)}\nonumber\\
    &=-h_{n+1}\leftlangle \nabla \log \frac{\rho^n}{\pi}, \left((1-\nu_{n+1})\iota_{k,\rho^n}^* \iota_{k,\rho^n}+\nu_{n+1} I_d\right)^{-1}\iota_{k,\rho^n}^* \nabla \log \frac{\rho^n}{\pi} \rightrangle_{L_2^d(\rho^n)} \nonumber\\
    &=-h_{n+1}\sum_{i=1}^\infty \frac{\lambda_i^{(n)}}{(1-\nu_{n+1})\lambda_i^{(n)}+\nu_{n+1}}\left|\leftlangle \nabla \log \frac{\rho^n}{\pi}, e_i^{(n)} \rightrangle_{L_2(\rho^n)}\right|^2 \nonumber\\
    &=-h_{n+1} I_{\nu_{n+1},Stein}(\rho^n|\pi).
\end{align*}
Since $\Tilde{\rho}_t=(\phi_t)_{\#}\rho^n$, for any function $h$ we have $\mb{E}_{X\sim \Tilde{\rho}_t}\left[ h(X) \right]=\mb{E}_{Y\sim \rho{^n}}\left[ h(\phi_t(Y)) \right]$. Hence, for the second term on the right side of \eqref{eq:deacay of KL pop limit difference}, we obtain
\begin{align*}
     \leftlangle \omega_t, \text{Hess}_{\KL(\cdot|\pi)}(\Tilde{\rho}_t)\omega_t \rightrangle_{L_2^d(\Tilde{\rho}_t)} &=\leftlangle \omega_t, \nabla^2 V \omega_t \rightrangle_{L_2^d(\Tilde{\rho}_t)}+\mb{E}_{\Tilde{\rho}_t}[\lv \nabla \omega_t(x) \rv_{HS}^2] \\
    &=\leftlangle g(\phi_t^{-1}), \nabla^2 V g(\phi_t^{-1}) \rightrangle_{L_2^d(\Tilde{\rho}_t)}+\mb{E}_{\rho^n} [ \lv \nabla \omega_t \circ \phi_t(x) \rv_{HS}^2 ] \\
    &=\mb{E}_{\rho^n}\left[ g(x)^T \nabla V^2(\phi_t(x)) g(x)\right] +\mb{E}_{\rho^n} \left[ \lv \nabla g(x) (\nabla \phi_t(x))^{-1} \rv_{HS}^2 \right]\\
    &\le L\lv g \rv_{L_2^d(\rho^n)}^2+\mb{E}_{\rho^n} \left[ \lv \nabla g(x)  (\nabla \phi_t(x))^{-1} \rv_{HS}^2 \right],
\end{align*}
where the last inequality follows from Assumption \ref{ass:decay of KL under pop limit}-(2). Therefore, we obtain
\begin{align*}
    \KL(\rho^{n+1}|\pi)-\KL(\rho^n|\pi)&\le -h_{n+1}I_{\nu_{n+1},Stein}(\rho^n|\pi)+\frac{Lh_{n+1}^2}{2}\lv g \rv_{L_2^d(\rho^n)}^2\\
    &\quad\quad +\frac{h_{n+1}^2}{2}\max_{t\in [0,h_{n+1}]}\mb{E}_{\rho^n} \left[ \lv \nabla g(x) (\nabla \phi_t(x))^{-1} \rv_{HS}^2 \right],
\end{align*}
where
\begin{align*}
    \lv g \rv_{L_2^d(\rho^n)}^2&=\lv \mc{D}_{\nu_{n+1},\rho^n} \nabla \log \frac{\rho^n}{\pi} \rv_{L_2^d(\rho^n)}^2 \\
     &=\lv \left((1-\nu_{n+1})\iota_{k,\rho^n}\iota_{k,\rho^n}^*+\nu_{n+1}I_d\right)^{-1}\iota_{k,\rho^n}\iota_{k,\rho^n}^* \nabla \log \frac{\rho^n}{\pi} \rv_{L_2^d(\rho^n)}^2 \\
     &=\sum_{i=1}^d \left( \frac{{\lambda_i^{(n)}}}{(1-\nu_{n+1})\lambda_i^{(n)}+\nu_{n+1}} \right)^2 \left| \leftlangle \nabla \log \frac{\rho^n}{\pi}, e_i^{(n)} \rightrangle_{L_2(\rho^n)} \right|^2\\
     &\le (1-\nu_{n+1})^{-2}I(\rho_n|\pi),
\end{align*}
with $(\lambda_i^{(n)}, e_i^{(n)})_{i=1}^\infty$ being the sequence of eigenvalues and eigenvectors to the operator $\iota_{k,\rho^n}\iota_{k,\rho^n}^*$ such that $\lambda_1^{(n)}\ge \cdots \ge \lambda_i^{(n)}\ge \cdots>0$ and $(e_i^{(n)})_{i=1}^\infty$ is an orthonormal basis of $L_2(\rho^n)$. According to Lemma \ref{lem:bound for second order term} and Assumption \ref{ass:decay of KL under pop limit},
\begin{align*}
    &\lv \nabla g(x) \rv_{HS}^2\\
\le &\sup_{i} \left(\frac{{\lambda_i^{(n)}}^{1+2\gamma_n}}{\left((1-\nu_{n+1})\lambda_i^{(n)}+\nu_{n+1}\right)^2}\right)B^2 R_n^2 \\
    \le&  \sup_{i} \left( \frac{\nu_{n+1}^{2\gamma_n-1}}{(1-\nu_{n+1})^{2\gamma_n+1}} \left(\frac{(1-\nu_{n+1})\lambda_i^{(n)}}{(1-\nu_{n+1})\lambda_i^{(n)}+\nu_{n+1}}\right)^{1+2\gamma_n}\left(\frac{\nu_{n+1}}{(1-\nu_{n+1})\lambda_i^{(n)}+\nu_{n+1}}\right)^{1-2\gamma_n} \right) B^2 R_n^2\\
    \le& \nu_{n+1}^{2\gamma_n-1}(1-\nu_{n+1})^{-2\gamma_n-1} B^2 R_n^2,
\end{align*}
and furthermore according to Lemma \ref{lem:bound for second order term}, $\lv(id-t \nabla g(x))^{-1}\rv_2^2\le \alpha^2$. Therefore, we get 
\begin{align*}
    \KL(\rho^{n+1}|\pi)-\KL(\rho^n|\pi)&\le -h_{n+1}I_{\nu_{n+1},Stein}(\rho^n|\pi)+\frac{L h_{n+1}^2(1-\nu_{n+1})^{-2}}{2}I(\rho^n|\pi)\\
    &\quad +\frac{1}{2}\alpha^2B^2\nu_{n+1}^{2\gamma_n-1}(1-\nu_{n+1})^{-2\gamma_n-1}R_n^2 h_{n+1}^2\\
    &\le -h_{n+1}(1-\nu_{n+1})^{-1} \left( 1-\frac{L h_{n+1}(1-\nu_{n+1})^{-1}}{2} \right) I(\rho^n|\pi)\\
    &\quad +h_{n+1} \nu_{n+1}^{2\gamma_n}(1-\nu_{n+1})^{-2\gamma_n-1} R_n^2\left( 1+\frac{1}{2}h_{n+1} \nu_{n+1}^{-1} \alpha^2 B^2 \right)\\
    &\le -\frac{1}{2}h_{n+1}(1-\nu_{n+1})^{-1} I(\rho^n|\pi)\\
    &\quad +h_{n+1} \nu_{n+1}^{2\gamma_n}(1-\nu_{n+1})^{-2\gamma_n-1} R_n^2\left( 1+\frac{1}{2}h_{n+1} \nu_{n+1}^{-1} \alpha^2 B^2 \right),
\end{align*}
where the last inequality follows from \eqref{eq:time nu constraint decay of Fisher pop limit} and the second inequality follows from the fact that
\begin{align*}
    I(\rho^n|\pi)-\textcolor{black}{(1-\nu_{n+1})}I_{\nu_{n+1},Stein}(\rho^n|\pi)\le \nu_{n+1}^{2\gamma_n}(1-\nu_{n+1})^{-2\gamma_n} R_n^2,
\end{align*}
which is proved in Proposition~\ref{lem:relation between I and regularized Istein}.
 Lastly, summing over $n$ and we obtain
 \begin{align*}
     \sum_{n=0}^\infty \frac{h_{n+1}}{2(1-\nu_{n+1})} I(\rho^n|\pi)&\le \sum_{n=0}^\infty\left( \KL(\rho^n|\pi)-\KL(\rho^{n+1}) \right)\\
     &\ +\sum_{n=0}^\infty h_{n+1} \nu_{n+1}^{2\gamma_n}(1-\nu_{n+1})^{-2\gamma_n-1} R_n^2\left( 1+\frac{1}{2}h_{n+1} \nu_{n+1}^{-1} \alpha^2 B^2 \right) \\
     &\le \KL(\rho^0|\pi)+\sum_{n=0}^\infty h_{n+1} \nu_{n+1}^{2\gamma_n}(1-\nu_{n+1})^{-2\gamma_n-1} R_n^2\left( 1+\frac{1}{2}h_{n+1} \nu_{n+1}^{-1} \alpha^2 B^2 \right),
 \end{align*}
 where the last inequality follows from the fact that KL divergence is non-negative. Therefore, \eqref{eq:decay of Fisher along pop limit thm} is proved. 
\end{proof}

\begin{remark}
We emphasize that the above result does not make any assumptions on the target density $\pi$, except for \textcolor{black}{$\pi\in \mc{P}(\mb{R}^d)$ and} the Lipschitz gradient assumption. In particular, it  holds for multi-modal densities. However, the metric of convergence is the weaker Fisher information metric. 
\end{remark}

We now provide a stronger result under the LSI assumption.

\begin{theorem}\label{thm:decay of KL along pop limit log-Sobolev version} Suppose Assumption \ref{ass:decay of KL under pop limit} holds and $\pi$ satisfies the log-Sobolev inequality with parameter $\lambda$. Let $(\rho^n)$ be as described in \eqref{eq:pop limit regularized SVGD} with initial condition $\rho^0=\rho_0$ such that $\KL(\rho_0|\pi)\le R$. Assume the regularization parameter and the step-size parameters are chosen such that for all $n\ge 0$, they satisfy 
\begin{align}\label{eq:time nu constraint decay of KL pop limit}
\begin{aligned}
      h_{n+1}\le \min\left\{ \frac{1-\nu_{n+1}}{L},  \frac{\alpha-1}{  \alpha B R_n \sqrt{ \mathfrak{S}_n}},\textcolor{black}{\frac{2\nu_{n+1}}{\alpha^2B^2}} ,  \frac{2(1-\nu_{n+1})}{\lambda} \right\}, ~~~~ \frac{\nu_{n+1}}{1-\nu_{n+1}}\le \left( \frac{I(\rho^n|\pi)}{2R_n^2} \right)^{\frac{1}{2\gamma_n}},
\end{aligned}
\end{align}
where $\alpha\in(1,2]$ is a constant, $\gamma_n\in(0,\frac{1}{2}]$, and $\mathfrak{I}(\rho^n, \gamma_n) \in L_2^d(\rho^n)$. 
Then, for all $n\ge 1$,
\begin{align}\label{eq:decay of KL along pop limit thm}
    \KL(\rho^n|\pi) \le R~ \prod_{i=1}^n \left(1-\frac{1}{2}\lambda (1-\nu_i)^{-1}h_{i}\right). 
\end{align}
% Here, in \eqref{eq:time nu constraint decay of KL pop limit}, for each $n$, there exists a $\gamma_n\in(0,\frac{1}{2}]$,   $\mathfrak{I}(\rho^n, \gamma_n) \in L_2^d(\rho^n)$ and $\alpha\in(1,2)$ is some constant.
\end{theorem}

\begin{proof}[Proof of Theorem~\ref{thm:decay of KL along pop limit log-Sobolev version}]
From the proof of Theorem \ref{thm:decay of Fisher along pop limit}, we can bound the difference in KL-divergence between two consecutive iterations by
\begin{align*}
    &~\quad\KL(\rho^{n+1}|\pi)-\KL(\rho^n|\pi)\\
    &\le -\frac{1}{2}h_{n+1}(1-\nu_{n+1})^{-1} I(\rho^n|\pi)
     +h_{n+1} \nu_{n+1}^{2\gamma_n}(1-\nu_{n+1})^{-2\gamma_n-1} R_n^2\left( 1+\frac{1}{2}h_{n+1} \nu_{n+1}^{-1} \alpha^2 B^2 \right) \\
    &=-\frac{1}{4}h_{n+1}(1-\nu_{n+1})^{-1}I(\rho^n|\pi) \left(2-\frac{\nu_{n+1}^{2\gamma_n}}{(1-\nu_{n+1})^{2\gamma_n}}\frac{R_n^2 \left( 1+\frac{1}{2}h_{n+1} \nu_{n+1}^{-1} \alpha^2 B^2 \right)}{I(\rho^n|\pi)} \right) \\
    &\le -\frac{1}{4}h_{n+1}(1-\nu_{n+1})^{-1}I(\rho^n|\pi) \left(2-\frac{\nu_{n+1}^{2\gamma_n}}{(1-\nu_{n+1})^{2\gamma_n}}\frac{\textcolor{black}{2}R_n^2 }{I(\rho^n|\pi)} \right) \\
    &\le -\frac{1}{4}h_{n+1}(1-\nu_{n+1})^{-1}I(\rho^n|\pi),
\end{align*}
where \textcolor{black}{the second inequality follows from the fact that $\frac{1}{2}h_{n+1} \nu_{n+1}^{-1} \alpha^2 B^2\le 1$, and} the last inequality follows from \eqref{eq:time nu constraint decay of KL pop limit}.
\iffalse
Hence, according to \eqref{eq:deacay of KL pop limit difference} we can bound the single step decay in KL-divergence as
\begin{align*}
    &\qquad \KL(\rho^{n+1}|\pi)-\KL(\rho^n|\pi)\\
    &\le -h_{n+1}I_{\nu_{n+1},Stein}(\rho^n|\pi)+\frac{h_{n+1}^2}{2}\max_{t\in [0,h_{n+1}]}  \leftlangle \omega_t, \text{Hess}_{\KL(\cdot|\pi)}(\Tilde{\rho}_t)\omega_t \rightrangle_{L_2^d(\Tilde{\rho}_t)}\\
    &\le -h_{n+1}\sum_{i=1}^\infty \frac{\lambda_i^{(n)}}{(1-\nu_{n+1})\lambda_i^{(n)}+\nu_{n+1}}\left( 1-\frac{L h_{n+1}}{2} \frac{{\lambda_i^{(n)}}}{(1-\nu_{n+1})\lambda_i^{(n)}+\nu_{n+1}} \right)\left| \leftlangle \nabla \log \frac{\rho^n}{\pi}, e_i^{(n)} \rightrangle_{L_2(\rho^n)} \right|^2\\
    &\quad +\frac{1}{4}h_{n+1}^2\alpha^2 B^2 {\nu_{n+1}}^{-2\gamma_n}(1-\nu_{n+1})^{2\gamma_n} I(\rho^n|\pi) \sup_{i} \left( \frac{{\lambda_i^{(n)}}^{1+2\gamma_n}}{\left({(1-\nu_{n+1})\lambda_i^{(n)}}+\nu_{n+1}\right)^2}  \right) \\
    &\le -\frac{1}{2}h_{n+1} I_{\nu_{n+1},Stein}(\rho^n|\pi)+\frac{1}{4}h_{n+1}^2\alpha^2 B^2 {\nu_{n+1}}^{-2\gamma_n}(1-\nu_{n+1})^{2\gamma_n} \sup_{i} \left( \frac{{\lambda_i^{(n)}}^{1+2\gamma_n}}{\left({(1-\nu_{n+1})\lambda_i^{(n)}}+\nu_{n+1}\right)^2}  \right)I(\rho^n|\pi)  \\
    &\le -\frac{1}{4}h_{n+1} \left( (1-\nu_{n+1})^{-1}-h_{n+1}\alpha^2 B^2\nu_{n+1}^{-2\gamma_n}(1-\nu_{n+1})^{2\gamma_n} \sup_{i} \left( \frac{{\lambda_i^{(n)}}^{1+2\gamma_n}}{\left({(1-\nu_{n+1})\lambda_i^{(n)}}+\nu_{n+1}\right)^2}  \right) \right) I(\rho^n|\pi)\\
    &\le -\frac{1}{8}(1-\nu_{n+1})^{-1}h_{n+1} I(\rho^n|\pi).
\end{align*}
\fi
 Last, since $\pi$ satisfies the log-Sobolev inequality with parameter $\lambda$, we get
\begin{align*}
    \KL(\rho^{n+1}|\pi)\le \left(1-\frac{1}{2}\lambda(1-\nu_{n+1})^{-1} h_{n+1}\right) \KL(\rho^n|\pi),
\end{align*}
and \eqref{eq:decay of KL along pop limit thm} follows from the above recursive inequality.
\end{proof}

\begin{remark}[Choice of $\{\nu_n\}_{n\ge 1}$] \textcolor{black}{From \eqref{eq:time nu constraint decay of KL pop limit} and~\eqref{eq:decay of KL along pop limit thm}, we observe that there is a trade-off in terms of $\{\nu_n\}_{n\ge 1}$, as smaller $\{\nu_n\}_{n\ge 1}$ will result in slower convergence of the time-discretized R-SVGF in KL-divergence. Indeed, if $\{\nu_n\}_{n\ge 1}$ are chosen to be small, \eqref{eq:time nu constraint decay of KL pop limit} requires the step-size $\{h_n\}_{n\ge 1}$ to be small as well. And, as shown in \eqref{eq:decay of KL along pop limit thm}, $\KL(\rho_n|\pi)$ decays at a slower rate when $\{h_n\}_{n\ge 1}$ are smaller. We also refer to Remark~\ref{gaussianexample} below and Remark~\ref{nustability} for further trade-offs with respect to the parameter $\{\nu_n\}_{n\ge 1}$. }
\end{remark}

\begin{remark} 
According to \eqref{eq:decay of KL along pop limit thm}, to reach an $\epsilon$-accuracy in KL-divergence, we need the number of iterations to be at least $n_\epsilon$ such that $\prod_{i=1}^{n_\epsilon} \left(1-\frac{1}{2}\lambda(1-\nu_i)^{-1} h_i\right)R\le \epsilon$. With the fact that $\log(1-x)<-x$ for all $x\in (0,1)$, we get $n_\epsilon$ satisfies 
\begin{align*}
\sum_{i=1}^{n_\epsilon} (1-\nu_i)^{-1} h_i\ge \frac{2}{\lambda} \log \left(\frac{R}{\epsilon}\right).
\end{align*}
Under \eqref{eq:time nu constraint decay of KL pop limit}, if we can choose the time step sizes $(h_i)_{i=1}^\infty$ to be a constant $h>0$, then we have $n_\epsilon = O(\log(R/\epsilon))$. For comparison, in Table~\ref{tab:comp}, we provide the iteration complexity results for different methods, to obtain $\text{KL}(\rho_n|\pi) \leq \epsilon $, under the assumption that the target $\pi$ satisfies LSI. 
\end{remark}

\iffalse
\kb{add discussion on $\nu$}

\textcolor{black}{
Notice that there is a trade-off on the choice of value for $(\nu_n)_{n\ge 1}$. For the convergence result in Theorem \ref{thm:decay of KL along pop limit log-Sobolev version}, $\nu_n$ can't be too large due to the constraint \eqref{eq:time nu constraint decay of KL pop limit}. On the other side, according to \eqref{eq:deacay of KL pop limit difference}, $\nu_n$ needs to be larger in order to get a faster convergence rate for $(\KL(\rho^n|\pi))_{n\ge 0}$. 
}
\fi

We also emphasize that prior results on the analysis of time-discretization of the SVGF under functional inequality assumptions are established only in the weaker Stein-Fisher information metric~\cite{korba2020non,salim2022convergence}. Our results above are established for the $\KL$-divergence and is more in line with similar results established for other randomized Monte Carlo algorithms~\cite{vempala2019rapid, chewi2021analysis, balasubramanian2022towards}. We end this section with the following remark on an illustrative example.

\begin{table}[t]
\label{tab:comp}
\begin{tabular}{ |c|c|c|c| }
\hline
 \textcolor{black}{Method} & Source & Type  & Iterations  \\ \hline \hline  
\textcolor{black}{SVGF} & NA &Deterministic& unknown \\ &&&\\
  \textcolor{black}{LMC} & \cite{vempala2019rapid, chewi2021analysis} & Randomized & $\mathcal{O}(\frac{1}{\epsilon}) $ \\ &&&\\
 
  \textcolor{black}{MALA} & NA& Randomized &unknown \\ &&&\\

    \textcolor{black}{Proximal sampler} & \cite{chen2022improved}& Randomized & $\mathcal{O} \left(\log^\lambda (\frac{1}{\epsilon})\right) $ \\ &&&\\

  \textcolor{black}{Regularized SVGF} & Theorem~\ref{thm:decay of KL along pop limit log-Sobolev version} & Deterministic &$\mathcal{O} \left(\log\left(\frac{1}{\epsilon}\right) \right)$  \\
\hline
\end{tabular}\vspace{2mm}
\caption{\textcolor{black}{The results for SVGF and Regularized SVGF are strictly non-algorithmic, as the stated results require that the initial density supplied to the method must have finite KL divergence to $\pi$. The results from~\cite{vempala2019rapid, chewi2021analysis, chen2022improved} are  for LMC, MALA and the proximal sampler are algorithmic, and are presented in a simplified manner to convey the dependency on the accuracy parameter $\epsilon$.} The result for the proximal sampler holds only in expectation. Currently, it is not clear how to obtain a high-probability result in KL-divergence; see~\cite{chen2022improved} for details.}
\end{table}

\begin{remark}[An illustrative example]\label{gaussianexample} 
 \textcolor{black}{Consider sampling from a Gaussian target $\pi=\mc{N}(0,Q)$, where $Q$ is strictly positive definite and  $k(x,y)=\langle x,y \rangle+1$, the linear kernel. This model has recently been studied in~\cite{liu2024towards} motivated by connections to Gaussian variational inference. Note that the above target $\pi$ satisfies LSI with $\lambda$ being the minimal eigenvalue of $Q$. Based on the results established in~\cite{liu2024towards}, our results on the R-SVGF and the time-discretized R-SVGF can be interpreted as follows. Before we proceed, we point out a subtle fact. To establish existence and uniqueness later in Section~\ref{sec:exist} we assume that the kernel is bounded, which is not satisfied by the linear kernel. However, based on an explicit calculation, existence and uniqueness results were shown for the linear kernel in~\cite{liu2024towards}.}\\

\textcolor{black}{\textbf{R-SVGF:} If $\rho_0=\mc{N}(0,\Sigma_0)$, then note that we have for all $t\ge 0$, $\rho_t=\mc{N}(0,\Sigma_t)$, $\gamma_t={1}/{2}$, $\lambda_{1,t}\le 1+\lv \Sigma_t \rv_F^2$ and $\lv  \mc{J}(\rho_t,\gamma_t) \rv_{L_2(\rho_t)}=\lv Q^{-1}-\Sigma_t^{-1} \rv_F $. \textcolor{black}{It follows from \cite[Equation 9]{liu2024towards} that}
    \begin{align*}
            \frac{\mathrm{d}}{\mathrm{d}t}\Sigma_t=2\big( (1-\nu)\Sigma_t+\nu I_d \big)^{-1}\Sigma_t-\big( (1-\nu)\Sigma_t+\nu I_d \big)^{-1}\Sigma_t^2 Q^{-1}-Q^{-1}\big( (1-\nu)\Sigma_t+\nu I_d \big)^{-1}\Sigma_t^2.
        \end{align*}
 Further assuming that $\Sigma_0 Q=Q \Sigma_0$, there exists an orthonormal matrix $P$ such that we have $Q=P^\intercal \text{diag}\{q_1,\cdots,q_d\} P$ and $\Sigma_t=P^\intercal \text{diag}\{\sigma_1(t),\cdots,\sigma_d(t)\} P$ with $q_1\ge q_2\ge \cdots\ge q_d$. \textcolor{black}{According to \cite[page 45]{liu2024towards}}, we also have that $\lambda=q_d$, $\lv \Sigma_t-Q \rv=O\big(\exp(-\frac{2t}{(1-\nu)q_1+\nu} )\big)$ and
 \begin{align*}
      \lambda_{1,t} &\le 1+ O\bigg(\sum_{i=1}^d \big(q_i+\exp(-\frac{2t}{(1-\nu)q_i+\nu})\big)^2 \bigg), \\
     \lv  \mc{J}(\rho_t,\gamma_t) \rv_{L_2(\rho_t)}^2&=O\left( \sum_{i=1}^d \sigma_i(0)^{-\frac{2\nu}{(1-\nu) q_i+\nu}}q_i^{-4+\frac{2\nu}{(1-\nu)q_i+\nu}}(\sigma_i(0)-q_i)^2e^{-\frac{4t}{(1-\nu)q_i+\nu}} \right).
 \end{align*}
 Our convergence result in Theorem \ref{thm:decay of KL under LSI}, hence translates as
 \begin{align*}%\label{eq:decay of KL under LSI}
    \KL(\rho_T|\pi)&\le e^{-2(1-\nu)^{-1}\lambda T}\KL(\rho_0|\pi)+ \int_0^T \nu(1-\nu)^{-2}  \lv \mathfrak{I}(\rho_t, \gamma_t) \rv_{L_2^d(\rho_t)}^2 e^{2(1-\nu)^{-1}\lambda(t-T)}dt\\
    &= e^{-2(1-\nu)^{-1}q_d T}\KL(\rho_0|\pi)+ O\left(\nu(1-\nu)^{-2} T \exp\left( -\min\left\{ \frac{4}{(1-\nu)q_1+\nu},\frac{2q_d}{1-\nu} \right\}T  \right)    \right),
\end{align*}
for $T\in (0,\infty)$. The first term in the above bound indicates the exponential convergence of the R-SVGF and the convergence rate is proportional to $1/(1-\nu)$. The second term characterizes the bias between $\rho_T$ and $\pi$. For any fixed values of $T$, the bias term vanishes as $\nu\to 0$. Therefore, in order to obtain an optimal upper bound on $\KL(\rho_T|\pi)$, there is a trade-off between choosing large $\nu$ (i.e., $\nu\to 1$) and small $\nu$ (i.e., $\nu\to 0$), which depends on $T$ and eigenvalues of $\Sigma_0$ and  $Q $.} \\

\textcolor{black}{\textbf{Time-discretized R-SVGF:} Note that if $\rho^0=\mc{N}(0,S_0)$, then for all $n\ge 0$, $\rho^n=\mc{N}(0,S_n)$, $\gamma_n={1}/{2}$. $R_n=\lv  \mc{J}(\rho^n,\gamma_n) \rv_{L_2(\rho^n)}=\lv Q^{-1}-S_n^{-1} \rv_F $ and $I(\rho^n|\pi)=\trace(S_n(Q^{-1}-S_n^{-1})^\intercal(Q^{-1}-S_n^{-1}))$, where $(S_n)_{n\ge 0}$ is updated as:
    \begin{align}
            S_{n+1}-S_n&=h_{n+1}\big( (1-\nu_{n+1})S_n+\nu_{n+1} I_d \big)^{-1}(S_n^{-1}-Q^{-1})S_n^2 \nonumber\\
            &\quad h_{n+1}S_n\big( (1-\nu_{n+1})S_n+\nu_{n+1} I_d \big)^{-1}(S_n^{-1}-Q^{-1}) S_n \nonumber\\
            &\quad +h_{n+1}^2 \big( (1-\nu_{n+1})S_n+\nu_{n+1} I_d \big)^{-1} (S_n^{-1}-Q^{-1}) S_n^3 \, \times \nonumber \\
            & \hspace{8em} \times  (S_n^{-1}-Q^{-1}) \big( (1-\nu_{n+1})S_n+\nu_{n+1} I_d \big)^{-1} \label{eq:example update}. 
        \end{align}
\iffalse
Meanwhile, as $\gamma_n=1/2$, we have {$B=0$} and recalling the definition of $\mathfrak{S}$ from~\eqref{eq:sandr}, we also have $\mathfrak{S}_n\le (1-\nu_{n+1})^{-2}$ for all $n\ge 0$. Therefore, the convergence result in Theorem \ref{thm:decay of KL along pop limit log-Sobolev version} translates as:
\begin{align*}
    \KL(\rho^n|\pi)\le \KL(\rho^0|\pi) \Pi_{i=1}^n \Big( 1-\frac{q_d}{2(1-\nu_i)} h_i \Big),
\end{align*}
where for all $i\ge 1$,
\begin{align*}
    \frac{\nu_i}{1-\nu_i}&\le \frac{\trace(S_{i-1}(Q^{-1}-S_{i-1}^{-1})^\intercal(Q^{-1}-S_{i-1}^{-1}))}{\textcolor{black}{2}\lv Q^{-1}-S_{i-1}^{-1} \rv_F^2},\\
    h_i&\le \min\left\{ \frac{1-\nu_i}{q_1},\frac{2(1-\nu_i)}{q_d} ,\textcolor{black}{ \frac{2(1-\nu_i)}{  \sqrt{d} \lv Q^{-1}-S_{i-1}^{-1} \rv_F },{\frac{\nu_i}{2d}}}\right\}.
\end{align*}
According to the iteration formula of $(S_n)_{n\ge 0}$, the optimal bound on $\KL(\rho^n|\pi)$ also depends on $n$ and eigenvalues of $S_0,Q$. From the above display, we note that the trade-off on the values of $\nu$ is qualitatively similar to that of the R-SVGF. Obtaining more quantitative results is possible by adapting the approach in~\cite[Theorem 3.10]{liu2024towards} to the above case, although the computations are more tedious.
\fi
}
\textcolor{black}{Since $\nabla_1 k(x,y)= y$, the RKHS norm of $\nabla_1 k(x,y)$ is obtained as follows
  \begin{align*}
      \lv \nabla_1 k(x,y) \rv_{\mc{H}_k^d}^2 = \sum_{i=1}^d \langle y_i, y_i \rangle_{\mc{H}_k} = \sum_{i=1}^d \partial_i \partial_{i+d} k(y,y) =d .
  \end{align*}
  Therefore, we have $\gamma_n=1/2$, {$B=\sqrt{d}$} and $\mathfrak{S}_n\le (1-\nu_{n+1})^{-2}$ for all $n\ge 0$. Now, the convergence result in Theorem \ref{thm:decay of KL along pop limit log-Sobolev version} translates as:
\begin{align*}
    \KL(\rho^n|\pi)\le \KL(\rho^0|\pi) \Pi_{i=1}^n \Big( 1-\frac{q_d}{2(1-\nu_i)} h_i \Big),
\end{align*}
where for all $i\ge 1$,
%\small{
\begin{align*}
    &\frac{\nu_i}{1-\nu_i}\le \frac{\trace(S_{i-1}(Q^{-1}-S_{i-1}^{-1})^\intercal(Q^{-1}-S_{i-1}^{-1}))}{2\lv Q^{-1}-S_{i-1}^{-1} \rv_F^2},\\
    h_i& \le \min\left\{ \frac{1-\nu_i}{q_1},\frac{2(1-\nu_i)}{q_d} , \frac{2(1-\nu_i)}{  \sqrt{d} \lv Q^{-1}-S_{i-1}^{-1} \rv_F },{\frac{\nu_i}{2d}} \right\}.
\end{align*} 
%}
We now show that such a choice of $\{\nu_i,h_i\}_{i\ge 1}$ exists.  Assuming $QS_0=S_0Q$, there exists an orthonormal matrix $P$ such that $Q=P^\intercal \text{diag}\{q_1,\cdots,q_d\} P$ and $S_i=P^\intercal \text{diag}\{\sigma_{1,i},\cdots,\sigma_{d,i}\} P$ with $q_1\ge q_2\ge \cdots\ge q_d$. Then, \eqref{eq:example update} implies that for any $k= 1,\cdots,d $,
\begin{align*}
    q_k^{-1}\sigma_{k,i+1}= \big( 1+ h_{i+1} ( (1-\nu_{i+1})\sigma_{k,i}+\nu_{i+1})^{-1}(1-q_k^{-1}\sigma_{k,i}) \big)^2 q_k^{-1}\sigma_{k,i}.
\end{align*}
Now, we pick 
\begin{align}\label{eq:gaussian example para}
    h_{i+1}= \delta \min_k \big((1-\nu_{i+1})\sigma_{k,i}+\nu_{i+1}\big) ,\quad \tfrac{\nu_i}{1-\nu_i}= \tfrac{\trace(S_{i-1}(Q^{-1}-S_{i-1}^{-1})^\intercal(Q^{-1}-S_{i-1}^{-1}))}{2\lv Q^{-1}-S_{i-1}^{-1} \rv_F^2},
\end{align} 
with $\delta\in (0,1/2)$. According to \cite[Theorem 3.10]{liu2024towards}, if $q_k^{-1}\sigma_{k,0}\in [u_\delta, 1/3+1/(3\delta)]$ for all $k=1,\cdots,d$, we have $\lv S_i-Q \rv_2\le e^{-i\delta} \lv S_0-Q \rv$. Here $u_\delta$ is the smaller root of $f_\delta'(u)=1-\delta$ with $f_\delta(x)=(1+\delta(1-x))^2x$. According to the convergence of $S_i$ and the fact that 
$$
\lv S_{i-1} \rv_2 \lv Q^{-1}-S_{i-1}^{-1} \rv_F^2\ge\trace(S_{i-1}(Q^{-1}-S_{i-1}^{-1})^\intercal(Q^{-1}-S_{i-1}^{-1}))\ge \lv S_{i-1}^{-1} \rv_2^{-1} \lv Q^{-1}-S_{i-1}^{-1} \rv_F^2,
$$ 
we have that $\{\nu_i\}_{i\ge 1}$ and $\{1-\nu_i\}_{i\ge 1}$ have uniform positive lower bounds, i.e., $\inf_i \nu_i>0$ and $\inf_i 1-\nu_i>0$. Therefore, we can further choose 
$$
\delta = \inf_i \left\{\big( (1-\nu_i) \lVert S_{i-1}^{-1}\rVert_2^{-1}+\nu_i \big)^{-1}\min\{ \frac{1-\nu_i}{q_1},\frac{2(1-\nu_i)}{q_d} , \frac{2(1-\nu_i)}{  \sqrt{d} \lv Q^{-1}-S_{i-1}^{-1} \rv_F },{\frac{\nu_i}{2d}}   \} \right\} >0,
$$
where the positivity of the third term follows from the fact that $$\lv Q^{-1}-S_{i-1}^{-1} \rv_F\le \lv Q^{-1} \rv_2 \lv S_{i-1}^{-1} \rv_2 \lv Q-S_{i-1} \rv_F.$$ Therefore, we have shown that there exists $\delta\in (0,1/2)$ such that $\{\nu_i,h_i\}_{i\ge 1}$ defined in \eqref{eq:gaussian example para} satisfies the assumptions in Theorem 6 and $\inf_i h_i>0$.}

\textcolor{black}{Last, we discuss the iteration complexity under the choice of $\{\nu_i,h_i\}_{i\ge 1}$ in \eqref{eq:gaussian example para}. According to Remark 9, $\KL(\rho^n|\pi)\le \epsilon$ if $n$ satisfies
\begin{align*}
     \sum_{i=1}^n \frac{h_i}{1-\nu_i} \ge \frac{2}{q_d} \log \frac{\KL(\rho^0|\pi)}{\epsilon}.
\end{align*}
Since $\inf_i h_i>0$ and $\inf_i 1-\nu_i>0$, we obtain $n=O(\log ({\KL(\rho^0|\pi)}/{\epsilon}))$.}
\end{remark}

\section{Existence and Uniqueness}\label{sec:exist}

The existence and uniqueness of the SVGF was studied in~\cite{lu2019scaling}. Motivated by their approach, in this section we study the existence and uniqueness of solutions to \eqref{eq:regularized SVGD mf PDE} under appropriate assumptions. Our main difficulty is in handling the non-linear operator $\left((1-\nu)\mathcal{T}_{k,\mu}+ \nu I\right)^{-1} \mathcal{T}_{k,\mu}$ in the R-SVGF.

We first introduce the definition of weak solutions to \eqref{eq:regularized SVGD mf PDE}. We restrict the initial conditions in the probability measure space $\mc{P}_V$ which is defined as
\begin{align*}%\label{eq:space Pv}
    \mc{P}_V:=\left\{ \rho\in \mc{P}~~:~~\lv \rho \rv_{\mc{P}_V}:=\int_{\mb{R}^d} (1+V(x)) d\rho(x)<\infty \right\},
\end{align*}
\textcolor{black}{where, in this section, with a slight overload of notations, we use $\mc{P}$ to denote the set of all probability measures on $\mb{R}^d$.} \textcolor{black}{We emphasize here that that $\mc{P}_V$ is a space of probability \emph{measures} because the weak solutions do not necessarily have densities even if the target measure and the initial measure has a density; see~\cite{lu2019scaling} for additional details.} We say that a measure-valued function $\rho\in \mc{C}([0,\infty),\mc{P}_V)$ is a weak solution to \eqref{eq:regularized SVGD mf PDE} with initial condition $\rho_0\in\mc{P}_V$ if 
\begin{align*}
    \sup_{t\in[0,T]} \lv \rho_t \rv_{\mc{P}_V}<\infty,\ \forall \ T>0,
\end{align*}
and 
\begin{align*}
    \int_0^\infty \int_{\mb{R}^d} (\partial_t \phi(t,x)+\nabla \phi(t,x) \cdot U[\rho_t](x))d\rho_t(x)dt+\int_{\mb{R}^d} \phi(0,x)d\rho_0(x)=0,
\end{align*}
for all $\phi\in \mc{C}_0^\infty ([0,\infty)\times \mb{R}^d)$ and $U[\rho]:=-\left((1-\nu)\iota_{k,\rho}\iota_{k,\rho}^*+\nu I\right)^{-1}\iota_{k,\rho}\iota_{k,\rho}^*(\nabla \log \frac{\rho}{\pi}) $. 

In order to study the existence of weak solutions, we consider the characteristic flow (see, for example,~\cite{muntean2016macroscopic} and~\cite[Definition 3.1]{lu2019scaling}) induced by \eqref{eq:regularized SVGD mf PDE}, which is written as
\begin{equation}\label{eq:CF of the regularized PDE}
    \left\{
    \begin{aligned}
    &\frac{d}{dt}\Phi(t,x,\rho_0)=-\mc{D}_{\nu,\rho_t}\nabla\log \frac{\rho_t}{\pi}(\Phi(t,x,\rho_0)),\\
    &\rho_t=(\Phi(t,\cdot,\rho_0))_{\#} \rho_0,\\
    & \Phi(0,x,\rho_0)=x,
    \end{aligned}
    \right.
\end{equation}
where $\mc{D}_{\nu,\rho_t}=\left((1-\nu)\iota_{k,\rho_t}\iota_{k,\rho_t}^*+\nu I\right)^{-1}\iota_{k,\rho_t}\iota_{k,\rho_t}^*$ for all $t>0$. Here, the expression $\rho_t=\Phi(t,\cdot,\rho_0)_{\#}\rho_0$ means that the measure $\rho_t$ is the push-forward measure of $\rho_0$ under the map $x\to \Phi(t,x,\rho_0)$. We think of $\{ X(t,\cdot,\rho_0)\}_{t\ge 0,\rho_0}$ as a family of maps from $\mb{R}^d$ to $\mb{R}^d$ parameterized by $t$ and $\rho_0$. The existence and uniqueness of the weak solutions of \eqref{eq:regularized SVGD mf PDE} is equivalent to the existence and uniqueness of solutions to \eqref{eq:CF of the regularized PDE}. In Theorem~\ref{thm:regularized mf PDE unique and existence}, we first prove that the mean field characteristic flow in \eqref{eq:CF of the regularized PDE} is well-defined. To do so, we also require the following additional assumptions on the kernel and the potential functions.

\iffalse
\textcolor{black}{
\begin{remark}\label{rem:well-define}
    Note that $\mc{D}_{\nu,\rho_t}(\nabla \log \frac{\rho_t}{\pi})$ is well-defined as long as $\iota^*_{k,\rho_t}\nabla\log \frac{\rho_t}{\pi}$ is well-defined. According to Proposition 1 and Proposition 3, $|\iota^*_{k,\rho_t}\nabla\log \frac{\rho_t}{\pi}|<\infty$ if $I_{\nu,Stein}(\rho_t|\pi)<\infty$.
\end{remark}
}
\fi

\begin{myassump}{K1}\label{ass:existence and uniqueness on K} The kernel $k:\mb{R}^d\times \mb{R}^d\to \mb{R}$ is symmetric, positive definite and fourth continuously differentiable in both variables with bounded derivatives up to fourth order. More explicitly, we assume
\begin{itemize}
    \item [(1)] $\lv k \rv_\infty:=\sup_{x\in \mb{R}^d} \sqrt{k(x,x)}<\infty$ 
    \item [(2)] $\lv \nabla k \rv_\infty:=\sup_{x,y\in \mb{R}^d} |\nabla_1 k(x,y)|=\sup_{x,y\in \mb{R}^d}|\nabla_2 k(x,y)|<\infty$ 
    \item [(3)] $\lv \nabla_1\cdot\nabla_2 k \rv_{\infty}:=\sup_{x,y\in \mb{R}^d} |\nabla_x\cdot \nabla_y k(x,y)|<\infty $.
    \item [(4)] $\lv \nabla^2 k \rv_\infty:=\sup_{x,y\in\mb{R}^d} \lv \nabla_x^2 k(x,y) \rv_2<\infty$.
    \item [(5)] $\lv \nabla_1\nabla_2 k(x,y) \rv_\infty:=\sup_{x,y\in\mb{R}^d} \lv \nabla_x\nabla_y k(x,y) \rv_2<\infty $.
    \item [(6)] $\lv \nabla^2(\nabla_1\cdot\nabla_2 k) \rv_\infty:=\sup_{x,y\in\mb{R}^d} \lv \nabla_x^2(\nabla_x\cdot\nabla_y k(x,y)) \rv_2<\infty $.
    \item [(7)] $\lv \nabla_1\nabla_2(\nabla_1\cdot\nabla_2 k) \rv_\infty:=\sup_{x,y\in \mb{R}^d} \lv \nabla_x\nabla_y (\nabla_x\cdot\nabla_y k(x,y)) \rv_2<\infty$.
    \item [(8)] $\lv \nabla_1^2 \cdot \nabla_2^2 k \rv_\infty:=\sup_{x,y\in \mb{R}^d} \sum_{i,j=1}^d |\partial_{x_i}\partial_{x_j}\partial_{y_i}\partial_{y_j}k(x,y)|<\infty $.
\end{itemize}
\end{myassump}

We emphasize here that~\cite{lu2019scaling} required that the kernel is radial for their analysis. However, our analysis does not require this assumption. A classical example of a kernel satisfying the above conditions is the Gaussian kernel.

\begin{myassump}{V1}\label{ass:existence and uniqueness on V} The potential function $V:\mb{R}^d\to \mb{R}$ satisfies
\begin{itemize}
    \item [(1)] $V\in \mc{C}^2(\mb{R}^d)$, $V\ge 0$ and $V(x)\to +\infty$ as $|x|\to +\infty$.
    \item [(2)] For any $\alpha,\beta>0$, there exists a constant $C_{\alpha,\beta}>0$ such that if $|y|\le \alpha|x|+\beta$, then 
    \begin{align*}
        (1+|x|)(|\nabla V(y)|+\lv \nabla^2 V(y)\rv_2)\le C_{\alpha,\beta}(1+V(x)).
    \end{align*}
    \item [(3)] $V$ is gradient Lipschitz with parameter $L_V$, i.e., for all $x\in \mb{R}^d$, $\lv \nabla^2 V (x)\rv_2\le L_V$.
\end{itemize}
\end{myassump}
To present our result, we define the set of functions
\begin{align*}%\label{eq:function space mean field CF global}
    Y:=\left\{ u\in \mc{C}(\mb{R}^d;\mb{R}^d) | \sup_{x\in \mb{R}^d} |u(x)-x|<\infty \right\},
\end{align*}
which is a complete metric space with the uniform metric $d_Y(u,v)=\sup_{x\in \mb{R}^d} |u(x)-v(x)|$.

\begin{theorem}\label{thm:regularized mf PDE unique and existence} Let $k$ satisfy Assumption \ref{ass:existence and uniqueness on K}, $V$ satisfy Assumption \ref{ass:existence and uniqueness on V} and $\rho_0\in \mc{P}_V$. 
\begin{itemize}[leftmargin=0.3in]
\item[(i)] For any $T>0$, there exists a unique solution $\Phi(\cdot,\cdot,\rho_0)\in\mc{C}^1([0,T];Y)$ to \eqref{eq:CF of the regularized PDE}. Moreover, the measure $\rho_t=\Phi(t,\cdot,\rho_0)_{\#}\rho_0$ satisfies 
$$
\lv\rho_t\rv_{\mc{P}_V}\le \lv \rho_0 \rv_{\mc{P}_V} \exp(C_{1,0}\nu^{-1/2} \lv k \rv_\infty \KL(\rho_0|\pi)^{1/2}t^{1/2}).
$$
\item[(ii)] For any $\rho_0\in \mc{P}_V$, there is a unique $\rho\in \mc{C}([0,\infty); \mc{P}_V )$ which is a weak solution to \eqref{eq:regularized SVGD mf PDE}. Moreover, for all $t\ge 0$, 
\begin{align*}
    \lv\rho_t\rv_{\mc{P}_V}\le \lv \rho_0 \rv_{\mc{P}_V} \exp(C_{1,0}\nu^{-1/2} \lv k \rv_\infty \KL(\rho_0|\pi)^{1/2}t^{1/2}).
\end{align*}
\end{itemize}
\end{theorem}

%\begin{theorem}\label{thm:regularized mf PDE unique and existence} Let $k$ satisfy \cref{ass:existence and uniqueness on K}, $V$ satisfy \cref{ass:existence and uniqueness on V}. Then, 
%\end{theorem}
%\noindent Theorem \ref{thm:regularized mf PDE unique and existence} follows directly from Theorem \ref{thm:CF well-defineness}.

\begin{remark}
In Theorem \ref{thm:regularized mf PDE unique and existence}, we introduce an upper bound to the $\mc{P}_V$-norm of the solution to \eqref{eq:regularized SVGD mf PDE} for any $\nu\in (0,1]$. A similar result is established for the case of SVGF, i.e., when $\nu=1$ in \cite[Theorem 2.4]{lu2019scaling}. In comparison to \cite[Theorem 2.4]{lu2019scaling}, our result requires that the initial KL-divergence to the target is bounded. Furthermore, if we set $\nu=1$ in our result, we do not end up recovering their result. When $\nu=1$, there is an explicit integral formula to $\mc{D}_{1,\rho_t}\nabla \log \frac{\rho_t}{\pi}$ which is leveraged in \cite{lu2019scaling} for their proof. For $\nu\in (0,1)$, due to the absence of an explicit representation, we get the result in Theorem \ref{thm:regularized mf PDE unique and existence} by carefully analyzing the quantity $\mc{D}_{\nu,\rho_t}(\nabla \log \frac{\rho_t}{\pi})$ along with the decay of KL-divergence property introduced in Proposition \ref{cor:KL derivative}.% This is mainly because that we lose certain accuracy in the analysis when we extend from $\nu=1$ to $\nu\in (0,1]$. 
\end{remark}

\begin{proof}[Proof of Theorem \ref{thm:regularized mf PDE unique and existence}]\label{pf:CF well-defineness} Our proof leverages the approach of \cite[Theorem 3.2]{lu2019scaling} for the case of SVGF. In comparison to \cite{lu2019scaling}, we handle various difficulties arising with the non-linear operator in R-SVGF. We first prove claim (i) based on the following two steps. Claim (ii) follows directly from claim (i) and \cite[Theorem 5.34]{villani2021topics}.\\

\noindent \textbf{Step 1 (Local well-posedness): } Fix $r>0$ and define
\begin{align}\label{eq:function space CF local} 
    Y_r:=\left\{ u\in Y | \sup_{x\in \mb{R}^d} |u(x)-x|\le r \right\}.
\end{align}
We will prove that there exists $T_0>0$ such that \eqref{eq:CF of the regularized PDE} has a unique solution $\Phi(t,x,\rho_0)$ in the set $S_r=\mc{C}([0,T_0];Y_r)$ which is a complete metric space with metric
\begin{align*}
    d_S\left(u,v\right)=\sup_{t\in [0,T_0]} d_Y \left(u(t,\cdot),v(t,\cdot)\right).
\end{align*}
The integral formulation of \eqref{eq:CF of the regularized PDE} is
\textcolor{black}{
\begin{align}\label{eq:integral formulation of CF}
    \Phi(t,x,\rho_0)&=x-\int_0^t \mc{D}_{\nu,\rho_s} \nabla \log \frac{\rho_s}{\pi}(\Phi(s,x,\rho_0)) ds \nonumber \\
    &=x+\int_0^t \big((1-\nu)\iota_{k,\rho_s}\iota_{k,\rho_s}^*+\nu I\big)^{-1}\mb{E}_{y\sim \rho_s} [-\nabla V(y)k(y,\cdot)+\nabla k(y,\cdot)] (\Phi(s,x,\rho_0))ds.
\end{align}
}
Let us define the operator $\mc{F}:u(t,\cdot)\mapsto \mc{F}(u)(t,\cdot)$ by 
\textcolor{black}{
\begin{align}\label{eq:defn Lip map}
    \mc{F}(u)(t,x)=x+\int_0^t \big((1-\nu)\iota_{k,\rho_{u,s}}\iota_{k,\rho_{u,s}}^*+\nu I\big)^{-1}\mb{E}_{y\sim \rho_{u,s}} [-\nabla V(y)k(y,\cdot)+\nabla k(y,\cdot)] (u(s,x))ds,
\end{align}
}
where $\rho_{u,t}=\left( u(t,\cdot) \right)_{\#}\rho_0$. \textcolor{black}{For the simplicity of notation, we will denote the map defined in \eqref{eq:defn Lip map} by $\mc{F}(u)(t,x)=x-\int_0^t \mc{D}_{\nu,\rho_{u,s}}\nabla\log \frac{\rho_{u,s}}{\pi}(u(s,x))ds$ for any $u\in S_r$.} We now show that $\mc{F}$ is a contraction from $S_r$ to $S_r$ and thus has a unique fixed point. First, we show that $\mc{F}$ maps $S_r$ into $S_r$ for some $T_0>0$. For any $u\in S_r$, checking that $(t,x)\mapsto \mc{F}(u)(t,x)$ is continuous is straightforward. We need to show that $\sup_{x\in \mb{R}^d} |\mc{F}(u)(t,x)-x|\le r $ for any $u\in S_r$. 
\iffalse
\begin{align*}
    \mc{F}(u)(t,x)-x&=-\int_0^t  \mc{D}_{\nu,\rho_s}\nabla \log \frac{\rho_s}{\pi}(u(s,x)) ds \\
    &=-\int_0^t \left((1-\nu)\iota_{k,\rho_s}\iota_{k,\rho_s}^*+\nu I\right)^{-1}\iota_{k,\rho_s}\iota_{k,\rho_s}^*\nabla \log \frac{\rho_s}{\pi}(\Phi(s,x))ds \\
    &=-\int_0^t \left((1-\nu)\iota_{k,\rho_s}\iota_{k,\rho_s}^*+\nu I\right)^{-1} \mb{E}_{y\sim \rho_s}[k(\cdot,y)\nabla V(y)-\nabla k(y,\cdot)] (u(s,x))ds.
\end{align*} 
\fi

Note that there is an equivalent representation for $\mc{D}_{\nu,\rho_{u,s}}$: $$\mc{D}_{\nu,\rho_{u,s}}=\iota_{k,\rho_{u,s}}\left((1-\nu)\iota_{k,\rho_{u,s}}^*\iota_{k,\rho_{u,s}}+\nu I_d\right)^{-1}\iota_{k,\rho_{u,s}}^*.$$ We then analyze the operators $\iota_{k,\rho_{u,s}}$ and $\left((1-\nu)\iota_{k,\rho_{u,s}}^*\iota_{k,\rho_{u,s}}+\nu I_d\right)^{-1}\iota_{k,\rho_{u,s}}^*$ respectively. Since  $\lv k \rv_{\infty}<\infty$, according to Proposition \ref{prop:RKHS property}, $\iota_{k,\rho_{u,s}}$ is the inclusion operator from $\mc{H}_k^d$ to $L_\infty^d(\mb{R}^d)$. The corresponding operator norm, denoted as $\lv \iota_{k,\rho_{u,s}} \rv_{\mc{H}_k^d\to L_\infty^d}$ can be bounded in the following way:
\begin{align}\label{eq:inclusion operator bound}
    \lv \iota_{k,\rho_{u,s}} \rv_{\mc{H}_k^d\to L_\infty^d} &:= \sup_{\lv f \rv_{\mc{H}_k^d}=1} \sup_{x\in \mb{R}^d} |f(x)| \nonumber\\
    &= \sup_{\lv f \rv_{\mc{H}_k^d}=1} \sup_{x\in \mb{R}^d} |\langle k(x,\cdot), f \rangle_{\mc{H}_k}| \nonumber\\
    &\le \sup_{x\in \mb{R}^d} \sqrt{k(x,x)}:=\lv k \rv_\infty.
\end{align}
Meanwhile, let $(\lambda_i,e_i)_{i=1}^\infty$ be the spectrum of $\iota_{k,\rho_{u,s}}\iota_{k,\rho_{u,s}}^*$ with $(e_i)_{i=1}^\infty$ being an orthonormal basis of $L^d_2(\rho_{u,s}) \equiv\overline{\text{Ran}(\iota_{k,\rho_{u,s}}\iota_{k,\rho_{u,s}}^*)}$. According to Proposition \ref{prop:RKHS property}, $(\sqrt{\lambda_i}e_i)_{i=1}^\infty$ is an orthonormal basis of $\mc{H}^d_k$; see also Remark~\ref{rem:trivial}. Hence, we have
\begin{align}\label{eq:inverse gradient RKHS bound}
    &\quad \lv \left((1-\nu)\iota_{k,\rho_{u,s}}^*\iota_{k,\rho_{u,s}}+\nu I_d\right)^{-1}\iota_{k,\rho_{u,s}}^*\nabla \log \frac{\rho_{u,s}}{\pi} \rv_{\mc{H}_k^d}^2 \nonumber \\
    &\le \lv \left((1-\nu)\iota_{k,\rho_{u,s}}^*\iota_{k,\rho_{u,s}}+\nu I_d\right)^{-\frac{1}{2}} \rv_{\mc{H}_k^d\to \mc{H}_k^d}^2 \lv \left((1-\nu)\iota_{k,\rho_{u,s}}^*\iota_{k,\rho_{u,s}}+\nu I_d\right)^{-\frac{1}{2}}\iota_{k,\rho_{u,s}}^*\nabla \log \frac{\rho_{u,s}}{\pi} \rv_{\mc{H}_k^d}^2 \nonumber \\
    &\le \nu^{-1}I_{\nu,Stein}(\rho_{u,s}|\pi).
\end{align}
where the last inequality follows from \eqref{eq:regularized Fisher information} and the fact that $(1-\nu)\iota_{k,\rho_{u,s}}^*\iota_{k,\rho_{u,s}}$ is positive. With \eqref{eq:inclusion operator bound} and \eqref{eq:inverse gradient RKHS bound}, we get the following uniform bound on $|\mc{D}_{\nu,\rho_{u,s}}\textcolor{black}{\nabla}\log \frac{\rho_{u,s}}{\pi}(x)|$ for all $x\in \mb{R}^d$,
\begin{align*}
   \left|\mc{D}_{\nu,\rho_{u,s}}\nabla\log \frac{\rho_{u,s}}{\pi}(x)\right|&\le \lv \mc{D}_{\nu,\rho_{u,s}}\nabla\log \frac{\rho_{u,s}}{\pi} \rv_{L_\infty^d}\\
   &\le \lv \iota_{k,\rho_{u,s}} \rv_{\mc{H}_k^d\to L_\infty^d} \lv \left( (1-\nu)\iota_{k,\rho_{u,s}}^*\iota_{k,\rho_{u,s}}+\nu I_d\right)^{-1} \iota_{k,\rho_{u,s}}^*\nabla\log \frac{\rho_{u,s}}{\pi} \rv_{\mc{H}_k^d}\\
   &\le \nu^{-\frac{1}{2}} \lv k \rv_\infty I_{\nu,Stein}(\rho_{u,s}|\pi)^{\frac{1}{2}}.
\end{align*}
Therefore, for all $t\in [0,T]$ and all $ x\in \mb{R}^d$:
\begin{align}\label{eq:integral solution bound}
    |\mc{F}(u)(t,x)-x|\le \nu^{-\frac{1}{2}} \lv k \rv_\infty \int_0^T I_{\nu,Stein}(\rho_{u,s}|\pi)^{\frac{1}{2}} ds.
\end{align}
According to Lemma \ref{lem:regularized Fisher}, there exists $T_0>0$ such that for all $u\in S_r$,
\begin{align*}
  \int_0^{T_0}I_{\nu,Stein}(\rho_{u,t}|\pi)^{\frac{1}{2}}dt < \nu^{1/2} \lv k \rv_\infty^{-1} r  ,
\end{align*}
which along with \eqref{eq:integral solution bound} implies $|\mc{F}(u)(t,x)-x|\le r$ for all $u\in S_r$. \\

\noindent Next we show that $\mc{F}$ is a contraction on $S_r$. Our goal is to show that there exists $T_0>0$ and $C\in (0,1)$ such that for any $u,v\in S_r$,
\begin{align*}
    \sup_{t\in [0,T_0]}\sup_{x\in \mb{R}^d} \left|\mc{F}(u)(t,x)-\mc{F}(v)(t,x)\right|< C \sup_{t\in [0,T_0]}\sup_{x\in \mb{R}^d} |u(t,x)-v(t,x)|.
\end{align*}
\textcolor{black}{Let  $\rho_{u,t}=(u(t,\cdot))_{\#}\rho_0$, $\rho_{v,t}=(v(t,\cdot))_{\#}\rho_0$,} we have that
\begin{align*}
    \left|\mc{F}(u)(t,x)-\mc{F}(v)(t,x)\right|&=\left|\int_0^t \mc{D}_{\nu,\rho_{u,s}}\nabla \log \frac{\rho_{u,s}}{\pi}\left(u(s,x)\right)-\mc{D}_{\nu,\rho_{v,s}}\nabla \log \frac{\rho_{v,s}}{\pi}\left(v(s,x)\right) ds\right| \\
    &\le \left|\int_0^t \mc{D}_{\nu,\rho_{u,s}}\nabla \log \frac{\rho_{u,s}}{\pi}(u(s,x))-\mc{D}_{\nu,\rho_{v,s}}\nabla \log \frac{\rho_{v,s}}{\pi}(u(s,x)) ds \right|\\
    &\ +\left|\int_0^t \mc{D}_{\nu,\rho_{v,s}}\nabla \log \frac{\rho_{v,s}}{\pi}(u(s,x))-\mc{D}_{\nu,\rho_{2,s}}\nabla \log \frac{\rho_{2,s}}{\pi}(v(s,x))ds \right|\\
    &\le  d_S(u,v)\int_0^{T_0} C_1(t) dt+\int_0^{T_0} L(t) dt \sup_{t\in [0,T_0]}\sup_{x\in \mb{R}^d} |u(t,x)-v(t,x)|\\
    &= d_S(u,v) \int_0^{T_0} C_1(t)+L(t)\  dt,
\end{align*}
where the second inequality follows from Lemma \ref{lem:Lipschitz condition pointwise} and Lemma \ref{lem:Lipschitz condition function space}. Furthermore, according to \eqref{eq:Lipschizt constant pointwise} and \eqref{eq:Lipschitz condition function space}, there exists $T_0>0$ such that
\begin{align*}
     \int_0^{T_0} C_1(t)+L(t)\ dt < 1.
\end{align*}
Therefore we have proved that there exists $T_0>0$ such that $\mc{F}$ is a contraction from $S_r$ into $S_r$. According to the contraction theorem, $\mc{F}$ has a unique fixed point $\Phi(\cdot,\cdot,\rho_0)\in S_r$ which solves \eqref{eq:CF of the regularized PDE}. Defining $\rho_t=(\Phi(t,\cdot,\rho_0))_{\#}\rho_0$, one sees that $\Phi(t,x,\rho_0)$ solves \eqref{eq:CF of the regularized PDE} in the time interval $[0,T_0]$.\\

\noindent\textbf{Step 2 (Extension of local solution):} According to \eqref{eq:Lipschitz condition pointwise} and \eqref{eq:Lipschitz condition function space}, we can extend the local solution beyond time $T_0$ as long as the quantity
\begin{align*}
    \lv \rho_t \rv_{\mc{P}_V} =\int_{\mb{R}^d} (1+V(\Phi(t,x,\rho_0))) d\rho_0(x)
\end{align*}
remains finite. Next we establish a bound on this quantity showing that the local solution can be extended to any $t>0$. \begin{align*}
    \partial_t \int_{\mb{R}^d} \left( 1+V(\Phi(t,x,\rho_0)) \right) \rho_0(dx) &=-\int_{\mb{R}^d} \leftlangle\nabla V( \Phi(t,x,\rho_0) ) , \mc{D}_{\nu,\rho_t}\nabla \log \frac{\rho_t}{\pi}(\Phi(t,x,\rho_0)) \rightrangle d\rho_0(x) \\
     &=-\leftlangle \nabla V, \iota_{k,\rho_t}\left((1-\nu)\iota_{k,\rho_t}^*\iota_{k,\rho_t}+\nu I_d\right)^{-1}\iota_{k,\rho_t}^* \nabla \log \frac{\rho_t}{\pi} \rightrangle_{L_2^d(\rho_t)} \\
     &=-\leftlangle \iota_{k,\rho_t}^* \nabla V, \left((1-\nu)\iota_{k,\rho_t}^*\iota_{k,\rho_t}+\nu I_d\right)^{-1}\iota_{k,\rho_t}^* \nabla \log \frac{\rho_t}{\pi} \rightrangle_{\mc{H}_k^d} \\
     &\le \lv \iota_{k,\rho_t}^* \nabla V \rv_{\mc{H}_k^d} \lv \left((1-\nu)\iota_{k,\rho_t}^*\iota_{k,\rho_t}+\nu I_d\right)^{-1}\iota_{k,\rho_t}^* \nabla \log \frac{\rho_t}{\pi} \rv_{\mc{H}_k^d},
 \end{align*}
 where 
 \begin{align*}
     \lv \iota_{k,\rho_t}^* \nabla V \rv_{\mc{H}_k^d}^2&=\leftlangle \nabla V, \iota_{k,\rho}\iota_{k,\rho_t}^*\nabla V  \rightrangle_{L_2^d(\rho_t)}\\
     &=\int_{\mb{R}^d} \int_{\mb{R}^d}  k(y,z) \leftlangle \nabla V(y) ,\nabla V(z) \rightrangle d\rho_t(y) d\rho_t(z)\\
     &=\int_{\mb{R}^d}\int_{\mb{R}^d} k(\Phi(t,y,\rho_0),\Phi(t,z,\rho_0)) \left\langle \nabla V(\Phi(t,y,\rho_0)), \nabla V(\Phi(t,z,\rho_0)) \right\rangle d\rho_0(y) d\rho_0(z) \\
     &\le \lv k \rv_\infty^2 \left( \int_{\mb{R}^d} |\textcolor{black}{\nabla V(\Phi(t,y,\rho_0))}| d\rho_0(y)  \right)^2 \\
     &\le  \lv k \rv_\infty^2 C_{1,0}^2 \lv \rho_t \rv_{\mc{P}_V}^2,
 \end{align*}
 where the last inequality follows from Assumption \ref{ass:existence and uniqueness on V}. Therefore,
 \begin{align*}
     \partial_t \lv \rho_t \rv_{\mc{P}_V}&\le C_{1,0}\lv k \rv_\infty \lv \rho_t \rv_{\mc{P}_V} \lv  ((1-\nu)\iota_{k,\rho_t}^*\iota_{k,\rho_t}+\nu I_d)^{-1}\iota_{k,\rho_t}^* \nabla \log \frac{\rho_t}{\pi} \rv_{\mc{H}_k^d} \\
     &\le C_{1,0}  \lv k \rv_\infty \nu^{-\frac{1}{2}} I_{\nu,Stein}(\rho_t|\pi)^{\frac{1}{2}} \lv \rho_t \rv_{\mc{P}_V},
 \end{align*}
 where the last inequality follows from \eqref{eq:inverse gradient RKHS bound}.
 It follows from Gronwall's inequality that
 \begin{align}\label{eq:PV norm bound}
     \lv \mc{\rho}_t \rv_{\mc{P}_V} &\le  \lv \rho_0 \rv_{\mc{P}_V} \exp\left( C_{1,0} \nu^{-\frac{1}{2}} \lv k \rv_\infty \int_0^t I_{\nu,Stein}(\rho_s|\pi)^{\frac{1}{2}} ds \right)\nonumber \\
     &\le  \lv \rho_0 \rv_{\mc{P}_V} \exp\left( C_{1,0} \nu^{-\frac{1}{2}} \lv k \rv_\infty \sqrt{t \int_0^t I_{\nu,Stein}(\rho_s|\pi) ds} \right) \nonumber \\
     &\le  \lv \rho_0 \rv_{\mc{P}_V} \exp\left( C_{1,0} \nu^{-\frac{1}{2}} \lv k \rv_\infty \sqrt{t \KL(\rho_0|\pi) } \right),
 \end{align}
 where the second inequality follows from Jensen's inequality and the last inequality follows from \eqref{eq:KL derivative}. With this bound, we can iterate the argument to extend the local solution defined on $[0,T_0]\times \mb{R}^d$ to all of $[0,\infty)\times \mb{R}^d$, so that \eqref{eq:PV norm bound} holds for all $t>0$. Finally, \textcolor{black}{$\Phi(\cdot,x,\rho_0)$} has continuous first order derivative due to the integral formulation in \eqref{eq:integral formulation of CF}, Assumption \ref{ass:existence and uniqueness on K} and Assumption \ref{ass:existence and uniqueness on V}. The proof is thus complete. 
\end{proof}
\iffalse
\begin{conjecture}\label{conjecture:regularized Fisher}
Let $\rho_0\in\mc{P}_V$. For any $\epsilon>0$, there exists a constant $T>0$ such that for all $u\in S_r$ and $t\in [0,T]$, with $\rho_t=u(t,\cdot)_{\#}\rho_0$ \begin{align*}
     \int_0^T I_{\nu,Stein}(\rho_t|\pi)^{\frac{1}{2}}dt < \epsilon.
 \end{align*}
\end{conjecture}
\fi
\begin{lemma}\label{lem:regularized Fisher} Let $\rho_0\in\mc{P}_V$ and suppose the assumptions in Theorem \ref{thm:regularized mf PDE unique and existence} hold. Then, for any $\epsilon>0$, there exists a constant $T>0$ such that for all $u\in S_r$ and $t\in [0,T]$, with $\rho_{u,t}=u(t,\cdot)_{\#}\rho_0$, we have 
\begin{align}\label{eq:bound regularized FI for small t}
     \int_0^T I_{\nu,Stein}(\rho_{u,t}|\pi)^{\frac{1}{2}}dt < \epsilon.
 \end{align}
\end{lemma}

\begin{proof}[Proof of Lemma~\ref{lem:regularized Fisher}]
According to \textcolor{black}{Proposition} \ref{lem:regularized SVGD optimal vf}, the regularized kernelized Stein discrepancy can be written as
\begin{align*}
    S(\rho_{u,t}|\pi)^2 &= \left( \mb{E}_{x\sim \rho_{u,t}} \left[ \text{trace}(\mc{A}_\pi \phi_{\rho_{u,t},\pi}^*(x)) \right] \right)^2\\
    &= \lv \left( (1-\nu)\iota_{k,\rho_{u,t}}^*\iota_{k,\rho_{u,t}}+\nu I_d \right)^{-\frac{1}{2}}\iota_{k,\rho_{u,t}}^*\nabla \log \frac{\rho_{u,t}}{\pi} \rv_{\mc{H}_k^d}^2 \\
    &=\leftlangle \iota_{k,\rho_{u,t}}^*\nabla \log \frac{\rho_{u,t}}{\pi},  \left((1-\nu) \iota_{k,\rho_{u,t}}^*\iota_{k,\rho_{u,t}}+\nu I_d \right)^{-1}\iota_{k,\rho_{u,t}}^*\nabla \log \frac{\rho_{u,t}}{\pi} \rightrangle_{\mc{H}_k^d}
    =I_{\nu,Stein}(\rho_{u,t},\pi).
\end{align*}
Meanwhile, since $\rho_{u,t}=u(t,\cdot)_{\#}\rho_0$ with $u\in S_r$, for any $\mb{R}^d$-valued random vector $X$, $u(t,X)\sim \rho_{u,t}$ and $|u(t,X)-X|\le r$ almost surely. Therefore,
\begin{align*}
    \textcolor{black}{\mc{W}_1}(\rho_{u,t},\pi)&\le \textcolor{black}{\mc{W}_1}(\rho_0,\rho_{u,t})+\textcolor{black}{\mc{W}_1}(\rho_0,\pi) 
    =\inf_{X\sim \rho_0,Y\sim \rho_{u,t}} \mb{E}\left[ |X-Y| \right]+\textcolor{black}{\mc{W}_1}(\rho_0,\pi) 
    \le r+\textcolor{black}{\mc{W}_1}(\rho_0,\pi),
\end{align*}
\textcolor{black}{where $\mc{W}_1$ is the Wasserstein-1 distance defined in Section \ref{sec:notation}.} Next, we upper bound the regularized kernelized Stein discrepancy by the Wasserstein-\textcolor{black}{1} distance. According to \cite[Lemma 18]{gorham2017measuring}, for any general vector field $\phi\in \mc{H}_k^d$, we have 
\begin{align*}
    \left| \mb{E}_{x\sim \rho_{u,t}} \left[ \text{trace}(\mc{A}_\pi \phi(x)) \right] \right|&\le \left(M_0(\phi)M_1(\nabla V)+M_2(\phi)d\right)\textcolor{black}{\mc{W}_1}(\rho_{u,t},\pi) \\ &\quad\quad +\sqrt{2M_0(\phi)M_1(\phi)\mb{E}_{x\sim \pi}\left[ |\nabla V(x)|^2\right] \textcolor{black}{\mc{W}_1}(\rho_{u,t},\pi)},
\end{align*}
where for any $g:\mb{R}^d\to \mb{R}^d$ and $g\in C^1(\mb{R}^d)$,
\begin{align*}
    M_0(g):=\sup_{x\in \mb{R}^d} |g(x)|, \quad M_1(g):=\sup_{x\neq y} \frac{|g(x)-g(y)|}{|x-y|}, \quad  M_2(g):=\sup_{x\neq y} \frac{\lv \nabla g(x)-\nabla g(y) \rv_2}{|x-y|}.
\end{align*}
For any $\phi\in \mc{H}_k^d$ and $\phi=[\phi_1,\cdots,\phi_d]^T$, according to \cite[Lemma 4.34]{steinwart2008support},
\begin{align*}
    \sup_{x\in \mb{R}^d} |D^\alpha \phi_i(x)|=\sup_{x\in \mb{R}^d} \left|D^\alpha \leftlangle \phi_i,k(x,\cdot) \rightrangle_{\mc{H}_k} \right|\le \lv \phi_i \rv_{\mc{H}_k^d} \sup_{x\in \mb{R}^d} |D_1^\alpha D_2^\alpha k(x,x) |^{\frac{1}{2}}.
\end{align*}
Therefore,
\begin{align*}
    M_0(\phi)&=\sup_{x\in \mb{R}^d} \sqrt{\sum_{i=1}^d \phi_i(x)^2}\le \sqrt{ \sum_{i=1} \lv \phi_i \rv_{\mc{H}_k}^2 \sup_{x\in \mb{R}^d} k(x,x) } =\lv k \rv_\infty \lv \phi \rv_{\mc{H}_k^d},\\
    &\\
    M_1(\phi)&= \sup_{x\neq y} \frac{\sqrt{\sum_{i=1}^d\left( \phi_i(x)-\phi_i(y)\right)^2}}{|x-y|}\le \sqrt{ \sum_{i=1}^d \sup_{x\in \mb{R}^d} |\nabla \phi_i(x)|^2 } \\
    &\le \sqrt{ \sum_{i=1}^d\sum_{j=1}^d \lv \phi_i \rv_{\mc{H}_k}^2 \sup_{x\in \mb{R}^d} D_1^{e_j}D_2^{e_j} k(x,x) } 
    =  \left(\sup_{x\in \mb{R}^d} \text{trace} \left(\nabla_1 \nabla_2 k(x,x) \right)\right)^{\frac{1}{2}}\lv \phi \rv_{\mc{H}_k^d} \\
    &\le \lv \nabla_1\cdot\nabla_2 k \rv_\infty^{\frac{1}{2}} \lv \phi \rv_{\mc{H}_k^d},\\
    &\\
    M_2(\phi)&=\sup_{x\neq y} \frac{\lv \nabla \phi(x)-\nabla \phi(y) \rv_2}{|x-y|}\le \sup_{x\neq y} \frac{\lv \nabla \phi(x)-\nabla \phi(y) \rv_F}{|x-y|} 
    \le \sqrt{\sum_{i,j=1}^d \sup_{x\neq y} \frac{|\partial_j \phi_i(x)-\partial_j \phi_i(y)|^2}{|x-y|^2} } \\
    &\le \sqrt{\sum_{i,j,l=1}^d \sup_{x\in \mb{R}^d} D_1^{e_j+e_l}D_2^{e_j+e_l} k(x,x) \lv \phi_i \rv_{\mc{H}_k}^2   }
    \le \lv \nabla_1^2 \cdot \nabla_2^2 k\rv_\infty^{\frac{1}{2}} \lv \phi \rv_{\mc{H}_k^d}.
\end{align*} 
According to Assumption \ref{ass:existence and uniqueness on V}, $M_1(\nabla V)=L_V$ and $\mb{E}_{x\sim \pi}\left[ |\nabla V(x)|^2 \right]\le C_{1,0} \lv \pi \rv_{\mc{P}_V} $. Therefore,
\begin{align*}
    \left| \mb{E}_{x\sim \rho_{u,t}} \left[ \text{trace}(\mc{A}_\pi \phi(x)) \right] \right|&\le \left( \lv k \rv_\infty L_V+\lv \nabla_1^2\cdot\nabla_2^2 k \rv_\infty^{\frac{1}{2}}d \right)\lv \phi \rv_{\mc{H}_k^d} (\textcolor{black}{\mc{W}_1}(\rho_0,\pi)+r)\\
    &\quad+\sqrt{ 2\lv k \rv_\infty \lv \nabla_1\cdot\nabla_2 k \rv_\infty^{\frac{1}{2}}C_{1,0}\lv \pi \rv_{\mc{P}_V}(\textcolor{black}{\mc{W}_1}(\rho_0,\pi)+r) } \lv \phi \rv_{\mc{H}_k^d}.
\end{align*}
Note that $\phi_{k,\rho_{u,t}}^*$ satisfies that $\nu \lv \phi_{k,\rho_{u,t}}^* \rv_{\mc{H}_k^d}^2+(1-\nu)\lv \phi_{k,\rho_{u,t}}^* \rv_{L^d_2(\rho_{u,t})^d}\le 1$. Therefore, 
$$
\lv \phi_{k,\rho_{u,t}}^* \rv_{\mc{H}_k^d}\le \nu^{-1/2},
$$ 
and
\begin{align*}
    I_{\nu,Stein}(\rho_{u,t}|\pi)^{\frac{1}{2}}&=S(\rho_{u,t},\pi) 
    =\left| \mb{E}_{x\sim \rho_{u,t}} \left[ \text{trace}(\mc{A}_\pi \phi_{k,\rho_{u,t}}^*(x)) \right] \right|\\
    &\le \nu^{-\frac{1}{2}} \left( \lv k \rv_\infty L_V+\lv \nabla_1^2\cdot\nabla_2^2 k \rv_\infty^{\frac{1}{2}}d \right)(\textcolor{black}{\mc{W}_1}(\rho_0,\pi)+r)\\
    &\quad+\nu^{-\frac{1}{2}}\sqrt{ 2\lv k \rv_\infty \lv \nabla_1\cdot\nabla_2 k \rv_\infty^{\frac{1}{2}}C_{1,0}\lv \pi \rv_{\mc{P}_V}(\textcolor{black}{\mc{W}_1}(\rho_0,\pi)+r) }. 
\end{align*}
Since the upper bound is independent of the choice of $u(t,\cdot)\in S_r$ and the time variable $t$, for any $\epsilon>0$, we can choose a small enough $T$ such that \eqref{eq:bound regularized FI for small t} holds.
\end{proof}

\begin{lemma}\label{lem:Lipschitz condition pointwise} Under the assumptions in Theorem \ref{thm:regularized mf PDE unique and existence}, let $S_r=\mc{C}([0,T];Y_r)$ with $Y_r$ defined in \eqref{eq:function space  CF local}. Then for any $t\in [0,T]$, there exists $L(t)>0$ such that for any $u\in S_r$, for all $x,y\in \mb{R}^d$ and $t\in [0,T]$,
\begin{align}\label{eq:Lipschitz condition pointwise}
    \left|\mc{D}_{\nu,\rho_{u,t}}\nabla \log \frac{\rho_{u,t}}{\pi}(x)-\mc{D}_{\nu,\rho_{u,t}}\nabla \log \frac{\rho_{u,t}}{\pi}(y)\right|\le  L(t)|x-y| ,
\end{align}
where for all $t\in [0,T]$,  $\rho_{u,t}=(u(t,\cdot))_{\#}\rho_0$ and 
\begin{align}\label{eq:Lipschizt constant pointwise}
    L(t)=\nu^{-\frac{1}{2}} \left( 2\lv \nabla_1\nabla_2 k \rv_\infty+3\lv \nabla^2 k \rv_\infty\right)^{\frac{1}{2}} I_{\nu,Stein}(\rho_{u,t}|\pi)^{\frac{1}{2}}.
\end{align}
\end{lemma}
\begin{proof}[Proof of Lemma~\ref{lem:Lipschitz condition pointwise}]
\label{pf:Lipschitz condition pointwise} 
Since $\mc{D}_{\nu,\rho_{u,t}}=\iota_{k,\rho_{u,t}}\left((1-\nu)\iota_{k,\rho_{u,t}}^*\iota_{k,\rho_{u,t}}+\nu I_d\right)^{-1}\iota_{k,\rho_{u,t}}^*$ and $\iota_{k,\rho_{u,t}}$ is the inclusion operator,
 \begin{align*}
     &\quad \left|\mc{D}_{\nu,\rho_{u,t}}\nabla \log \frac{\rho_{u,t}}{\pi}(x)-\mc{D}_{\nu,\rho_{u,t}}\nabla \log \frac{\rho_{u,t}}{\pi}(y)\right|\\
     & =\left|\left((1-\nu)\iota_{k,\rho_{u,t}}^*\iota_{k,\rho_{u,t}}+\nu I_d\right)^{-1}\iota_{k,\rho_{u,t}}^*\nabla \log \frac{\rho_{u,t}}{\pi}(x)-\left((1-\nu)\iota_{k,\rho_{u,t}}^*\iota_{k,\rho_{u,t}}+\nu I_d\right)^{-1}\iota_{k,\rho_{u,t}}^*\nabla \log \frac{\rho_{u,t}}{\pi}(y)\right|\\
     &=\left|\leftlangle k(x,\cdot)-k(y,\cdot),\left((1-\nu)\iota_{k,\rho_{u,t}}^*\iota_{k,\rho_{u,t}}+\nu I_d\right)^{-1}\iota_{k,\rho_{u,t}}^*\nabla \log \frac{\rho_{u,t}}{\pi}  \rightrangle_{\mc{H}_k}\right|\\
     &\le \lv  k(x,\cdot)-k(y,\cdot) \rv_{\mc{H}_k} \lv \left((1-\nu)\iota_{k,\rho_{u,t}}^*\iota_{k,\rho_{u,t}}+\nu I_d\right)^{-1}\iota_{k,\rho_{u,t}}^*\nabla \log \frac{\rho_{u,t}}{\pi} \rv_{\mc{H}_k^d}\\
     & \le \nu^{-\frac{1}{2}}I_{\nu,Stein}(\rho_{u,t}|\pi)^{\frac{1}{2}} \lv  k(x,\cdot)-k(y,\cdot) \rv_{\mc{H}_k^d},
 \end{align*}
 where the second identity follows from the reproducing property and the last inequality follows from \eqref{eq:inverse gradient RKHS bound}. Furthermore, we can write
 \begin{align*}
     \lv  k(x,\cdot)-k(y,\cdot) \rv_{\mc{H}_k^d}^2&=k(x,x)-2k(x,y)+k(y,y)\\
     &\le \left( 2\lv \nabla_1\nabla_2 k \rv_\infty+3\lv \nabla^2 k \rv_\infty\right) |x-y|^2,
 \end{align*}
 where the first identity follows from the RKHS property and the second identity follows from Taylor expansion and Assumption \ref{ass:existence and uniqueness on K}. Therefore, \eqref{eq:Lipschitz condition pointwise} holds with $L(t)$ defined in \eqref{eq:Lipschizt constant pointwise}.
\end{proof}
\begin{lemma}\label{lem:Lipschitz condition function space}
Under the assumptions in Theorem \ref{thm:regularized mf PDE unique and existence}, let $S_r=\mc{C}([0,T];Y_r)$ with $Y_r$ defined in \eqref{eq:function space  CF local}. Then for any $t\in [0,T]$, there exists $C_1(t)>0$ such that for any $u,v\in S_r$,
\begin{align}\label{eq:Lipschitz condition function space}
    \sup_{x\in \mb{R}^d } \left| \mc{D}_{\nu,\rho_{u,t}}\nabla \log \frac{\rho_{u,t}}{\pi}(x)-\mc{D}_{\nu,\rho_{v,t}}\nabla \log \frac{\rho_{v,t}}{\pi}(x) \right|\le C_1(t) d_S(u,v),
\end{align}
where for all $t\in [0,T]$, $\rho_{u,t}=(u(t,\cdot))_{\#}\rho_0$, $\rho_{v,t}=(v(t,\cdot))_{\#}\rho_0$ and
\begin{align}\label{eq:Lipschitz constant function space}
     C_1(t)= 2\nu^{-\frac{3}{2}}(1-\nu)\lv k \rv_\infty^2\left( 2\lv \nabla_1\nabla_2 k\rv_\infty+3\lv \nabla^2 k \rv_\infty \right)^{\frac{1}{2}}  \lv \nabla k \rv_\infty I_{\nu,Stein}(\rho_{u,t}|\pi)^{\frac{1}{2}}+\nu^{-1}L_r \lv k \rv_\infty ,
\end{align}
with $L_r$ being defined in \eqref{eq:parameter Lr}.
\end{lemma}
\begin{proof}[Proof of Lemma~\ref{lem:Lipschitz condition function space}]\label{pf:Lipschitz condition function space} 
 With the facts that $\mc{D}_{\nu,\mu}=\iota_{k,\mu}(\iota_{k,\mu}^*\iota_{k,\mu}+\nu I)^{-1}\iota_{k,\mu}^*$ and $\iota_{k,\mu}$ is the inclusion operator we get,
    \begin{align*}
       & \quad \left| \mc{D}_{\nu,\rho_{u,t}}\nabla \log \frac{\rho_{u,t}}{\pi}(x)-\mc{D}_{\nu,\rho_{v,t}}\nabla \log \frac{\rho_{v,t}}{\pi}(x) \right| \\
       &= \left|\left((1-\nu)\iota_{k,\rho_{u,t}}^*\iota_{k,\rho_{u,t}}+\nu I\right)^{-1}\iota_{k,\rho_{u,t}}^*\nabla \log \frac{\rho_{u,t}}{\pi}(x)\right.\\
       &\qquad\qquad\left. -\left((1-\nu)\iota_{k,\rho_{v,t}}^*\iota_{k,\rho_{v,t}}+\nu I\right)^{-1}\iota_{k,\rho_{v,t}}^*\nabla \log \frac{\rho_{v,t}}{\pi}(x)\right|\\
        &\le \left|\left( \left((1-\nu)\iota_{k,\rho_{u,t}}^*\iota_{k,\rho_{u,t}}+\nu I\right)^{-1}-\left((1-\nu)\iota_{k,\rho_{v,t}}^*\iota_{k,\rho_{v,t}}+\nu I\right)^{-1}  \right)\iota_{k,\rho_{u,t}}^*\nabla \log \frac{\rho_{u,t}}{\pi}(x) \right|\\
        &\ \qquad\qquad+\left|\left((1-\nu)\iota_{k,\rho_{v,t}}^*\iota_{k,\rho_{v,t}}+\nu I\right)^{-1}  \left( \iota_{k,\rho_{u,t}}^*\nabla \log \frac{\rho_{u,t}}{\pi}(x)-\iota_{k,\rho_{v,t}}^*\nabla \log \frac{\rho_{v,t}}{\pi}(x) \right) \right|.
    \end{align*}
    We then turn to study the two terms in the upper bound separately.\\

\noindent \textbf{First term:}
Note that, we have 
\begin{align}\label{eq:bound term 1 Lipschitz condition function space}
     & \quad \left|\left( \left((1-\nu)\iota_{k,\rho_{u,t}}^*\iota_{k,\rho_{u,t}}+\nu I\right)^{-1}-\left((1-\nu)\iota_{k,\rho_{v,t}}^*\iota_{k,\rho_{v,t}}+\nu I\right)^{-1}  \right)\iota_{k,\rho_{u,t}}^*\nabla \log \frac{\rho_{u,t}}{\pi}(x) \right| \nonumber\\
     &=\left|\iota_{k,\rho_{u,t}}\left((1-\nu)\iota_{k,\rho_{v,t}}^*\iota_{k,\rho_{v,t}}+\nu I\right)^{-1}\left((1-\nu)\iota_{k,\rho_{v,t}}^*\iota_{k,\rho_{v,t}}-(1-\nu)\iota_{k,\rho_{u,t}}^*\iota_{k,\rho_{u,t}}\right)\right. \nonumber\\
     &  \qquad \qquad \left. \times\left((1-\nu)\iota_{k,\rho_{u,t}}^*\iota_{k,\rho_{u,t}}+\nu I\right)^{-1}\iota_{k,\rho_{u,t}}^*\nabla \log \frac{\rho_{u,t}}{\pi}(x)\right| \nonumber \\
     &\le \lv \iota_{k,\rho_{u,t}} \rv_{\mc{H}_k^d\to L_\infty^d} \lv \left((1-\nu)\iota_{k,\rho_{v,t}}^*\iota_{k,\rho_{v,t}}+\nu I\right)^{-1} \rv_{\mc{H}_k^d\to \mc{H}_k^d} (1-\nu)\lv \iota_{k,\rho_{v,t}}^*\iota_{k,\rho_{v,t}}-\iota_{k,\rho_{u,t}}^*\iota_{k,\rho_{u,t}} \rv_{\mc{H}_k^d\to\mc{H}_k^d} \nonumber \\
     &\qquad \qquad\times \lv \left((1-\nu)\iota_{k,\rho_{u,t}}^*\iota_{k,\rho_{u,t}}+\nu I\right)^{-1}\iota_{k,\rho_{u,t}}^*\nabla \log \frac{\rho_{u,t}}{\pi} \rv_{\mc{H}_k^d} \nonumber \\
     &\le  \lv k \rv_\infty \nu^{-1}(1-\nu)  \nu^{-\frac{1}{2}} I_{\nu,Stein}(\rho_{u,t}|\pi)^{\frac{1}{2}} \nonumber\\
     &\qquad\qquad \times \sup_{\lv \phi \rv_{\mc{H}_k^d}=1} \leftlangle \int_{\mb{R}^d} k(\cdot,x)\phi(x)(d\rho_{u,t}(x)-d\rho_{v,t}(x)), \int_{\mb{R}^d} k(\cdot,y)\phi(y)(d\rho_{u,t}(y)-d\rho_{v,t}(y)) \rightrangle_{\mc{H}_k^d}^{\frac{1}{2}} \nonumber \\
     %&\le  \lv k \rv_\infty \nu^{-1}(1-\nu)  \nu^{-\frac{1}{2}} I_{\nu,Stein}(\rho_{1,t}|\pi)^{\frac{1}{2}} \nonumber\\
    % &\quad \left(\sup_{\lv \phi \rv_{\mc{H}_k^d}=1}  \int_{\mb{R}^d}\int_{\mb{R}^d} k(x,y)\langle \phi(x), \phi(y) \rangle (d\rho_{1,t}(x)-d\rho_{2,t}(x))(d\rho_{1,t}(y)-d\rho_{2,t}(y))\right)^{\frac{1}{2}} \nonumber \\
     &\le 2 \nu^{-\frac{3}{2}}(1-\nu)\lv k \rv_\infty^2\left( 3\lv \nabla^2 k \rv_\infty+2\lv \nabla_1\nabla_2 k\rv_\infty \right)^{\frac{1}{2}} I_{\nu,Stein}(\rho_{u,t}|\pi)^{\frac{1}{2}} d_{S}(u,v).
     %&\le  \lv k \rv_\infty \nu^{-\frac{3}{2}}(1-\nu)I_{\nu,Stein}(\rho_{1,t}|\pi)^{\frac{1}{2}} \left[\int_{\mb{R}^d}\int_{\mb{R}^d} k(x,y)^2 (d\rho_{1,t}(x)-d\rho_{2,t}(x))(d\rho_{1,t}(y)-d\rho_{2,t}(y))\right]^{1/2} \nonumber \\
    %&\le  \lv k \rv_\infty \nu^{-\frac{3}{2}}(1-\nu)I_{\nu,Stein}(\rho_{1,t}|\pi)^{\frac{1}{2}} \left[\int_{\mb{R}^d}\int_{\mb{R}^d} k(x,x)k(y,y) (d\rho_{1,t}(x)-d\rho_{2,t}(x))(d\rho_{1,t}(y)-d\rho_{2,t}(y))\right]^{1/2} \nonumber \\
    % &\le 2 \nu^{-\frac{3}{2}}(1-\nu)\lv k \rv_\infty^2 \lv \nabla k \rv_\infty I_{\nu,Stein}(\rho_{1,t}|\pi)^{\frac{1}{2}} d_S(u,v).
    \end{align}
    % \bk{I do not think the inequality above (48) is needed. The key inequality is the one two above (48) involving inner product. One can directly go from there to the rhs of the first inequality in the below chain of inequalities.}\ye{I agree. It's fixed now.}
As we are bounding the function value by its $L_\infty^d$ norm, the second step allows the function to be in the space of $L_\infty^d$, without which we think of the function as belonging to the RKHS. The second inequality follows from \eqref{eq:inclusion operator bound} and \eqref{eq:inverse gradient RKHS bound}. %The third inequality follows from the reproducing property of the RKHS. The fourth inequality follows from the fact that $k(x,y)=\langle k(x,\cdot), k(\cdot,y) \rangle_{\mc{H}_k}\le \sqrt{k(x,x)k(y,y)}$. 
The last inequality follows from the fact that
     \begin{align*}
     &\quad  \leftlangle \int_{\mb{R}^d} k(\cdot,x)\phi(x)(d\rho_{u,t}(x)-d\rho_{v,t}(x)), \int_{\mb{R}^d} k(\cdot,y)\phi(y)(d\rho_{u,t}(y)-d\rho_{v,t}(y)) \rightrangle_{\mc{H}_k^d}^{\frac{1}{2}}\\
     &=\bigg( \sup_{\lv \phi \rv_{\mc{H}_k^d}=1}  \int_{\mb{R}^d}\int_{\mb{R}^d} \bigg\langle k\left(u(t,x),\cdot\right)\phi\left(u(t,x)\right)-k\left(v(t,x),\cdot\right)\phi\left(v(t,x)\right), \\
     &\qquad\qquad\qquad  k\left(u(t,y),\cdot\right)\phi\left(u(t,y)\right)-k\left(v(t,y),\cdot\right)\phi\left(v(t,y)\right)\bigg\rangle_{\mc{H}_k^d} d\rho_0(x)d\rho_0(y) \bigg)^{\frac{1}{2}}\\
     &\le\sup_{\lv \phi \rv_{\mc{H}_k^d}=1}   \int_{\mb{R}^d} \lv k\left(u(t,x),\cdot\right)\phi\left(u(t,x)\right)-k\left(v(t,x),\cdot\right)\phi\left(v(t,x)\right) \rv_{\mc{H}_k^d} d\rho_0(x) \\
     &\le \sup_{\lv \phi \rv_{\mc{H}_k^d}=1}   \int_{\mb{R}^d} \lv \left(k\left(u(t,x),\cdot\right)-k\left(v(t,x),\cdot\right)\right)\phi\left(u(t,x)\right) \rv_{\mc{H}_k^d}d\rho_0(x) \\
     &\qquad+\sup_{\lv \phi \rv_{\mc{H}_k^d}=1}   \int_{\mb{R}^d}\lv k\left(v(t,x),\cdot\right)\left(\phi\left(u(t,x)\right)-\phi\left(v(t,x)\right)\right) \rv_{\mc{H}_k^d}  d\rho_0(x)\\
     &=\sup_{\lv \phi \rv_{\mc{H}_k^d}=1}   \int_{\mb{R}^d} \lv \left(k\left(u(t,x),\cdot\right)-k\left(v(t,x),\cdot\right)\right) \rv_{\mc{H}_k}\left|\langle \phi(\cdot),k\left( u\left(t,x\right),\cdot \right)\rangle_{\mc{H}_k}\right|d\rho_0(x)\\
     &\qquad + \sup_{\lv \phi \rv_{\mc{H}_k^d}=1}   \int_{\mb{R}^d}\lv k\left(v(t,x),\cdot\right) \rv_{\mc{H}_k}\left|\langle \phi(\cdot),k\left(u(t,x),\cdot\right)-k\left(v(t,x),\cdot\right) \rangle_{\mc{H}_k}\right|  d\rho_0(x)\\
     &\le\sup_{\lv \phi \rv_{\mc{H}_k^d}=1}   \int_{\mb{R}^d} \lv \left(k\left(u(t,x),\cdot\right)-k\left(v(t,x),\cdot\right)\right) \rv_{\mc{H}_k}\lv\phi\rv_{\mc{H}_k^d}\lv k\left( u\left(t,x\right),\cdot \right)\rv_{\mc{H}_k}d\rho_0(x)\\
     &\quad + \sup_{\lv \phi \rv_{\mc{H}_k^d}=1}   \int_{\mb{R}^d}\lv k\left(v(t,x),\cdot\right) \rv_{\mc{H}_k}\lv \phi\rv_{\mc{H}_k^d}\lv k\left(u(t,x),\cdot\right)-k\left(v(t,x),\cdot\right) \rv_{\mc{H}_k} d\rho_0(x)\\
     &= \sup_{x\in \mb{R}^d} \left(\sqrt{k\left(u(t,x),u(t,x)\right)+k\left(v(t,x),v(t,x)\right)-2k\left(u(t,x),v(t,x)\right)} \right.\\
&\qquad\qquad\times\left.\left(\sqrt{k\left(u(t,x),u(t,x)\right)}+\sqrt{k\left(v(t,x),v(t,x)\right)}\right)\right)\\
     &\le 2\lv k \rv_\infty \left( 3\lv \nabla^2 k \rv_\infty+2\lv \nabla_1\nabla_2 k\rv_\infty \right)^{\frac{1}{2}} d_{S}(u,v),
    \end{align*}
    where the first identity follows from the definitions of $\rho_{u,t}$ and $\rho_{v,t}$ and change of variable. The second inequality holds due to the symmetry in $x$ and $y$. The second identity follows from the reproducing property of the RKHS. The last identity follows from the fact that $\lv k(x,\cdot) \rv_{\mc{H}_k}=\sqrt{k(x,x)}$ for all $x$ and the last inequality follows from Assumption \ref{ass:existence and uniqueness on K} and Taylor expansion on both variables in $k$ up to second order.\\
    
    \vspace{0.05in}

     \noindent \textbf{Second term:} Note that we have 
  \begin{align*}
     &\quad \left|\left((1-\nu)\iota_{k,\rho_{v,t}}^*\iota_{k,\rho_{v,t}}+\nu I\right)^{-1} \left( \iota_{k,\rho_{u,t}}^*\nabla \log \frac{\rho_{u,t}}{\pi}(x)-\iota_{k,\rho_{v,t}}^*\nabla \log \frac{\rho_{v,t}}{\pi}(x) \right) \right|  \\
     &=\left|\iota_{k,\rho_{v,t}}\left((1-\nu)\iota_{k,\rho_{v,t}}^*\iota_{k,\rho_{v,t}}+\nu I\right)^{-1} \left( \iota_{k,\rho_{u,t}}^*\nabla \log \frac{\rho_{u,t}}{\pi}(x)-\iota_{k,\rho_{v,t}}^*\nabla \log \frac{\rho_{v,t}}{\pi}(x) \right) \right| \\
     &\le \lv \iota_{k,\rho_{v,t}} \rv_{\mc{H}_k^d\to L_\infty^d} \lv \left((1-\nu)\iota_{k,\rho_{v,t}}^*\iota_{k,\rho_{v,t}}+\nu I\right)^{-1} \rv_{\mc{H}_k^d\to \mc{H}_k^d} \lv \iota_{k,\rho_{u,t}}^*\nabla \log \frac{\rho_{u,t}}{\pi}-\iota_{k,\rho_{v,t}}^*\nabla \log \frac{\rho_{v,t}}{\pi} \rv_{\mc{H}_k^d}\\
     &\le \lv k \rv_\infty \nu^{-1}  \lv \iota_{k,\rho_{u,t}}^*\nabla \log \frac{\rho_{u,t}}{\pi}-\iota_{k,\rho_{v,t}}^*\nabla \log \frac{\rho_{v,t}}{\pi} \rv_{\mc{H}_k^d},
    \end{align*}
    where the last inequality follows from \eqref{eq:inclusion operator bound} and for all $x\in \mb{R}^d$,
    \begin{align*}
        &\quad \iota_{k,\rho_{u,t}}^*\nabla \log \frac{\rho_{u,t}}{\pi}(x)-\iota_{k,\rho_{v,t}}^*\nabla \log \frac{\rho_{v,t}}{\pi}(x) \\
        &=\int_{\mb{R}^d} k(x,y)\nabla \log \frac{\rho_{u,t}}{\pi}(y)d\rho_{u,t}(y)-\int_{\mb{R}^d} k(x,y)\nabla \log \frac{\rho_{v,t}}{\pi}d\rho_{v,t}(y) \\
        &= \int_{\mb{R}^d} \left( k(x,y)\nabla V(y)-\nabla_2 k(x,y) \right)d\rho_{u,t}(y)-\int_{\mb{R}^d} \left( k(x,y)\nabla V(y)-\nabla_2 k(x,y) \right)d\rho_{v,t}(y) \\
        &=\int_{\mb{R}^d} \left(k(x,u(t,y))\nabla V(u(t,y))-k(x,v(t,y))\nabla V(v(t,y))\right) d\rho_0(y)\\ &\qquad\qquad- \int_{\mb{R}^d}\left( \nabla_2 k(x,u(t,y))-\nabla_2 k(x,v(t,y)) \right) d\rho_0(y).
    \end{align*}
    Therefore, we have
    \begin{align*}
     &\lv \iota_{k,\rho_{u,t}}^*\nabla \log \frac{\rho_{u,t}}{\pi}-\iota_{k,\rho_{v,t}}^*\nabla \log \frac{\rho_{v,t}}{\pi} \rv_{\mc{H}_k^d} \\
      &\le \int_{\mb{R}^d} \Bigg(\lv k(\cdot,u(t,y))\nabla V(u(t,y))-k(\cdot,v(t,y))\nabla V(v(t,y)) \rv_{\mc{H}_k^d} \\
      &\qquad \qquad+\lv \nabla_2 k(\cdot,u(t,y))-\nabla_2 k(\cdot,v(t,y)) \rv_{\mc{H}_k^d}\Bigg) d\rho_0(y).
    \end{align*}
    For simplicity, in the following calculations, we denote $u(t,y)$ and $v(t,y)$ as $u$ and $v$ respectively. We will bound $\lv k(\cdot,u)\nabla V(u)-k(\cdot,v)\nabla V(v) \rv_{\mc{H}_k^d}$ and $\lv \nabla_2 k(\cdot,u)-\nabla_2 k(\cdot,v) \rv_{\mc{H}_k^d}$ respectively. Note that we have
    \begin{align*}
        & \quad \lv k(\cdot,u)\nabla V(u)-k(\cdot,v)\nabla V(v) \rv_{\mc{H}_k^d}^2\\
        &= \leftlangle k(\cdot,u)\nabla V(u)-k(\cdot,v)\nabla V(v), k(\cdot,u)\nabla V(u)-k(\cdot,v)\nabla V(v) \rightrangle_{\mc{H}_k^d}\\
      &=|\nabla V(u)|^2 k(u,u)-2\leftlangle \nabla V(u) , \nabla V(v) \rightrangle k(u,v)+|\nabla V(v)|^2 k(v,v)\\
      &\le\left| \leftlangle \nabla V(u)-\nabla V(v) ,   \nabla V(u)k(u,u)-\nabla V(v)k(v,v) \rightrangle \right|\\
      &\quad +\left| \leftlangle \nabla V(u) , \nabla V(v)\rightrangle \left( k(u,u)+k(v,v)-2k(u,v) \right) \right|,
    \end{align*}
    where 
        \begin{align*}
        &\quad  \left| \leftlangle \nabla V(u)-\nabla V(v) ,   \nabla V(u)k(u,u)-\nabla V(v)k(v,v) \rightrangle \right|\\
        &\le \left|\nabla V(u)-\nabla V(v)\right|^2 k(u,u)+\left|\nabla V(u)-\nabla V(v)\right|\left|\nabla V(v)\right|\left|k(u,u)-k(v,v) \right|\\
        &\le C_{1,r}^2 (1+V(y))^2 d_{S_T}(u,v)^2 k(u,u) +C_{1,r}^2(1+V(y))^2 d_{S}(u,v)\left|k(u,u)-k(v,v)\right|\\
        &\le  C_{1,r}^2 (1+V(y))^2 d_{S}(u,v)^2 \lv k \rv_\infty^2+2C_{1,r}^2 \lv \nabla k \rv_\infty (1+V(y))^2 d_{S}(u,v)^2.
    \end{align*}
    The second inequality follows from Assumption \ref{ass:existence and uniqueness on V} and the last inequality follows from Assumption \ref{ass:existence and uniqueness on K} and Taylor expansion on both variables in $k$ up to first order. And, we also have 
    \begin{align*}
        &\quad \left| \leftlangle \nabla V(u) , \nabla V(v)\rightrangle \left( k(u,u)+k(v,v)-2k(u,v) \right) \right|\\
        &\le C_{1,r}^2(1+V(y))^2 \left|k(u,u)+k(v,v)-2k(u,v)\right|\\
        &\le C_{1,r}^2(1+V(y))^2\left( 3\lv \nabla^2 k \rv_\infty+2\lv \nabla_1\nabla_2 k\rv_\infty \right) d_{S}(u,v)^2,
    \end{align*}
    where the first inequality follows from Assumption \ref{ass:existence and uniqueness on V} and the last inequality follows from Assumption \ref{ass:existence and uniqueness on K} and Taylor expansion on both variables in $k$ up to second order. With the above two inequalities, we have
    \begin{align}\label{eq:term 2 bound 1}
   & \quad \lv k(\cdot,u(t,y))\nabla V(u(t,y))-k(\cdot,v(t,y))\nabla V(v(t,y)) \rv_{\mc{H}_k^d}   \nonumber \\
   &\le C_{1,r}\left( \lv k \rv_\infty+2{\lv \nabla k \rv_\infty}^{\frac{1}{2}}+3{\lv \nabla^2 k \rv_\infty}^{\frac{1}{2}}+2{\lv \nabla_1\nabla_2 k \rv_\infty}^{\frac{1}{2}} \right) d_{S}(u,v) (1+V(y)) .
    \end{align}
    Observe that, for all $x,y\in \mb{R}^d$
    \begin{align*}
        \langle \nabla_2 k(\cdot,x), \nabla_2 k(\cdot, y) \rangle_{\mc{H}_k^d}&=\nabla_1\cdot\nabla_2 \langle k(\cdot,x),k(\cdot,y) \rangle_{\mc{H}_k^d}\\
        &=\nabla_1\cdot\nabla_2 k(x,y).
    \end{align*}
    If we denote the function $\nabla_1\cdot\nabla_2 k=D_{1,2}k$ where $D_{1,2}k$ is symmetric since $k$ is symmetric, we get
    \begin{align*}
     \lv \nabla_2 k(\cdot,u)-\nabla_2 k(\cdot,v) \rv_{\mc{H}_k^d}^2&=D_{1,2}k(u,u)+ D_{1,2}k(v,v)-2D_{1,2}k(u,v)  \\
     &\le \left(2\lv \nabla^2(D_{1,2}k) \rv_\infty+\lv \nabla_1\nabla_2(D_{1,2}k) \rv_\infty \right)d_{S}(u,v)^2.  
    \end{align*}
    where the inequality follows from Taylor expansion on both variables of $D_{1,2}k$. Therefore,
    \begin{align}\label{eq:term 2 bound 2}
      \lv \nabla_2 k(\cdot,u)-\nabla_2 k(\cdot,v) \rv_{\mc{H}_k^d}\le  \left(2\lv \nabla^2(D_{1,2}k) \rv_\infty^{\frac{1}{2}}+\lv \nabla_1\nabla_2(D_{1,2}k) \rv_\infty^{\frac{1}{2}} \right)d_{S}(u,v).   
    \end{align}
    According to \eqref{eq:term 2 bound 1} and \eqref{eq:term 2 bound 2}, we get 
    \begin{align*}
      \lv \iota_{k,\rho_{u,t}}^*\nabla \log \frac{\rho_{u,t}}{\pi}-\iota_{k,\rho_{v,t}}^*\nabla \log \frac{\rho_{v,t}}{\pi} \rv_{\mc{H}_k^d}\le L_r d_s(u,v)   
    \end{align*}
    with
    \begin{align}\label{eq:parameter Lr}
    \begin{aligned}
    L_r &= C_{1,r}\lv \rho_0 \rv_{\mc{P}_V}\left(\lv k \rv_\infty+2{\lv \nabla k \rv_\infty}^{\frac{1}{2}}+3{\lv \nabla^2 k \rv_\infty}^{\frac{1}{2}}+2{\lv \nabla_1\nabla_2 k \rv_\infty}^{\frac{1}{2}}\right) \\ 
    &\quad\qquad + 2{\lv \nabla^2(D_{1,2}k) \rv_\infty}^{\frac{1}{2}}+{\lv \nabla_1\nabla_2(D_{1,2}k) \rv_\infty}^{\frac{1}{2}}.
    \end{aligned}
\end{align}
Therefore, the second term is bounded as
\begin{align}\label{eq:bound term 2 Lipschitz function space}
  \left|\left((1-\nu)\iota_{k,\rho_{v,t}}^*\iota_{k,\rho_{v,t}}+\nu I\right)^{-1} \left( \iota_{k,\rho_{u,t}}^*\nabla \log \frac{\rho_{u,t}}{\pi}(x)-\iota_{k,\rho_{v,t}}^*\nabla \log \frac{\rho_{v,t}}{\pi}(x) \right) \right| \le \nu^{-1} L_r \lv k \rv_\infty d_S(u,v).
\end{align}
The Lemma is proved based on \eqref{eq:bound term 1 Lipschitz condition function space} and \eqref{eq:bound term 2 Lipschitz function space}.
%Then according to Proposition \ref{prop:append 1}, we have
%\begin{align*}
%     \sup_{x\in \mb{R}^d } \left| \mc{D}_{\nu,\rho_{1,t}}\nabla \log \frac{\rho_{1,t}}{\pi}(x)-\mc{D}_{\nu,\rho_{2,t}}\nabla \log \frac{\rho_{2,t}}{\pi}(x) \right|\le C_1(t) d_S(u,v)
%\end{align*}
%with $C_1(t)$ being defined in \eqref{eq:Lipschitz constant function space}.
\end{proof}    
\section{Stability}\label{sec:stable}
In this section, we prove a stability estimate for the weak solutions to \eqref{eq:regularized SVGD mf PDE}. To do this, we introduce a space of probability measures on $\mb{R}^d$ and assumptions on $V$ as follows,
\begin{align*}%\label{eq:space Pp}
    \mc{P}_p:=\left\{ \rho\in \mc{P}:\lv \rho \rv_{\mc{P}_p}:=\int_{\mb{R}^d} |x|^p \rho(x)dx<\infty \right\},
\end{align*}
where $\mc{P}$ denotes the set of all probability measures on $\mb{R}^d$.
\begin{myassump}{V2}\label{ass:existence and uniqueness on V V2} In addition to Assumption \ref{ass:existence and uniqueness on V}, there exists a constant $C_V>0$ and $q>1$ such that $|\nabla V(x)|^q \le C_V (1+V(x))$ for all $x\in \mb{R}^d$ and $\sup_{\theta\in [0,1]} \lv \nabla^2 V (\theta x+(1-\theta)y)\rv_2^q \le C_V (1+V(x)+V(y)) $.
\end{myassump}
\begin{theorem}\label{thm:stability Wp} Let $V$ satisfy Assumption \ref{ass:existence and uniqueness on V V2} with $q\in  (1,\infty)$ and $k$ satisfy Assumption \ref{ass:existence and uniqueness on K}. Let $p$ be the conjugate of $q$, i.e., $p^{-1}+q^{-1}=1$. Let $\rho_1,\rho_2\in \mc{P}_p$ be two initial probability measures satisfying $\lv \rho_i \rv_{\mc{P}_p}\le R$ for some constant $R>0$ and $i=1,2$. Let $\rho_{1,t}$ and $\rho_{2,t}$ be the associated weak solution to \eqref{eq:regularized SVGD mf PDE}. Then given any $T>0$, there exists a constant $C>0$ depending on $k,V,q,\nu,\rho_1,\rho_2$ such that,
 \begin{align*}%\label{eq:stability Wp}
     \sup_{t\in [0,T]} \mc{W}_p(\rho_{1,t},\rho_{2,t})\le C\mc{W}_p(\rho_1,\rho_2).
 \end{align*}
 More explicitly, the constant $C$ is given by
 \begin{align}\label{eq:stability constant} 
     C=&\exp\bigg( \nu^{-1}\lv k \rv_\infty C(T,k,V,\nu,\rho_1,\rho_2,q)T++ \nu^{-\frac{1}{2}} \left(2\lv \nabla_1\nabla_2 k \rv_\infty+3\lv \nabla^2 k \rv_\infty\right)^{\frac{1}{2}}\sqrt{ \KL(\rho_2|\pi) } T^{\frac{1}{2}} \nonumber \\
     &\qquad\qquad +2\nu^{-\frac{3}{2}}(1-\nu)\lv k \rv_\infty^2\left( 2\lv \nabla_1\nabla_2 k+3\lv \nabla^2 k \rv_\infty\rv_\infty \right)^{\frac{1}{2}}\sqrt{\KL(\rho_1|\pi)} T^{\frac{1}{2}} \bigg),
 \end{align}
 where $C(T,k,V,\nu,\rho_1,\rho_2,q)$ is defined in \eqref{eq:stability proof bound 5}.
 \end{theorem}

\textcolor{black}{
Before we prove Theorem \ref{thm:stability Wp}, we introduce the following corollary to Theorem \ref{thm:stability Wp}.  
\begin{corollary}\label{cor:stability} Suppose the assumptions in Theorem \ref{thm:stability Wp} hold. Let $\rho_1 \in \mc{P}_p$ and $\rho_1^N$ be the empirical measure formed from $N$ samples from $\rho_1$. Then for any $T>0$,
\begin{align*}
    \sup_{t\in [0,T]} \mc{W}_p(\rho_{1,t},\rho_{1,t}^N)\to 0, \qquad \qquad\text{as }N\to\infty,
\end{align*}  
where $\rho_{1,t}$ and $\rho_{1,t}^N$ are the unique weak solutions to \eqref{eq:regularized SVGD mf PDE} with initial conditions $\rho_1$ and $\rho_1^N$ respectively.
\end{corollary}
Since $\mc{W}_p(\rho_1^N,\rho_1)\to 0$ as $N\to\infty$, Corollary \ref{cor:stability} follows directly from Theorem \ref{thm:stability Wp}. Non-asymptotic bounds on $\sup_{t\in [0,T]}\mc{W}_p(\rho_{1,t},\rho_{1,t}^N)$ that are non-uniform in time $T$ (i.e., for $T<\infty$) in terms of $N$ can also be obtained based on Theorem \ref{thm:stability Wp} and the convergence of empirical measures in $\mc{W}_p$~\cite{fournier2015rate,weed2019sharp}. We leave that to the interested reader.}

\begin{remark} \label{nustability}
If we focus on the dependency on $\nu$ and $T$ in \eqref{eq:stability constant}, we have $$C\le C'\exp\left( \nu^{-1}T\exp\left( C' \nu^{-\frac{1}{2}}T^{\frac{1}{2}} \right)+\nu^{-\frac{3}{2}}(1-\nu)T^{\frac{1}{2}}+\nu^{-\frac{1}{2}}T^{\frac{1}{2}} \right)$$
where $C'$ is a constant independent of $\nu$ and $T$. \textcolor{black}{In particular, the stability constant blows up when $\nu\to 0$, i.e. the smaller is $\nu$, the larger is the $\mc{W}_p$ distance between $\rho_{1,t}$ and $\rho_{1,t}^N$ in Corollary \ref{cor:stability}. Therefore, if we choose smaller $\nu$ in the finite-particle algorithm, the R-SVGD, we will need more particles (larger $N$) in order to get the comparable iteration complexities.}
\end{remark}

The proof of Theorem \ref{thm:stability Wp} is inspired by that of \cite[Theorem 2.7]{lu2019scaling} which in turn is motivated by the Dobrushin's coupling argument (see, for example,~\cite{dobrushin1979vlasov} and~\cite[Theorem 1.4.1]{muntean2016macroscopic}). In the following proof, we mainly highlight the parts of our proof that are different from the proof of \cite[Theorem 2.7]{lu2019scaling}.
 \begin{proof}[Proof of Theorem~\ref{thm:stability Wp}]
 First, under Assumption \ref{ass:existence and uniqueness on V V2}, there exists a constant $C_0>0$ such that $V(x)\le C_0(1+|x|^p)$ for all $x\in \mb{R}^d$. Therefore, $\mc{P}_p\subset \mc{P}_V$  and $\lv \rho_i \rv_{\mc{P}_V}\le C(R)<\infty$ for $i=1,2$. By Theorem \ref{thm:regularized mf PDE unique and existence}, the weak solutions take the form
 \begin{align*}
     \rho_{i,t}=\left( \Phi(t,\cdot,\rho_i) \right)_{\textcolor{black}{\#}\rho_i},\quad i=1,2  
 \end{align*}
 where $\Phi(\cdot,\cdot,\rho_i)$ solves \eqref{eq:CF of the regularized PDE} with $\rho_0=\rho_i$.
 Let $\mathbf{\rho}^{1,2}$ be a coupling measure between $\rho_1$ and $\rho_2$. \textcolor{black}{Notice that, according to conditions in Theorem \ref{thm:stability Wp}, $p>1$. Define $\phi(x)=\frac{1}{p}|x|^p$ and observe that $p>1$. We have that $\phi$ is continuously differentiable on $\mb{R}^d$ with
\begin{align*}
~|\nabla \phi(x)|= |x|^{p-1},\quad\text{for all}\quad x \in \mb{R}^d.
\end{align*}
To see this, note that as $p>1$, for any $x\neq 0$, $\nabla \phi(x)=|x|^{p-2}x$ and $|\nabla \phi(x)|=|x|^{p-1}$. At $x=0$, by definition, we have $\nabla \phi(0)=0$. Since $p>1$, $\nabla\phi(x)$ is continuous.
}
% \bk{Is the differentiability mentioned above true? It is true for even powers but I do not think it is true for odd powers.}
Then, we start from estimating the derivative of $\phi$ in the time variable, for which we have
 \begin{align*}
     &\partial_t \textcolor{black}{\phi}(\Phi(t,x_1,\rho_1)-\Phi(t,x_2,\rho_2))\\
     =&-\nabla \textcolor{black}{\phi}(\Phi(t,x_1,\rho_1)-\Phi(t,x_2,\rho_2))\left( \mc{D}_{\nu,\rho_{1,t}}\nabla\log \frac{\rho_{1,t}}{\pi}(\Phi(t,x_1,\rho_1))-\mc{D}_{\nu,\rho_{2,t}}\nabla\log \frac{\rho_{2,t}}{\pi}(\Phi(t,x_2,\rho_2)) \right).
 \end{align*}
 The next step is to estimate $$\left|\mc{D}_{\nu,\rho_{1,t}}\nabla\log \frac{\rho_{1,t}}{\pi}(\Phi(t,x_1,\rho_1))-\mc{D}_{\nu,\rho_{2,t}}\nabla\log \frac{\rho_{2,t}}{\pi}(\Phi(t,x_2,\rho_2))\right|.$$ Note that
 
 \iffalse The way to bound it is similar to the derivation of the Lipschitz constants $C_1(t)$ and $L(t)$. The difference is that here we compare the characteristic flows that solve \eqref{eq:CF of the regularized PDE} with different initial condition while there we compared two characteristic flows starting from the same initial condition. As a result of this difference, $\mathbf{\rho}^{1,2}$ would be involved in the upper bound here. 
 \fi
 
 \begin{align*}
     \mc{D}_{\nu,\rho_{1,t}}\nabla\log \frac{\rho_{1,t}}{\pi}(\Phi(t,x_1,\rho_1))-\mc{D}_{\nu,\rho_{2,t}}\nabla\log \frac{\rho_{2,t}}{\pi}(\Phi(t,x_2,\rho_2)):=I_1+I_2,
 \end{align*}
 where
 \begin{align*}
     I_1&:=\mc{D}_{\nu,\rho_{1,t}}\nabla\log \frac{\rho_{1,t}}{\pi}(\Phi(t,x_1,\rho_1))-\mc{D}_{\nu,\rho_{2,t}}\nabla\log \frac{\rho_{2,t}}{\pi}(\Phi(t,x_1,\rho_1)),\\
     I_2&:=\mc{D}_{\nu,\rho_{2,t}}\nabla\log \frac{\rho_{2,t}}{\pi}(\Phi(t,x_1,\rho_1))-\mc{D}_{\nu,\rho_{2,t}}\nabla\log \frac{\rho_{2,t}}{\pi}(\Phi(t,x_2,\rho_2)).
 \end{align*}
 According to Proposition \ref{prop:Lipschitz in stability Wp}, we have 
 \begin{align*}
      & \left|-\nabla \textcolor{black}{\phi}(\Phi(t,x_1,\rho_1)-\Phi(t,x_2,\rho_2)) \cdot I_1\right|\\
      &\le 2\nu^{-\frac{3}{2}}(1-\nu)\lv k \rv_\infty^2\left( 2\lv \nabla_1\nabla_2 k\rv_\infty+3\lv \nabla^2 k \rv_\infty \right)^{\frac{1}{2}}  I_{\nu,Stein}(\rho_{1,t}|\pi)^{\frac{1}{2}}|\Phi(t,x_1,\rho_1)-\Phi(t,x_2,\rho_2)|^{p-1} \\
      &\qquad\qquad\qquad\times\left(\int_{\mb{R}^d} |\Phi(t,y_1,\rho_1)-\Phi(t,y_2,\rho_2)|\mathbf{\rho}^{1,2}(dy_1,dy_2)\right) \\
      &\quad+\nu^{-1}\lv k \rv_\infty C(t,k,V,\nu,\rho_1,\rho_2,q) |\Phi(t,x_1,\rho_1)-\Phi(t,x_2,\rho_2)|^{p-1} \\
      &\qquad \qquad\qquad\times\left(\int_{\mb{R}^d\times\mb{R}^d}|\Phi(t,y_1,\rho_1)-\Phi(t,y_2,\rho_2)|^p \mathbf{\rho}^{1,2}(dy_1,dy_2) \right)^{1/p}
\end{align*}
and
\begin{align*}
    &\left|-\nabla \textcolor{black}{\phi}(\Phi(t,x_1,\rho_1)-\Phi(t,x_2,\rho_2)) \cdot I_2\right| \\
    &\quad\le \nu^{-\frac{1}{2}} \left( 2\lv \nabla_1\nabla_2 k \rv_\infty+3\lv \nabla^2 k \rv_\infty\right)^{\frac{1}{2}}  {I_{\nu,Stein}(\rho_{2,t}|\pi)}^{\frac{1}{2}}|\Phi(t,x_1,\rho_1)-\Phi(t,x_2,\rho_2)|^p.
 \end{align*}
Now, defining 
 \begin{align*}
     D_p(\mathbf{\rho}^{1,2})(s):=\left(\int_{\mb{R}^d\times\mb{R}^d}|\Phi(s,y_1,\rho_1)-\Phi(s,y_2,\rho_2)|^p \mathbf{\rho}^{1,2}(dy_1,dy_2) \right)^{1/p},
 \end{align*}
 we have, for any $t\in [0,T]$ that 
 \begin{align*}
     & \textcolor{black}{\phi}(\Phi(t,x_1,\rho_1)-\Phi(t,x_2,\rho_2)) \\
     &\le  \nu^{-\frac{1}{2}} \left( 2\lv \nabla_1\nabla_2 k \rv_\infty+3\lv \nabla^2 k \rv_\infty\right)^{\frac{1}{2}} \int_0^t{I_{\nu,Stein}(\rho_{2,s}|\pi)}^{\frac{1}{2}}|\Phi(s,x_1,\rho_1)-\Phi(s,x_2,\rho_2)|^p ds\\
     & +\textcolor{black}{\phi}(x_1-x_2)+\nu^{-1}\lv k \rv_\infty C(T,k,V,\nu,\rho_1,\rho_2,q) \int_0^t |\Phi(s,x_1,\rho_1)-\Phi(s,x_2,\rho_2)|^{p-1}D_p(\mathbf{\rho}^{1,2})(s)ds\\
     & +2\nu^{-\frac{3}{2}}(1-\nu)\lv k \rv_\infty^2\left( 2\lv \nabla_1\nabla_2 k\rv_\infty+3\lv \nabla^2 k \rv_\infty \right)^{\frac{1}{2}}\\
     &\qquad\qquad\times\int_0^t I_{\nu,Stein}(\rho_{1,s}|\pi)^{\frac{1}{2}}|\Phi(s,x_1,\rho_1)-\Phi(s,x_2,\rho_2)|^{p-1}D_p(\mathbf{\rho}^{1,2})(s)ds.
      \end{align*}
 Integrating the above inequality w.r.t. the coupling $\mathbf{\rho}^{1,2}$, and using the fact that
 \begin{align*}
     \int_{\mb{R}^d\times \mb{R}^d} |\Phi(t,x_1,\rho_1)-\Phi(t,x_2,\rho_2)|^{p-1}\mathbf{\rho}^{1,2}(dx_1,dx_2)\le D_p(\mathbf{\rho}^{1,2})(\textcolor{black}{t})^{p-1},
 \end{align*}
we get
 \begin{align*}
     D_p(\mathbf{\rho}^{1,2})(t)^p
     &\le D_p(\mathbf{\rho}^{1,2})(0)^p+\nu^{-1}\lv k \rv_\infty C(T,k,V,\nu,\rho_1,\rho_2,q) \int_0^t D_p(\mathbf{\rho}^{1,2})(s)^p ds\\
     & + 2\nu^{-\frac{3}{2}}(1-\nu)\lv k \rv_\infty^2\left( 2\lv \nabla_1\nabla_2 k\rv_\infty+3\lv \nabla^2 k \rv_\infty \right)^{\frac{1}{2}}  \int_0^t{I_{\nu,Stein}(\rho_{1,s}|\pi)}^{\frac{1}{2}} D_p(\mathbf{\rho}^{1,2})(s)^p ds\\
     &\quad + \left(\nu^{-\frac{1}{2}} \left( 2\lv \nabla_1\nabla_2 k \rv_\infty+3\lv \nabla^2 k \rv_\infty\right)^{\frac{1}{2}}\right)\int_0^t{I_{\nu,Stein}(\rho_{2,s}|\pi)}^{\frac{1}{2}} D_p(\mathbf{\rho}^{1,2})(s)^p ds.
 \end{align*}
 By using Gronwall's inequality, we further obtain 
 \begin{align*}
     &\quad\quad~D_p(\mathbf{\rho}^{1,2})(t)^p\\
     & \le D_p(\mathbf{\rho}^{1,2})(0)^p \exp\left( \nu^{-1}\lv k \rv_\infty C(T,k,V,\nu,\rho_1,\rho_2,q) t\right)  \\
     &\quad \qquad  \times\exp\left( 2\nu^{-\frac{3}{2}}(1-\nu)\lv k \rv_\infty^2\left( 2\lv \nabla_1\nabla_2 k\rv_\infty+3\lv \nabla^2 k \rv_\infty \right)^{\frac{1}{2}}  \int_0^t{I_{\nu,Stein}(\rho_{1,s}|\pi)}^{\frac{1}{2}}ds \right)\\
     &\quad \qquad \qquad\times\exp\left(\nu^{-\frac{1}{2}} \left( 2\lv \nabla_1\nabla_2 k \rv_\infty+3\lv \nabla^2 k \rv_\infty\right)^{\frac{1}{2}} \int_0^t{I_{\nu,Stein}(\rho_{2,s}|\pi)}^{\frac{1}{2}}ds \right)\\
     &\le D_p(\mathbf{\rho}^{1,2})(0)^p \exp\Bigg( \nu^{-\frac{1}{2}} \left(2\lv \nabla_1\nabla_2 k \rv_\infty+3\lv \nabla^2 k \rv_\infty\right)^{\frac{1}{2}}\sqrt{ \KL(\rho_2|\pi) } t^{\frac{1}{2}}\\
     &  +\nu^{-1}\lv k \rv C(T,k,V,\nu,\rho_1,\rho_2,q)t+2\nu^{-\frac{3}{2}}(1-\nu)\lv k \rv_\infty^2\left( 2\lv \nabla_1\nabla_2 k\rv_\infty+3\lv \nabla^2 k \rv_\infty \right)^{\frac{1}{2}} \sqrt{\KL(\rho_1|\pi)}t^{\frac{1}{2}} \Bigg).
 \end{align*}
Hence, we obtain 
 \begin{align*}
     &\mc{W}_p^p(\rho_{1,t},\rho_{2,t})\\
     &=\inf_{\pi\in \Gamma(\rho_{1,t},\rho_{2,t})} \int_{\mb{R}^d\times \mb{R}^d} |x_1-x_2|^p \pi(dx_1,dx_2) \le \inf_{\mathbf{\rho}^{1,2}\in \Gamma(\rho_{1},\rho_{2})} D_p(\mathbf{\rho}^{1,2})(t)^p\\
     &\le \exp\Bigg( \nu^{-1}\lv k \rv_\infty C(T,k,V,\nu,\rho_1,\rho_2,q)t+ \nu^{-\frac{1}{2}} \left(2\lv \nabla_1\nabla_2 k \rv_\infty+3\lv \nabla^2 k \rv_\infty\right)^{\frac{1}{2}}\sqrt{ \KL(\rho_2|\pi) } t^{\frac{1}{2}}\\
     &\qquad \qquad +2\nu^{-\frac{3}{2}}(1-\nu)\lv k \rv_\infty^2\left( 2\lv \nabla_1\nabla_2 k\rv_\infty+3\lv \nabla^2 k \rv_\infty \right)^{\frac{1}{2}}\sqrt{\KL(\rho_1|\pi)}t^{\frac{1}{2}} \Bigg) \inf_{\mathbf{\rho}^{1,2}\in \Gamma(\rho_{1},\rho_{2})} D_p(\mathbf{\rho}^{1,2})(0)^p\\
     &=\exp\Bigg( \nu^{-1}\lv k \rv_\infty C(T,k,V,\nu,\rho_1,\rho_2,q)t+ \nu^{-\frac{1}{2}} \left(2\lv \nabla_1\nabla_2 k \rv_\infty+3\lv \nabla^2 k \rv_\infty\right)^{\frac{1}{2}}\sqrt{ \KL(\rho_2|\pi) } t^{\frac{1}{2}}\\
     &\qquad \qquad +2\nu^{-\frac{3}{2}}(1-\nu)\lv k \rv_\infty^2\left( 2\lv \nabla_1\nabla_2 k\rv_\infty+3\lv \nabla^2 k \rv_\infty \right)^{\frac{1}{2}}\sqrt{\KL(\rho_1|\pi)}t^{\frac{1}{2}} \Bigg)\mc{W}_p^p(\rho_1,\rho_2),
 \end{align*}
 yielding the result.
 \end{proof}
 
\textcolor{black}{Note that in Lemma~\ref{lem:Lipschitz condition function space}, we studied perturbation bounds for \eqref{eq:CF of the regularized PDE} under two different push-forward maps. In the next result, with the existence and uniqueness of solutions proved in Theorem \ref{thm:regularized mf PDE unique and existence}, we study perturbation results for \eqref{eq:regularized SVGD mf PDE} with two different initial conditions. }
 
\begin{proposition}\label{prop:Lipschitz in stability Wp} Under the assumptions of Theorem \ref{thm:stability Wp}, %let $\Phi(\cdot,\cdot,\rho_i)$ be the solution to \eqref{eq:CF of the regularized PDE} with $\rho_0=\rho_i$ for $i=1,2$. 
 Let $\mathbf{\rho}^{1,2}$ be a probability measure on $\mb{R}^{2d}$ with marginals $\rho_1$ and $\rho_2$. Then \textcolor{black}{for any $x,y\in \mb{R}^d$}, we have
 \begin{align}
    &\left|\mc{D}_{\nu,\rho_{2,t}}\nabla\log \frac{\rho_{2,t}}{\pi}(\textcolor{black}{x})-\mc{D}_{\nu,\rho_{2,t}}\nabla\log \frac{\rho_{2,t}}{\pi}(\textcolor{black}{y})\right| \nonumber\\
    \le&~~ \nu^{-\frac{1}{2}} \left( 2\lv \nabla_1\nabla_2 k \rv_\infty+3\lv \nabla^2 k \rv_\infty\right)^{\frac{1}{2}} {I_{\nu,Stein}(\rho_{2,t}|\pi)}^{\frac{1}{2}}\left|\textcolor{black}{x-y}\right|, \label{eq:Lipschitz stability in Wp bound 2}\\
&\text{\hspace{-0.1in}and}\nonumber\\
         &\left|\mc{D}_{\nu,\rho_{1,t}}\nabla\log \frac{\rho_{1,t}}{\pi}(\textcolor{black}{x})-\mc{D}_{\nu,\rho_{2,t}}\nabla\log \frac{\rho_{2,t}}{\pi}(\textcolor{black}{x})\right| \nonumber \\
    \le &~~ 2\nu^{-\frac{3}{2}}(1-\nu)\lv k \rv_\infty^2\left( 2\lv \nabla_1\nabla_2 k\rv_\infty+3\lv \nabla^2 k \rv_\infty \right)^{\frac{1}{2}} I_{\nu,Stein}(\rho_{1,t}|\pi)^{\frac{1}{2}} \nonumber\\
    &\qquad\qquad\times\int_{\mb{R}^d} |\Phi(t,x_1,\rho_1)-\Phi(t,x_2,\rho_2)|\mathbf{\rho}^{1,2}(dx_1,dx_2) \nonumber \\
      &~~ +\nu^{-1}\lv k\rv_\infty C(t,k,V,\nu,\rho_1,\rho_2,q)\left(\int_{\mb{R}^d\times\mb{R}^d}\left|\Phi(t,y_1,\rho_1)-\Phi(t,y_2,\rho_2)\right|^p \mathbf{\rho}^{1,2}(dy_1,dy_2) \right)^{1/p}, \label{eq:Lipschitz stability in Wp bound 1} 
\end{align}
 where $C(t,k,V,\nu,\rho_1,\rho_2,q)$ is given in \eqref{eq:stability proof bound 5}.
 \end{proposition}
 \begin{proof}[Proof of Proposition~\ref{prop:Lipschitz in stability Wp}] 
 First, we prove \eqref{eq:Lipschitz stability in Wp bound 2}. For \textcolor{black}{any $x,y\in \mb{R}^d$},
 \begin{align*}
     &\quad \left|\mc{D}_{\nu,\rho_{2,t}}\nabla \log \frac{\rho_{2,t}}{\pi}(x)-\mc{D}_{\nu,\rho_{2,t}}\nabla \log \frac{\rho_{2,t}}{\pi}(y)\right|\\
     & =\left|\left((1-\nu)\iota_{k,\rho_{2,t}}^*\iota_{k,\rho_{2,t}}+\nu I_d\right)^{-1}\iota_{k,\rho_{2,t}}^*\nabla \log \frac{\rho_{2,t}}{\pi}(x)\right.\\
     &\qquad\qquad-\left.\left((1-\nu)\iota_{k,\rho_{2,t}}^*\iota_{k,\rho_{2,t}}+\nu I_d\right)^{-1}\iota_{k,\rho_{2,t}}^*\nabla \log \frac{\rho_{2,t}}{\pi}(y)\right|\\
     &=\left|\leftlangle k(x,\cdot)-k(y,\cdot),\left((1-\nu)\iota_{k,\rho_{2,t}}^*\iota_{k,\rho_{2,t}}+\nu I_d\right)^{-1}\iota_{k,\rho_{2,t}}^*\nabla \log \frac{\rho_{2,t}}{\pi}(\cdot)  \rightrangle_{\mc{H}_k}\right|\\
     &\le \lv  k(x,\cdot)-k(y,\cdot) \rv_{\mc{H}_k} \lv \left((1-\nu)\iota_{k,\rho_{2,t}}^*\iota_{k,\rho_{2,t}}+\nu I_d\right)^{-1}\iota_{k,\rho_{2,t}}^*\nabla \log \frac{\rho_{2,t}}{\pi} \rv_{\mc{H}_k^d}\\
     & \le \nu^{-\frac{1}{2}}I_{\nu,Stein}(\rho_{2,t}|\pi)^{\frac{1}{2}} \lv  k(x,\cdot)-k(y,\cdot) \rv_{\mc{H}_k^d}\\
     &\le \nu^{-\frac{1}{2}}I_{\nu,Stein}(\rho_{2,t}|\pi)^{\frac{1}{2}} \left( 2\lv \nabla_1\nabla_2 k \rv_\infty+3\lv \nabla^2 k \rv_\infty\right)^{\frac{1}{2}} |x-y|,
 \end{align*}
 where the second inequality follows from \eqref{eq:inverse gradient RKHS bound} and the last inequality follows from the reproducing property and Taylor expansion. The claim in \eqref{eq:Lipschitz stability in Wp bound 2} then follows from the above inequality. 
 
 To prove \eqref{eq:Lipschitz stability in Wp bound 1}, first we have for any $x\in \mb{R}^d$,
 \begin{align}\label{eq:ref 1.1}
       &\quad \left|\mc{D}_{\nu,\rho_{1,t}}\nabla\log \frac{\rho_{1,t}}{\pi}(x)-\mc{D}_{\nu,\rho_{2,t}}\nabla\log \frac{\rho_{2,t}}{\pi}(x)\right| \nonumber\\
        &\le \left|\left( \left((1-\nu)\iota_{k,\rho_{1,t}}^*\iota_{k,\rho_{1,t}}+\nu I_d\right)^{-1}-\left((1-\nu)\iota_{k,\rho_{2,t}}^*\iota_{k,\rho_{2,t}}+\nu I_d\right)^{-1}  \right)\iota_{k,\rho_{1,t}}^*\nabla \log \frac{\rho_{1,t}}{\pi}(x) \right|\nonumber\\
        &\ +\left|\left((1-\nu)\iota_{k,\rho_{2,t}}^*\iota_{k,\rho_{2,t}}+\nu I_d\right)^{-1} \left( \iota_{k,\rho_{1,t}}^*\nabla \log \frac{\rho_{1,t}}{\pi}(x)-\iota_{k,\rho_{2,t}}^*\nabla \log \frac{\rho_{2,t}}{\pi}(x) \right) \right|.
    \end{align}
    Next we bound the two terms on the right-hand side of \eqref{eq:ref 1.1}.
    
    \vspace{.4cm}
    \noindent \textbf{First term:}  we have
    \begin{align*}
     &\quad \left|\left( \left((1-\nu)\iota_{k,\rho_{1,t}}^*\iota_{k,\rho_{1,t}}+\nu I_d\right)^{-1}-\left((1-\nu)\iota_{k,\rho_{2,t}}^*\iota_{k,\rho_{2,t}}+\nu I_d\right)^{-1}  \right)\iota_{k,\rho_{1,t}}^*\nabla \log \frac{\rho_{1,t}}{\pi}(x) \right|\\
     &=\left|\iota_{k,\rho_{1,t}}\left((1-\nu)\iota_{k,\rho_{2,t}}^*\iota_{k,\rho_{2,t}}+\nu I\right)^{-1}\left((1-\nu)\iota_{k,\rho_{2,t}}^*\iota_{k,\rho_{2,t}}-(1-\nu)\iota_{k,\rho_{1,t}}^*\iota_{k,\rho_{1,t}}\right)\right. \\
     &  \qquad \qquad \left. \times\left((1-\nu)\iota_{k,\rho_{1,t}}^*\iota_{k,\rho_{1,t}}+\nu I\right)^{-1}\iota_{k,\rho_{1,t}}^*\nabla \log \frac{\rho_{1,t}}{\pi}(x)\right|  \\
     &\le \lv \iota_{k,\rho_{1,t}} \rv_{\mc{H}_k^d\to L_\infty^d} \lv \left((1-\nu)\iota_{k,\rho_{2,t}}^*\iota_{k,\rho_{2,t}}+\nu I\right)^{-1} \rv_{\mc{H}_k^d\to \mc{H}_k^d} (1-\nu)\lv \iota_{k,\rho_{2,t}}^*\iota_{k,\rho_{2,t}}-\iota_{k,\rho_{1,t}}^*\iota_{k,\rho_{1,t}} \rv_{\mc{H}_k^d\to\mc{H}_k^d}  \\
     &\qquad \qquad\times \lv \left((1-\nu)\iota_{k,\rho_{1,t}}^*\iota_{k,\rho_{1,t}}+\nu I\right)^{-1}\iota_{k,\rho_{1,t}}^*\nabla \log \frac{\rho_{1,t}}{\pi} \rv_{\mc{H}_k^d}  \\
     &\le  \lv k \rv_\infty \nu^{-\frac{3}{2}}(1-\nu)  I_{\nu,Stein}(\rho_{1,t}|\pi)^{\frac{1}{2}} \nonumber\\
     &\quad  \sup_{\lv \phi \rv_{\mc{H}_k^d}=1} \leftlangle \int_{\mb{R}^d} k(\cdot,x)\phi(x)(d\rho_{1,t}(x)-d\rho_{2,t}(x)), \int_{\mb{R}^d} k(\cdot,y)\phi(y)(d\rho_{1,t}(y)-d\rho_{2,t}(y)) \rightrangle_{\mc{H}_k^d}^{\frac{1}{2}},
    \end{align*}
    and 
    \begin{align*}
     &   \sup_{\lv \phi \rv_{\mc{H}_k^d}=1} \leftlangle \int_{\mb{R}^d} k(\cdot,x)\phi(x)(d\rho_{1,t}(x)-d\rho_{2,t}(x)), \int_{\mb{R}^d} k(\cdot,y)\phi(y)(d\rho_{1,t}(y)-d\rho_{2,t}(y)) \rightrangle_{\mc{H}_k^d}^{\frac{1}{2}}\\
     &=\bigg( \sup_{\lv \phi \rv_{\mc{H}_k^d}=1}  \int_{\mb{R}^{d}\times\mb{R}^d}\int_{\mb{R}^{d}\times\mb{R}^d} \bigg\langle k\left(\Phi(t,x_1,\rho_1),\cdot\right)\phi\left(\Phi(t,x_1,\rho_1)\right)-k\left(\Phi(t,x_2,\rho_2),\cdot\right)\phi\left(\Phi(t,x_2,\rho_2)\right), \\
     &%\qquad\qquad\qquad  
k\left(\Phi(t,y_1,\rho_1),\cdot\right)\phi\left(\Phi(t,y_1,\rho_1)\right)-k\left(\Phi(t,y_2,\rho_2),\cdot\right)\phi\left(\Phi(t,y_2,\rho_2)\right)\bigg\rangle_{\mc{H}_k^d} \mathbf{\rho}^{1,2}(dx_1,dx_2)\mathbf{\rho}^{1,2}(dy_1,y_2) \bigg)^{\frac{1}{2}}\\
     &\le \int_{\mb{R}^d}\int_{\mb{R}^d} \lv k\left(\Phi(t,x_1,\rho_1),\cdot\right)\phi\left(\Phi(t,x_1,\rho_1)\right)-k\left(\Phi(t,x_2,\rho_2),\cdot\right)\phi\left(\Phi(t,x_2,\rho_2)\right)\rv_{\mc{H}_k^d}   \mathbf{\rho}^{1,2}(dx_1,dx_2) \\
     &\le 2\lv k \rv_\infty \left( 2\lv \nabla_1\nabla_2 k\rv_\infty+3\lv \nabla^2 k \rv_\infty \right)^{\frac{1}{2}} \int_{\mb{R}^d\times\mb{R}^d} |\Phi(t,x_1,\rho_1)-\Phi(t,x_2,\rho_2)|\mathbf{\rho}^{1,2}(dx_1,dx_2).
    \end{align*}

\vspace{.4cm}

\noindent \textbf{Second term:} we have
  \begin{align*}
     &\quad \left|\left((1-\nu)\iota_{k,\rho_{2,t}}^*\iota_{k,\rho_{2,t}}+\nu I\right)^{-1} \left( \iota_{k,\rho_{1,t}}^*\nabla \log \frac{\rho_{1,t}}{\pi}(x)-\iota_{k,\rho_{2,t}}^*\nabla \log \frac{\rho_{2,t}}{\pi}(x) \right) \right|  \\
     &=\left|\iota_{k,\rho_{2,t}}\left((1-\nu)\iota_{k,\rho_{2,t}}^*\iota_{k,\rho_{2,t}}+\nu I\right)^{-1} \left( \iota_{k,\rho_{1,t}}^*\nabla \log \frac{\rho_{1,t}}{\pi}(x)-\iota_{k,\rho_{2,t}}^*\nabla \log \frac{\rho_{2,t}}{\pi}(x) \right) \right| \\
     &\le \lv \iota_{k,\rho_{2,t}} \rv_{\mc{H}_k^d\to L_\infty^d} \lv \left((1-\nu)\iota_{k,\rho_{2,t}}^*\iota_{k,\rho_{2,t}}+\nu I\right)^{-1} \rv_{\mc{H}_k^d\to \mc{H}_k^d} \lv \iota_{k,\rho_{1,t}}^*\nabla \log \frac{\rho_{1,t}}{\pi}-\iota_{k,\rho_{2,t}}^*\nabla \log \frac{\rho_{2,t}}{\pi} \rv_{\mc{H}_k^d}\\
     &\le \lv k \rv_\infty \nu^{-1}  \lv \iota_{k,\rho_{1,t}}^*\nabla \log \frac{\rho_{1,t}}{\pi}-\iota_{k,\rho_{2,t}}^*\nabla \log \frac{\rho_{2,t}}{\pi} \rv_{\mc{H}_k^d}.
    \end{align*}
    For the factor $\lv \iota_{k,\rho_{1,t}}^*\nabla \log \frac{\rho_{1,t}}{\pi}-\iota_{k,\rho_{2,t}}^*\nabla \log \frac{\rho_{2,t}}{\pi} \rv_{\mc{H}_k^d}$, notice that
    \begin{align*}
        &\quad \iota_{k,\rho_{1,t}}^*\nabla \log \frac{\rho_{1,t}}{\pi}(x)-\iota_{k,\rho_{2,t}}^*\nabla \log \frac{\rho_{2,t}}{\pi}(x) \\
        &=\int_{\mb{R}^d} k(x,y)\nabla \log \frac{\rho_{1,t}}{\pi}(y)d\rho_{1,t}(y)-\int_{\mb{R}^d} k(x,y)\nabla \log \frac{\rho_{2,t}}{\pi}d\rho_{2,t}(y) \\
        &= \int_{\mb{R}^d} \left( k(x,y)\nabla V(y)-\nabla_2 k(x,y) \right)d\rho_{1,t}(y)-\int_{\mb{R}^d} \left( k(x,y)\nabla V(y)-\nabla_2 k(x,y) \right)d\rho_{2,t}(y) \\
        &=\int_{\mb{R}^d\times \mb{R}^d} \left(k(x,\Phi(t,y_1,\rho_1))\nabla V(\Phi(t,y_1,\rho_1))-k(x,\Phi(t,y_2,\rho_2))\nabla V(\Phi(t,y_2,\rho_2))\right) d\mathbf{\rho}^{1,2}(dy_1,dy_2) \\
        & \quad - \int_{\mb{R}^d\times \mb{R}^d}  \left( \nabla_2 k(x,\Phi(t,y_1,\rho_1))-\nabla_2 k(x,\Phi(t,y_2,\rho_2)) \right) d\mathbf{\rho}^{1,2}(dy_1,dy_2),
    \end{align*}
    and we get
    \begin{align}\label{eq:ref 1.3}
     &\quad \lv \iota_{k,\textcolor{black}{\rho_{1,t}}}^*\nabla \log \frac{\textcolor{black}{\rho_{1,t}}}{\pi}-\iota_{k,\textcolor{black}{\rho_{2,t}}}^*\nabla \log \frac{\textcolor{black}{\rho_{2,t}}}{\pi} \rv_{\mc{H}_k^d} \nonumber \\
      &\le \int_{\mb{R}^d\times \mb{R}^d} \lv k(\cdot,\Phi(t,y_1,\rho_1))\nabla V(\Phi(t,y_1,\rho_1))-k(\cdot,\Phi(t,y_2,\rho_2))\nabla V(\Phi(t,y_2,\rho_2))\rv_{\mc{H}_k^d} d\mathbf{\rho}^{1,2}(dy_1,dy_2) \nonumber \\
        & \quad + \int_{\mb{R}^d\times \mb{R}^d}  \lv \nabla_2 k(\cdot,\Phi(t,y_1,\rho_1))-\nabla_2 k(\cdot,\Phi(t,y_2,\rho_2)) \rv_{\mc{H}_k^d} d\mathbf{\rho}^{1,2}(dy_1,dy_2).
        \end{align}
    For simplicity, we denote $\Phi(t,y_1,\rho_1),\Phi(t,y_2,\rho_2)$ as $\Phi_1$ and $\Phi_2$ respectively in the following calculations. We will bound the two integrals \textcolor{black}{in \eqref{eq:ref 1.3}} separately.\\
    
     \noindent For the first integral in \eqref{eq:ref 1.3}, we have 
    \begin{align*}
      & \quad \lv k(\cdot,\Phi_1)\nabla V(\Phi_1)-k(\cdot,\Phi_2)\nabla V(\Phi_2) \rv_{\mc{H}_k^d}^2   \\
      &= \leftlangle k(\cdot,\Phi_1)\nabla V(\Phi_1)-k(\cdot,\Phi_2)\nabla V(\Phi_2), k(\cdot,\Phi_1)\nabla V(\Phi_1)-k(\cdot,\Phi_2)\nabla V(\Phi_2) \rightrangle_{\mc{H}_k^d}\\
      &=|\nabla V(\Phi_1)|^2 k(\Phi_1,\Phi_1)-2\leftlangle \nabla V(\Phi_1) , \nabla V(\Phi_2) \rightrangle k(\Phi_1,\Phi_2)+|\nabla V(\Phi_2)|^2 k(\Phi_2,\Phi_2)\\
      &\textcolor{black}{\le |}\leftlangle \nabla V(\Phi_1)-\nabla V(\Phi_2) ,  \nabla V(\Phi_1)k(\Phi_1,\Phi_1)-\nabla V(\Phi_2)k(\Phi_2,\Phi_2) \rightrangle \textcolor{black}{|}\\
      &\quad +\textcolor{black}{|}\leftlangle \nabla V(\Phi_1), \nabla V(\Phi_2)\rightrangle \left( k(\Phi_1,\Phi_1)+k(\Phi_2,\Phi_2)-2k(\Phi_1,\Phi_2) \right)\textcolor{black}{|},
    \end{align*}
    where 
        \begin{align}\label{eq:ref 1.4}
        &\quad  \left| \leftlangle \nabla V(\Phi_1)-\nabla V(\Phi_2) ,  \nabla V(\Phi_1)k(\Phi_1,\Phi_1)-\nabla V(\Phi_2)k(\Phi_2,\Phi_2) \rightrangle \right|\nonumber\\
        &\le |\nabla V(\Phi_1)-\nabla V(\Phi_2)|^2 k(\Phi_1,\Phi_1)+|\nabla V(\Phi_1)-\nabla V(\Phi_2)||\nabla V(\Phi_2)||k(\Phi_1,\Phi_1)-k(\Phi_2,\Phi_2) |\nonumber\\
        &\le \sup_{\theta \in [0,1]}\lv \nabla ^2 ( \theta \Phi_1+(1-\theta)\Phi_2 )\rv_2^2  |\Phi_1-\Phi_2|^2 \lv k \rv_\infty^2\nonumber\\
        & \quad +\sup_{\theta \in [0,1]}\lv \nabla ^2 ( \theta \Phi_1+(1-\theta)\Phi_2 )\rv_2  |\Phi_1-\Phi_2| \textcolor{black}{|\nabla V(\Phi_2)|} |k(\Phi_1,\Phi_1)-k(\Phi_2,\Phi_2) |\nonumber \\
        &\le C_V^{2/q}(1+V(\Phi_1)+V(\Phi_2))^{2/q} \lv k \rv_\infty^2 |\Phi_1-\Phi_2|^2\nonumber\\
        & \quad + C_V^{1/q} (1+V(\Phi_1)+V(\Phi_2))^{1/q} C_V^{\textcolor{black}{1/q}}(1+V(\Phi_2))^{1/q} \lv \nabla k \rv_\infty |\Phi_1-\Phi_2|^2 \nonumber\\
        &\le C_V^{2/q} \left( \lv k \rv_\infty^2+\lv \nabla k \rv_\infty \right) (1+V(\Phi_1)+V(\Phi_2))^{2/q} |\Phi_1-\Phi_2|^2.
    \end{align}
    The third inequality follows from Assumption \ref{ass:existence and uniqueness on V V2} and the last inequality follows from Assumption \ref{ass:existence and uniqueness on K} and Taylor expansion on both variables in $k$ up to first order. Furthermore, we have
    \begin{align}\label{eq:ref 1.5}
        &\quad \left| \leftlangle \nabla V(\Phi_1), \nabla V(\Phi_2)\rightrangle \left( k(\Phi_1,\Phi_1)+k(\Phi_2,\Phi_2)-2k(\Phi_1,\Phi_2) \right) \right|\nonumber\\
        &\le C_V^{2/q} (1+V(\Phi_1))^{1/q}(1+V(\Phi_2))^{1/q} |k(\Phi_1,\Phi_1)+k(\Phi_2,\Phi_2)-2k(\Phi_1,\Phi_2)|\nonumber\\
        &\le C_V^{2/q} (1+V(\Phi_1))^{1/q}(1+V(\Phi_2))^{1/q}\left( 3\lv \nabla^2 k \rv_\infty+2\lv \nabla_1\nabla_2 k\rv_\infty \right) |\Phi_1-\Phi_2|^2,
    \end{align}
    where the first inequality follows from Assumption \ref{ass:existence and uniqueness on V V2} and the last inequality follows from Assumption \ref{ass:existence and uniqueness on K} and Taylor expansion on both variables in $k$ up to second order. 
    
    From \textcolor{black}{\eqref{eq:ref 1.4} and \eqref{eq:ref 1.5}}, we get
    \begin{align}%
    &\quad  \int_{\mb{R}^d\times \mb{R}^d} \lv k(\cdot,\Phi(t,y_1,\rho_1))\nabla V(\Phi(t,y_1,\rho_1))-k(\cdot,\Phi(t,y_2,\rho_2))\nabla V(\Phi(t,y_2,\rho_2))\rv_{\mc{H}_k^d} d\mathbf{\rho}^{1,2}(dy_1,dy_2) \nonumber\\
   &\le C_V^{1/q}\left( \lv k \rv_\infty+{\lv \nabla k \rv_\infty}^{\frac{1}{2}}+3{\lv \nabla^2 k \rv_\infty}^{\frac{1}{2}}+2{\lv \nabla_1\nabla_2 k \rv_\infty}^{\frac{1}{2}} \right)  \nonumber \\
   & \quad \qquad   \times\int_{\mb{R}^d\times\mb{R}^d}|\Phi_1(t,y_1,\rho_1)-\Phi(t,y_2,\rho_2)|(1+V(\Phi(t,y_1,\rho_1))+V(\Phi(t,y_2,\rho_2)))^{1/q}\mathbf{\rho}^{1,2}(dy_1,dy_2) \nonumber\\
   &\le C_V^{1/q}\left( \lv k \rv_\infty+{\lv \nabla k \rv_\infty}^{\frac{1}{2}}+3{\lv \nabla^2 k \rv_\infty}^{\frac{1}{2}}+2{\lv \nabla_1\nabla_2 k \rv_\infty}^{\frac{1}{2}} \right)   \nonumber \\
   & \quad \qquad  \times\left(\int_{\mb{R}^d\times\mb{R}^d}|\Phi_1(t,y_1,\rho_1)-\Phi(t,y_2,\rho_2)|^p \mathbf{\rho}^{1,2}(dy_1,dy_2) \right)^{1/p} \nonumber  \\  & \quad \qquad \qquad\times\left(\int_{\mb{R}^d\times\mb{R}^d} 
\textcolor{black}{\big(}1+V(\Phi(t,y_1,\rho_1))+V(\Phi(t,y_2,\rho_2)) \textcolor{black}{\big)} \mathbf{\rho}^{1,2}(dy_1,dy_2)\right)^{1/q} \nonumber \\
   &\le \left( \lv \rho_1 \rv_{\mc{P}_V} \exp(C_{1,0}\nu^{-1/2}\lv k \rv_\infty \sqrt{t \KL(\rho_1|\pi)} )\right.\nonumber\\
   &\qquad\qquad\qquad\left. +\lv \rho_2 \rv_{\mc{P}_V} \exp(C_{1,0}\nu^{-1/2}\lv k \rv_\infty \sqrt{t \KL(\rho_2|\pi) )}  \right)^{1/q} \nonumber\\
   &\quad \qquad   \times C_V^{1/q}\left( \lv k \rv_\infty+{\lv \nabla k \rv_\infty}^{\frac{1}{2}}+3{\lv \nabla^2 k \rv_\infty}^{\frac{1}{2}}+2{\lv \nabla_1\nabla_2 k \rv_\infty}^{\frac{1}{2}} \right)\nonumber\\
   & \quad \quad\qquad \qquad\times\left(\int_{\mb{R}^d\times\mb{R}^d}|\Phi_1(t,y_1,\rho_1)-\Phi(t,y_2,\rho_2)|^p \mathbf{\rho}^{1,2}(dy_1,dy_2) \right)^{1/p} \nonumber\\
   &\le  3C_V^{1/q}\left( \lv k \rv_\infty+{\lv \nabla k \rv_\infty}^{\frac{1}{2}}+{\lv \nabla^2 k \rv_\infty}^{\frac{1}{2}}+{\lv \nabla_1\nabla_2 k \rv_\infty}^{\frac{1}{2}} \right)\left( \lv \rho_1 \rv_{\mc{P}_V}+\lv \rho_2 \rv_{\mc{P}_V} \right)^{1/q} \nonumber\\
   & \quad\qquad \times \exp(C_{1,0}\nu^{-1/2}q^{-1}\lv k \rv_\infty \sqrt{t (\KL(\rho_1|\pi)+\KL(\rho_2|\pi)})) \nonumber \\
   & \quad\qquad  \qquad\times \left(\int_{\mb{R}^d\times\mb{R}^d}|\Phi_1(t,y_1,\rho_1)-\Phi(t,y_2,\rho_2)|^p \mathbf{\rho}^{1,2}(dy_1,dy_2) \right)^{1/p} \nonumber \\
   &:=C_1(k,V) \left( \lv \rho_1 \rv_{\mc{P}_V}+\lv \rho_2 \rv_{\mc{P}_V} \right)^{\frac{1}{q}}\exp\left( D_1(k,\nu,q) \left(\KL(\rho_1|\pi)+\KL(\rho_2|\pi)\right)^{\frac{1}{2}}t^{\frac{1}{2}} \right)  \nonumber \\
   &\quad \qquad  \times\left(\int_{\mb{R}^d\times\mb{R}^d}|\Phi_1(t,y_1,\rho_1)-\Phi(t,y_2,\rho_2)|^p \mathbf{\rho}^{1,2}(dy_1,dy_2) \right)^{1/p} \label{eq:stability proof bound 2}.
    \end{align}
   \textcolor{black}{where the third inequality follows from \eqref{eq:PV norm bound}}.
   
    \vspace{.2cm}
    \noindent For the second integral in \eqref{eq:ref 1.3}, denoting the function $\nabla_1\cdot\nabla_2 k=D_{1,2}k$, we first notice that $D_{1,2}k$ is symmetric since $k$ is symmetric. According to the above identity, we get
    \begin{align*}
     \lv \nabla_2 k(\cdot,\Phi_1)-\nabla_2 k(\cdot,\Phi_2) \rv_{\mc{H}_k^d}^2&=D_{1,2}k(\Phi_1,\Phi_1)+ D_{1,2}k(\Phi_2,\Phi_2)-2D_{1,2}k(\Phi_1,\Phi_2)  .
    \end{align*}
    Applying Taylor's series expansion on both variables of $D_{1,2}k$, we get
    \begin{align*}
      \lv \nabla_2 k(\cdot,\Phi_1)-\nabla_2 k(\cdot,\Phi_2) \rv_{\mc{H}_k^d}^2\le \left(2\lv \nabla^2(D_{1,2}k) \rv_\infty+\lv \nabla_1\nabla_2(D_{1,2}k) \rv_\infty \right)|\Phi_1-\Phi_2|^2 .  
    \end{align*}
    \textcolor{black}{Therefore} we obtain
    \begin{align}\label{eq:stability proof bound 3}
         &\int_{\mb{R}^d}\lv \nabla_2 k(\cdot,\Phi(t,y_1,\rho_1))-\nabla_2 k(\cdot,\Phi(t,y_2,\rho_2)) \rv_{\mc{H}_k^d} d\mathbf{\rho}^{1,2}(dy_1,dy_2) \nonumber\\
         \le & \left(2{\lv \nabla^2(D_{1,2}k) \rv_\infty}^{\frac{1}{2}}+{\lv \nabla_1\nabla_2(D_{1,2}k) \rv_\infty}^{\frac{1}{2}} \right)\int_{\mb{R}^d\times\mb{R}^d}|\Phi(t,y_1,\rho_1)-\Phi(t,y_2,\rho_2)| \mathbf{\rho}^{1,2}(dy_1,dy_2).
    \end{align}
    Based on \textcolor{black}{\eqref{eq:ref 1.3}, \eqref{eq:stability proof bound 2} and \eqref{eq:stability proof bound 3}}, we then get
     \begin{align}\label{eq:ref 1.6}
     & \quad \lv \iota_{k,\textcolor{black}{\rho_{1,t}}}^*\nabla \log \frac{\textcolor{black}{\rho_{1,t}}}{\pi}-\iota_{k,\textcolor{black}{\rho_{2,t}}}^*\nabla \log \frac{\textcolor{black}{\rho_{2,t}}}{\pi} \rv_{\mc{H}_k^d}\nonumber\\
     & \le C_1(k,V) \left( \lv \rho_1 \rv_{\mc{P}_V}+\lv \rho_2 \rv_{\mc{P}_V} \right)^{\frac{1}{q}}\exp\left( D_1(k,\nu,q) \left(\KL(\rho_1|\pi)+\KL(\rho_2|\pi)\right)^{\frac{1}{2}}t^{\frac{1}{2}} \right)  \nonumber\\
     & \quad\qquad \qquad \left(\int_{\mb{R}^d\times\mb{R}^d}|\Phi_1(t,y_1,\rho_1)-\Phi(t,y_2,\rho_2)|^p \mathbf{\rho}^{1,2}(dy_1,dy_2) \right)^{1/p} \nonumber\\
     & \quad+ \left(2{\lv \nabla^2(D_{1,2}k) \rv_\infty}^{\frac{1}{2}}+{\lv \nabla_1\nabla_2(D_{1,2}k) \rv_\infty}^{\frac{1}{2}} \right)\int_{\mb{R}^d\times\mb{R}^d}|\Phi(t,y_1,\rho_1)-\Phi(t,y_2,\rho_2)| \mathbf{\rho}^{1,2}(dy_1,dy_2)\nonumber\\
     &\le C(t,k,V,\nu,\rho_1,\rho_2,q)\left(\int_{\mb{R}^d\times\mb{R}^d}|\Phi_1(t,y_1,\rho_1)-\Phi(t,y_2,\rho_2)|^p \mathbf{\rho}^{1,2}(dy_1,dy_2) \right)^{1/p},
     \end{align}
    with 
    \begin{align}\label{eq:stability proof bound 5}
        &C(t,k,V,\nu,\rho_1,\rho_2,q)\nonumber\\
        &=C_1(k,V) \left( \lv \rho_1 \rv_{\mc{P}_V}+\lv \rho_2 \rv_{\mc{P}_V} \right)^{\frac{1}{q}}\exp\left( D_1(k,\nu,q) \left(\KL(\rho_1|\pi)+\KL(\rho_2|\pi)\right)^{\frac{1}{2}}t^{\frac{1}{2}} \right) +C_2(k),
    \end{align}
    where 
    \begin{align*}
        C_1(k,V)&= 3C_V^{1/q}\left( \lv k \rv_\infty+{\lv \nabla k \rv_\infty}^{\frac{1}{2}}+{\lv \nabla^2 k \rv_\infty}^{\frac{1}{2}}+{\lv \nabla_1\nabla_2 k \rv_\infty}^{\frac{1}{2}} \right), \\
        D_1(k,\nu,q)&=C_{1,0}\nu^{-1/2}q^{-1}\lv k \rv_\infty,\\
        C_2(k)&=2{\lv \nabla^2(\nabla_1\cdot \nabla_2 k) \rv_\infty}^{\frac{1}{2}}+{\lv \nabla_1\nabla_2(\nabla_1\cdot \nabla_2 k) \rv_\infty}^{\frac{1}{2}}.
    \end{align*}
    Therefore \eqref{eq:Lipschitz stability in Wp bound 1} follows from our estimations on the \textbf{First term} and \textbf{Second term}.
     \end{proof}

\section{Space-time Discretization: A Practical Algorithm}\label{sec:practical}
In this section, we introduce a practical space-time discretization to the R-SVGF described in~\eqref{eq:CF of the regularized PDE}. In the algorithm, we let positive integers $N$ and $n$ to denote the number of particles and (discrete) iterations. We denote by $(X_n^i)_{i=1}^N$  the position of the $N$ particles at the $n$-th step. We let $\bar{X}_n:=[X_n^1,\ldots,X_n^N]^T$. For all functions $f:\mb{R}^d\to \mb{R}^d$, we define the operator $L_n$ as $$L_n f:=[f(X_n^1),\cdots,f(X_n^N)]^T.$$ 
The positions of the particles are then updated as \begin{align}\label{eq:algorithm}
     \bar{X}_{n+1}=\bar{X}_n-h_{n+1}\left( \frac{(1-\nu_{n+1})}{N} K_n +\nu_{n+1} I_N \right)^{-1} \left( \frac{1}{N} K_n (L_n\nabla V) -\frac{1}{N}\sum_{j=1}^N L_n\nabla k(X_n^j,\cdot) \right),
 \end{align}
where $(h_n)_{n=1}^\infty$ is the sequence of step-sizes, $I_{N\times N}$ is the $N \times N$ identity matrix and $K_n \in \mathbb{R}^{N\times N}$ is the gram matrix defined as $(K_n)_{ij}=k(X_n^i,X_n^j)$ for all $i,j\in [N]$. We call the above algorithm as the Regularized SVGD algorithm. The iterates in \eqref{eq:algorithm} follow from \textcolor{black}{Proposition} \ref{lem:regularized SVGD optimal vf} and the finite-sample representations for the operators $\iota_{k,\hat{\rho}^n}\iota_{k,\hat{\rho}^n}$ where $\hat{\rho}^n$ is the empirical measure of the particles at the $n$-th step, i.e., $\hat{\rho}^n=\sum_{i=1}^N \delta_{X_n^i}$.

While the convergence analysis of space-time discretization of the SVGF (i.e., the SVGD algorithm) and the R-SVGD (i.e., the regularized SVGD algorithm) is an interesting and challenging open question, in this section we demonstrate the improved performance of the regularized SVGD algorithm over the SVGD algorithm in some simulation examples. All experiments were done in MacBook Pro (2021 model). Specifically, we consider the simulation setup in~\cite{liu2016stein}: We let the target $\pi: = (1/3) \pi_1 + (2/3) \pi_2$, where $\pi_1\equiv \text{Normal}(-2, 1)$ and 
$\pi_2 \equiv \text{Normal}(+2, 1)$, and we let the initial distribution to be $\text{Normal}(-10, 1)$. We now focus on numerically computing the expectations of the form $\mathbb{E}_{x\sim \pi} [h_i(x)]$, for three cases, $h_1(x):= x$, $h_2(x):= x^2$ and $h_3(x):= \cos (\omega x +b)$, where $\omega \sim \text{Normal}(0,1)$ and $b\sim \text{Uniform}([0,2\pi])$.  

\begin{figure}[t]
\centering
\includegraphics[scale=0.35]{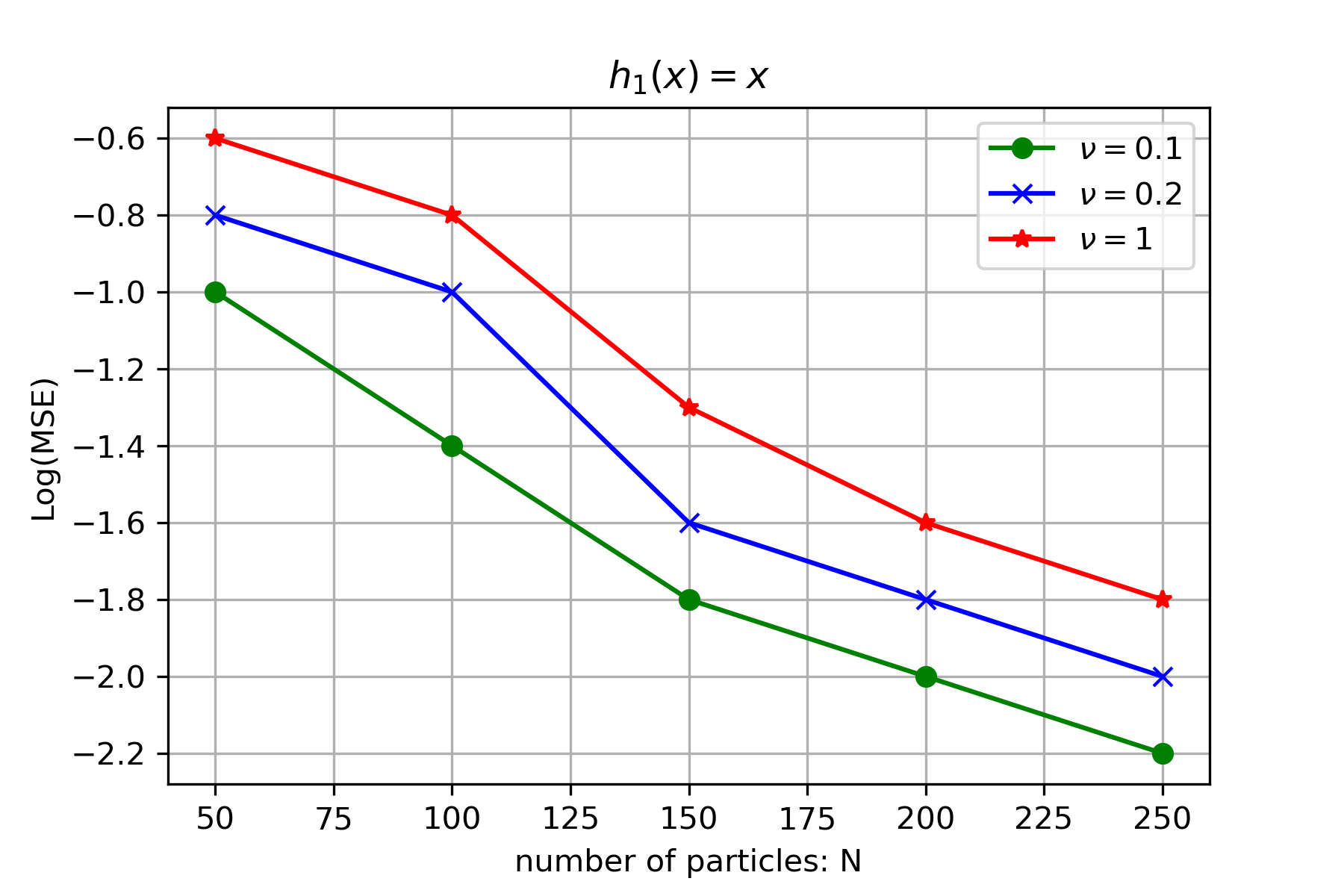}
\includegraphics[scale=0.35]{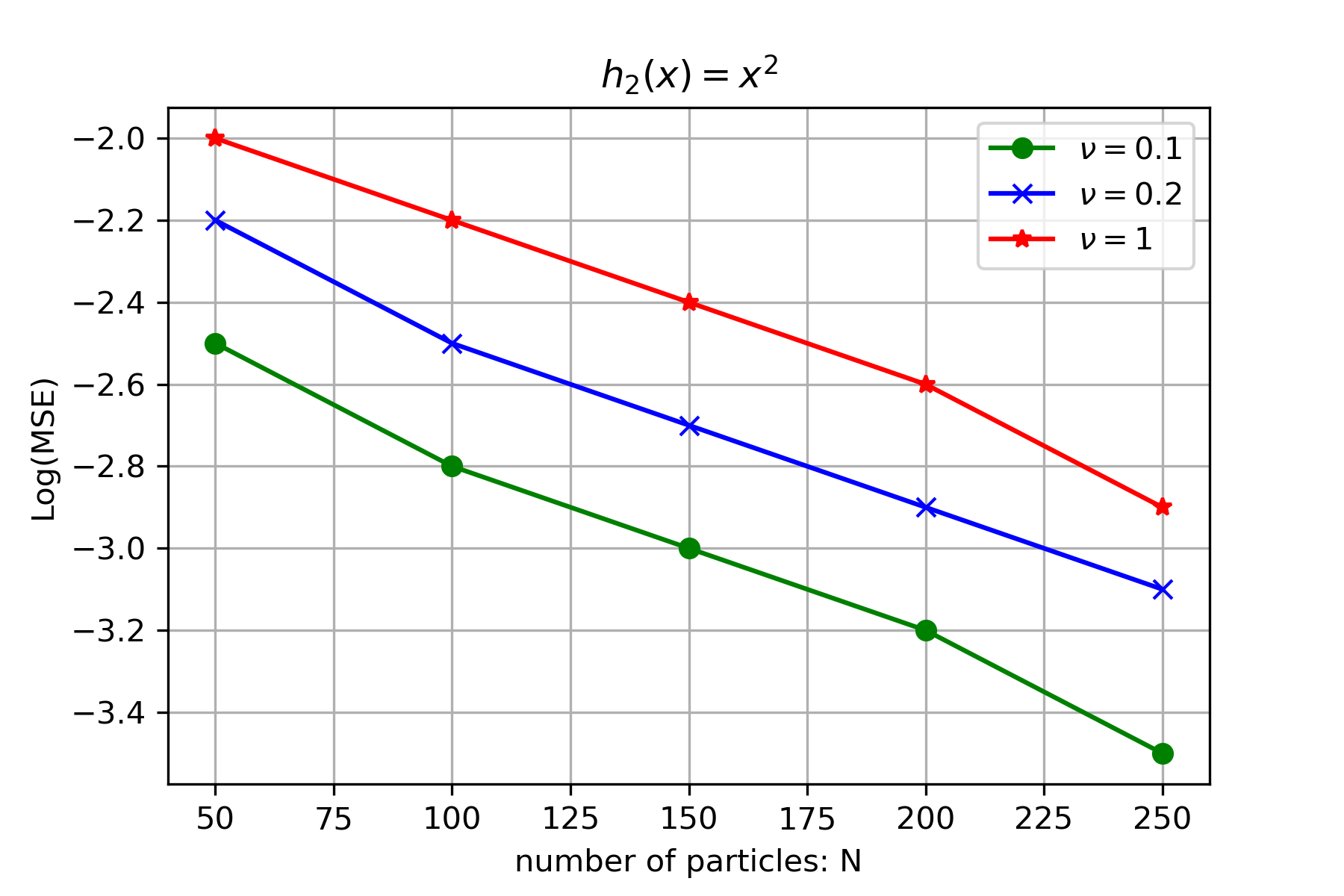}
\includegraphics[scale=0.35]{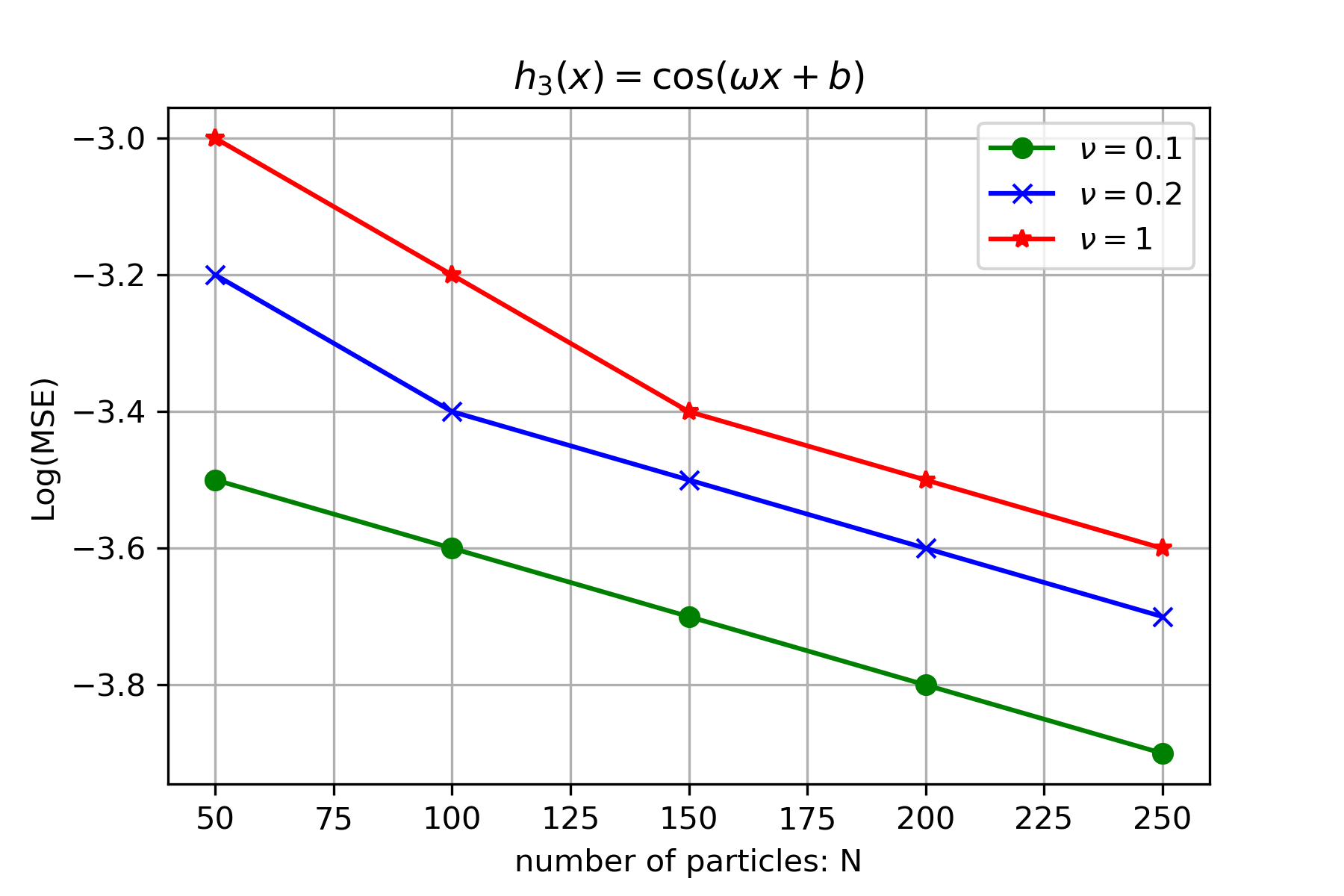}\\
\includegraphics[scale=0.35]{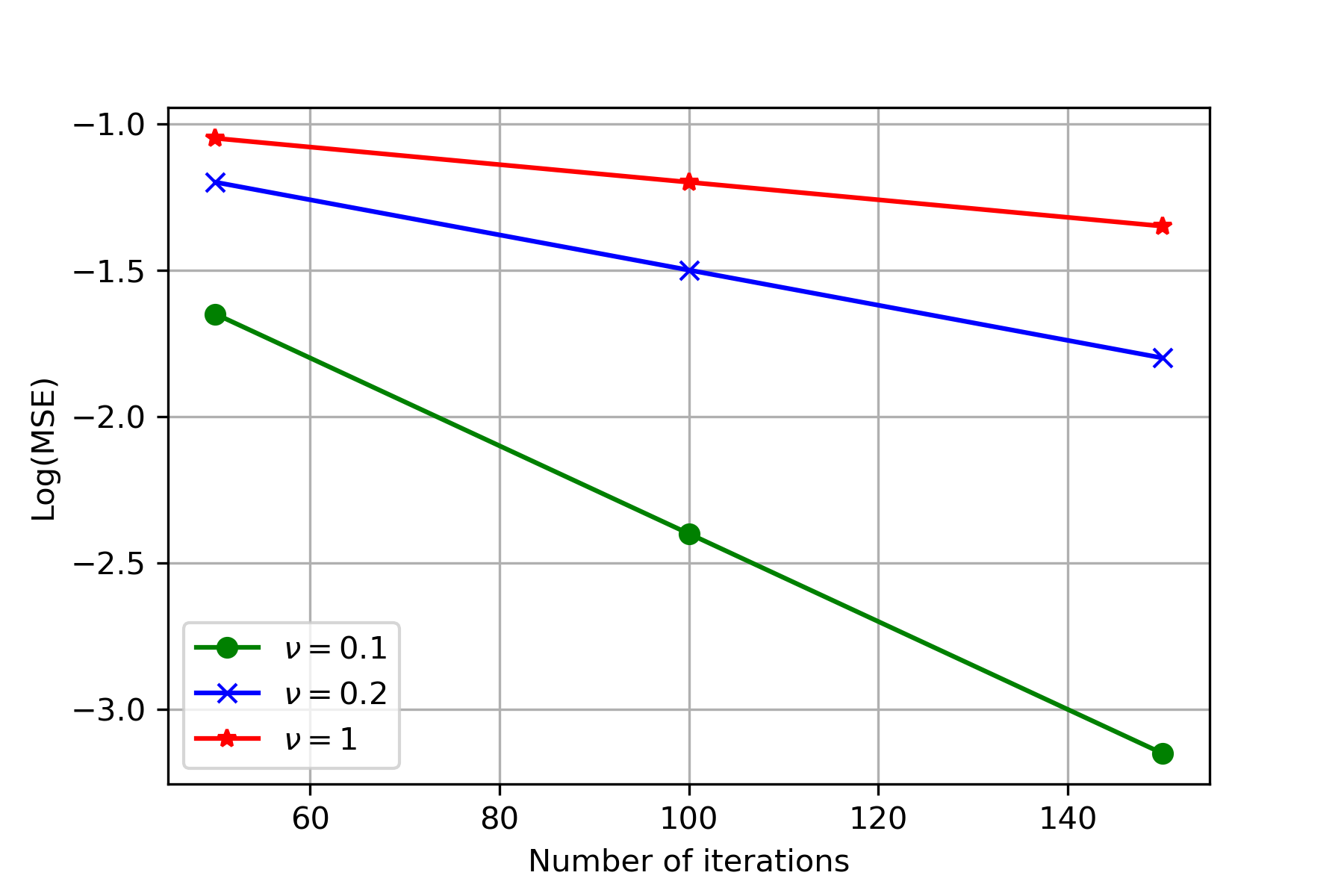}
\includegraphics[scale=0.35]{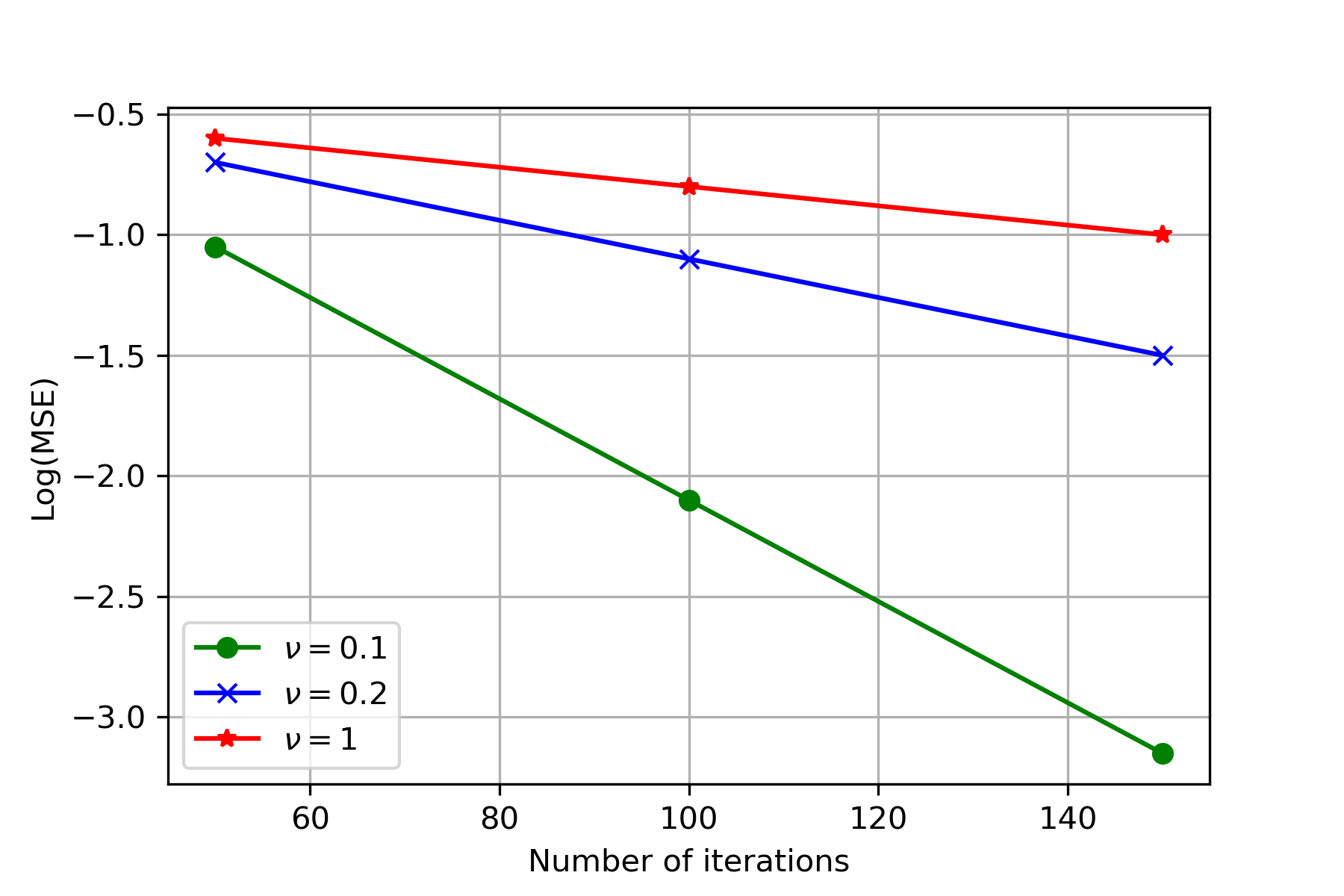}
\includegraphics[scale=0.35]{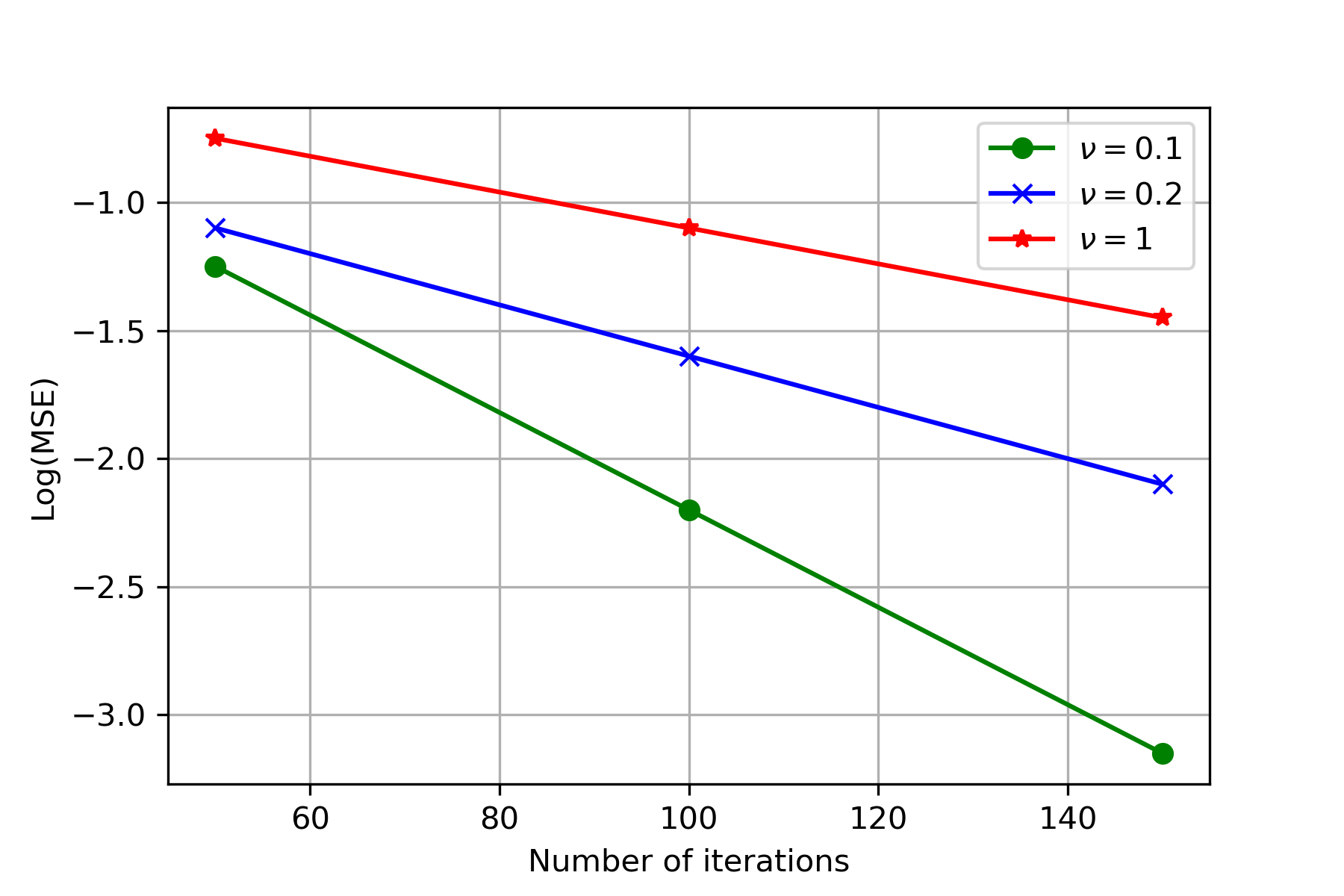}\\
\caption{R-SVGD for various values of the regularization parameter $\nu$. The case of $\nu=1$ corresponds to SVGD. Left, Middle and Right columns correspond respectively to  $h_1(x):= x$, $h_2(x):= x^2$ and $h_3(x):= \cos (\omega x +b)$. Top and bottom rows correspond respectively to log(MSE) versus number of particles and number of iterations.}
\label{fig:sim}
\end{figure}

\begin{wraptable}{r}{9cm}    \centering
    \begin{tabular}{|c|c|c|c|c|c|}
    \hline 
        $N$ & 50  & 100 & 150&200&250  \\ \hline \hline 
        $\nu=0.1$& 0.0114  & 0.0343  & 0.0861 &0.1442 & 0.2163 \\ \hline
         $\nu=0.2$& 0.0099    &0.0257 & 0.0586 & 0.1058 & 0.1899  \\ \hline
 %        150&   & \\ \hline
%         200&   & \\ \hline
%         250&   & \\ \hline
    \end{tabular}
    \caption{Average \emph{additional} per-iteration wall-clock run-time (in seconds) for R-SVGD over SVGD.}
    \label{tab:time}
\end{wraptable}

In Figure~\ref{fig:sim}, we plot the mean-squared error in estimating the above expectations with the regularized and unregularized SVGD algorithm. Here, the expectation is over the initialization (and over $\omega$ and $b$ for $h_3$). In the top row, we report the logarithm of the mean-squared error versus the number of particles $N$ for a fixed number of iterations (set to 100). The corresponding average per-iteration wall-clock run times are reported in Table~\ref{tab:time}. In the bottom row, we report the logarithm of the mean-squared error versus the number of iterations for a fixed number of particles (set to 200). For both algorithms, we use the Gaussian kernel $k(u,v)= \exp \left(-\frac{1}{\gamma} \| u-v\|_2^2\right)$, where the bandwidth parameter $\gamma$ is set using the median heuristic~\cite{liu2016stein}. We use the Adagrad step-size choice for both cases, following~\cite{liu2016stein}. For the choice of the regularization parameter, we report results for various choices of $\nu$. The case of $\nu=1$ corresponds exactly to the SVGD algorithm. We notice that for small values of $\nu$ the regularized SVGD algorithm performs better than the SVGD algorithm.

\textcolor{black}{We now discuss the per-iteration complexity of regularized SVGD and standard SVGD. SVGD requires computing the kernel matrix (which requires $\mathcal{O}(N^2)$ operations per-iterations) and the gradient of the potential (i.e., $\nabla V$). As pointed out in~\cite{liu2016stein} in Bayesian inference problems, typically, computing the gradient of the potential might be the main bottleneck. However, this computation could easily be done in parallel due to its finite-sum structure in Bayesian inference problems. We denote by $\mathcal{O}(\text{Score})$, the time complexity of computing/evaluating the gradient of the potential for a given target density. On top of the above computations, each iteration of regularized SVGD requires inverting a $N \times N$ matrix (or equivalently, solving a positive-definite linear system of equations), which costs $\mathcal{O}(N^3)$ complexity. With standard implementations, the per-iteration complexity regularized SVGD and standard SVGD are hence of order $\max\{\mathcal{O}(N^3), \mathcal{O}(\text{Score})\}$ and $\max\{\mathcal{O}(N^2), \mathcal{O}(\text{Score})\}$, respectively.}

\textcolor{black}{The matrix being inverted in regularized SVGD is the regularized kernel matrix also arising in other problems like kernel ridge regression. Hence, the rich literature on efficiently inverting kernel matrices (under various structural assumptions) could be leveraged in this context. While we leave a detailed study of speeding up the regularized SVGD algorithm as future work, we conclude here two concrete methods for provably speeding-up practical implementations of the regularized SVGD algorithms (under structural assumptions):}
\begin{itemize}
\item \textcolor{black}{\textit{Pre-conditioned CG methods:} State-of-the-art methods on designing pre-conditioners for conjugate gradient methods for kernel ridge regression from~\cite{diaz2023robust} could be leveraged in the context of regularized SVGD. In particular, such results may help reduce the computational complexity of regularized SVGD to $\max\{\mathcal{O}(N^2), \mathcal{O}(\text{Score})\}$, matching that of SVGD.}
%\item \textcolor{black}{Nystr\"om method}
%\item It is possible to speed-up the regularized SVGD algorithm by reinterpreting the iterations as solving a system of linear equations, and using faster practical implementations of linear systems solvers.
\item \textcolor{black}{\textit{Randomized methods}: Randomized algorithms designed in the context of kernel ridge regression, namely Random Fourier Features~\cite{rudi2017generalization} and sketching~\cite{yang2017randomized} could also be leveraged to speed-up practical implementations of regularized SVGD.} 
\end{itemize}

\subsection*{Acknowledgements}
YH was supported in part by NSF TRIPODS grant CCF-1934568. KB was supported in part by NSF grant DMS-2053918. BKS was supported in part by NSF CAREER Award DMS-1945396. JL was supported in part by NSF via award DMS-2012286. Parts of this work were done when YH, KB and JL visited the Simons Institute for the Theory of Computing as a part of the ``Geometric Methods in Optimization and Sampling" program during Fall 2021. 

%\newpage

\appendix
\bibliographystyle{alpha}
\bibliography{citation}
\end{document}